%% file: main.tex
\documentclass[11pt]{article}

\input{0_config/init_load}
\input{0_config/math_commands.tex}

\newcommand*\samethanks[1][\value{footnote}]{\footnotemark[#1]}

\title{Faster Sampling without Isoperimetry \\ via Diffusion-based Monte Carlo}

\author[$\dagger$]{\normalsize Xunpeng Huang\thanks{Mail to \href{xhuangck@connect.ust.hk}{xhuangck@connect.ust.hk}, \href{dzou@cs.hku.hk} {dzou@cs.hku.hk}}}
\author[$\S$]{Difan Zou\samethanks}
\author[$\dagger$]{Hanze Dong}
\author[$\P$]{Yian Ma}
\author[$\ddag$]{Tong Zhang}

\affil[$\dagger$]{Hong Kong University of Science and Technology}
\affil[$\S$]{The University of Hong Kong}
\affil[$\P$]{University of California San Diego}
\affil[$\ddag$]{University of Illinois Urbana-Champaign}

\begin{document}

\date{}
\maketitle

\begin{abstract}
    \input{0_contents/abstract}
\end{abstract}

\input{0_contents/010introduction}
\input{0_contents/020preliminaries}

\input{0_contents/030methods}

\input{0_contents/040analysis}
\input{0_contents/050exp}
\input{0_contents/060conclusion}

\bibliographystyle{apalike}
\bibliography{0_contents/ref}  





\newpage
\appendix
\input{0_contents/0Xappendix_OU/appendix_main}
\end{document}

%% file: 0_config/init_load.tex
\usepackage[utf8]{inputenc} 

\pdfoutput=1

\usepackage{algorithm}
\usepackage[noend]{algpseudocode}
\usepackage{amsmath}
\usepackage{amssymb}
\usepackage{amsthm,amsfonts}
\usepackage{array}
\usepackage{authblk}
\usepackage{babel}
\usepackage{booktabs}
\usepackage{caption}
\usepackage{colortbl}
\usepackage{enumitem}
\usepackage[T1]{fontenc}    
\usepackage{framed}
\usepackage{fullpage}
\usepackage{graphicx}
\usepackage[colorlinks, linkcolor=red, anchorcolor=blue, citecolor=blue]{hyperref}
\usepackage{indentfirst}
\usepackage[utf8]{inputenc}
\usepackage{lipsum}        
\usepackage{makecell}
\usepackage{mathrsfs}
\usepackage{mathtools}

\usepackage[framemethod=default]{mdframed}
\usepackage{microtype}     
\usepackage{multicol}
\usepackage{multirow}
\usepackage{natbib}
\usepackage{nicefrac}       
\usepackage{tablefootnote}
\usepackage{url}            
\usepackage{wrapfig,lipsum}
\usepackage{xcolor}         
\usepackage{xspace}

\usepackage{doi}
\usepackage{tikz}
\usetikzlibrary{decorations.pathreplacing}
\def \dz {\textcolor{red}}

%% file: 0_config/math_commands.tex
\usepackage{amsmath,amsfonts,bm}

\newcommand{\name}{{{RDS}}}

\def\floor#1{\lfloor #1 \rfloor}
\def\1{\bm{1}}

\def\rvv{{\mathbf{v}}}

\def\rvx{{\mathbf{x}}}
\def\rvy{{\mathbf{y}}}
\def\rvz{{\mathbf{z}}}

\def\vzero{{\bm{0}}}

\def\va{{\bm{a}}}
\def\vb{{\bm{b}}}

\def\ve{{\bm{e}}}
\def\vf{{\bm{f}}}

\def\vv{{\bm{v}}}

\def\vx{{\bm{x}}}
\def\vy{{\bm{y}}}
\def\vz{{\bm{z}}}

\def\mI{{\bm{I}}}

\DeclareMathAlphabet{\mathsfit}{\encodingdefault}{\sfdefault}{m}{sl}
\SetMathAlphabet{\mathsfit}{bold}{\encodingdefault}{\sfdefault}{bx}{n}

\def\sP{{\mathbb{P}}}

\def\sS{{\mathbb{S}}}

\newcommand{\grad}{\ensuremath{\nabla}}

\newcommand{\E}{\mathbb{E}}

\newcommand{\R}{\mathbb{R}}

\newcommand{\KL}[2]{\mathrm{KL}\left(#1 \big\| #2\right)}
\newcommand{\FI}[2]{\mathrm{FI}\left(#1 \| #2\right)}

\newcommand{\Var}{\mathrm{Var}}

\newcommand{\der}{\mathrm{d}}

\makeatletter
\def\thickhline{%
  \noalign{\ifnum0=`}\fi\hrule \@height \thickarrayrulewidth \futurelet
   \reserved@a\@xthickhline}
\def\@xthickhline{\ifx\reserved@a\thickhline
               \vskip\doublerulesep
               \vskip-\thickarrayrulewidth
             \fi
      \ifnum0=`{\fi}}
\makeatother

\newlength{\thickarrayrulewidth}
\newtheorem{lemma}{Lemma}[section]
\newtheorem{theorem}[lemma]{Theorem}
\newtheorem{corollary}[lemma]{Corollary}

\newtheorem{definition}{Definition}

\algnewcommand{\IfThenElse}[3]{
\State \algorithmicif\ #1\ \algorithmicthen\ #2\ \algorithmicelse\ #3}

\newcommand{\rbkwx}{\rvx^{\gets}}
\newcommand{\bkwx}{\vx^{\gets}}
\newcommand{\rbkwv}{\rvv^{\gets}}
\newcommand{\bkwp}{p^{\gets}}
\newcommand{\ourmethod}{\text{RS-DMC}}

%% file: 0_contents/abstract.tex

To sample from a general target distribution $p_*\propto e^{-f_*}$ beyond the isoperimetric condition, \citet{huang2023monte} proposed to perform sampling through reverse diffusion, giving rise to \textit{Diffusion-based  Monte Carlo} (DMC). Specifically,  DMC follows the reverse SDE of a diffusion process that transforms the target distribution to the standard Gaussian, utilizing a non-parametric score estimation. However, the original DMC algorithm encountered high gradient complexity\footnote{We denote gradient complexity as the required number of gradient calculations to achieve at most $\epsilon$ sampling error.}, resulting in an \textit{exponential dependency} on the error tolerance $\epsilon$ of the obtained samples. In this paper, we demonstrate that 
the high complexity of the original DMC algorithm originates from its redundant design of score estimation, and proposed a  more efficient DMC algorithm, called $\ourmethod$, based on a novel recursive score estimation method. In particular, we first divide the entire diffusion process into multiple segments and then formulate the score estimation step (at any time step) as a series of interconnected mean estimation and sampling subproblems accordingly, which are correlated in a recursive manner. Importantly, we show that with a proper design of the segment decomposition, all sampling subproblems will only need to tackle a strongly log-concave distribution, which can be very efficient to solve using the standard sampler (e.g., Langevin Monte Carlo) with a provably rapid convergence rate. As a result, we prove that the gradient complexity of $\ourmethod$ only has a \textit{quasi-polynomial dependency} on $\epsilon$, which significantly improves exponential gradient complexity in \citet{huang2023monte}. 
Furthermore, under commonly used dissipative conditions, our algorithm is provably much faster than the popular Langevin-based algorithms. Our algorithm design and theoretical framework illuminate a novel direction for addressing sampling problems, which could be of broader applicability in the community.

%% file: 0_contents/010introduction.tex
\section{Introduction}

Sampling problems, i.e., generating samples from a given target distribution $p_*\propto \exp(-f_*)$, have received increasing attention in recent years. For resolving this problem, a popular option is to apply gradient-based Markov chain Monte Carlo (MCMC) methods, such as Unadjusted Langevin Algorithms (ULA)~\citep{neal1992bayesian,roberts1996exponential}, Underdamped Langevin Dynamics (ULD)~\citep{cheng2018underdamped, ma2021there,mou2021high}, Metropolis-Adjusted Langevin Algorithm (MALA)~\citep{Roberts2014Langevin,Xifara2014Langevin}, and Hamiltonian Monte Carlo (HMC)~\citep{Duane1987Hybrid,Neal2010MCMC}. In particular, these algorithms can be seen as the discretization of the continuous Langevin dynamics (LD) and its variants~\citep{ma2015complete}, which will converge to a unique stationary distribution that follows $p_*\propto \exp(-f_*)$, under regularity conditions on the energy function $f_*(\vx)$ \citep{roberts1996exponential}.

However, the convergence rate of the Langevin-based algorithms heavily depends on the target distribution $p_*$: guaranteeing the convergence in polynomial time requiring $p_*$ to have some nice properties, e.g., being strongly log-concave, satisfying log-Sobolev or Poincar\'e inequality with a large coefficient. However, for more general non-log-concave distributions, the convergence rate may exponentially
depend on the problem dimension \citep{raginsky2017non,holzmuller2023convergence} (i.e., $\sim \exp(d)$), or even the convergence itself (to $p_*$) cannot be guaranteed (one can only guarantee to converge to some locally stationary distribution~\citep{balasubramanian2022towards}), implying that the Langevin-based algorithms are extremely inefficient for solving such hard sampling problems.
To this end, we are interested in addressing the following question:
\begin{quote}
\emph{Can we develop a new sampling algorithm that enjoys a non-exponential convergence rate for sampling general non-log-concave distributions?  } 
\end{quote}

To address this problem, we are inspired by several recent studies, including~\cite{montanari2023sampling, huang2023monte}, that attempt to design samplers based on diffusion models \citep{sohl2015deep,ho2020denoising,vargas2023denoising}, which we refer to as the \textbf{diffusion-based Monte Carlo} (DMC). In particular, the algorithm developed in \citet{huang2023monte} is based on the reverse process of the Ornstein-Uhlenbeck (OU) process, which starts from the target distribution $p_*$ and converges to a standard Gaussian distribution. The mathematical formula of the OU process and its reverse process are given as follows \citep{anderson1982reverse,song2020score}:
\begin{align}
\der \rvx_t &= -\rvx_t \der t + \sqrt{2}\der B_t,\quad \rvx_0\sim p_0(\vx) = p_*,\tag{OU Process}\\
\der\rbkwx_{t} &= \big[\rbkwx_t + 2\nabla \log p_{T-t}(\rbkwx_t)\big]\der t + \sqrt{2}\der B_t, \quad \rbkwx_0\sim p_T(\vx)\approx \mathcal N(\boldsymbol{0},\mathbf I), \label{eq:reverse_process}\tag{Reverse Process}
\end{align}
where $B_t$ denotes the Brownian term, $p_t(\vx)$ denotes the underlying distribution of the particle at time $t$ along the OU process, $T$ denotes the end time of the OU process,
and $\nabla \log p_t(\vx)$ denotes the score function of the distribution $p_t(\vx)$. In fact, the exponentially slow convergence rate of the Langevin-based algorithms stems from the rather long mixing time of Langevin dynamics to its stationary distribution, while in contrast, the OU process exhibits a much shorter mixing time. Therefore, principally, if the reverse process of the OU process can be perfectly recovered, one can avoid suffering from the issue of slow mixing of Langevin dynamics, and develop more efficient sampling algorithms accordingly.


Then, the key to recovering \eqref{eq:reverse_process} is to obtain a good estimation for the score $\nabla \log p_{t}(\vx)$ for all $t\in[0,T]$. \citet{huang2023monte} proposed a score estimation method called reverse diffusion sampling (RDS) based on an inner-loop ULA.
However, it still suffers from the exponential dependency with respect to the target sampling error, which requires $\exp\big(\mathcal O(1/\epsilon)\big)$ gradient complexity to achieve the $\epsilon$ sampling error in KL divergence.  The reason behind this is that RDS involves many \textit{hard} subproblems that need to sample non-log-concave distributions with bad isoperimetric properties, which incurs huge gradient complexities in the desired Langevin algorithms.






In this work, we argue that the \textit{hard} subproblems in \citet{huang2023monte} are redundant or even unnecessary, and propose a more efficient diffusion-based Monte Carlo method, called recursive score DMC (RS-DMC), that only requires \textbf{quasi-polynomial gradient complexity}  to sampling general non-log-concave distributions. At the core of RS-DMC is a novel non-parametric method for score estimation, which involves a series of interconnected mean estimation
and sampling subproblems that are correlated in a recursive manner.
In particular,  we first divide the entire forward process into several segments starting from $0,S,\dots,(K-1)S$, and estimate the scores $\{\nabla \log p_{kS}(\vx)\}_{k=0,\dots,K-1}$ recursively. Given the segments, the score within each segment $\nabla \log p_{kS+\tau}(\vx)$ will be further estimated according to the reference score $\nabla \log p_{kS}(\vx)$, where $\tau\in[0, S]$ can be arbitrarily chosen. Importantly, given proper configuration of the segment length (i.e., $S$), we can show that all sampling subproblems in the developed score estimation method are \textit{much easier}, as long as the target distribution $p_*$ is log-smooth and has bounded second moment. Then, all intermediate target distributions are guaranteed to be strongly log-concave, which can be sampled very efficiently via standard ULA. Accordingly, based o n the samples generated via ULA, the mean estimation subproblems can be then resolved very efficiently under some mild assumptions on the tail of the posterior distribution (e.g., moment bounds).   We  summarize the main contributions of this paper as follows:

\if
To address the problem mentioned above, we propose a novel non-parametric method for score estimation. 
We divide the entire forward process into several segments and estimate the score at time $t$ using Tweedie’s formula.
In the formula, the conditional probability is the transition kernel of the forward process (i.e., OU process) from the starting point of the segment $t$ located at time $t^\prime$ to $t$, rather than starting from the beginning of the entire process (at time $0$) to $t$, as done in~\cite{huang2023monte}.
Compared with the time $0$, time $t^\prime$ will be much closer to $t$, making the estimation much easier.
For exchange, to use Tweedie's formula, we need to estimate the score at time $t^\prime$ because it does not have an analytical form like the score at time $0$, which is given by the target distribution as $\grad\ln p_*$.
This leads us to introduce a recursive method to estimate the score at time $t$, $t^\prime$, etc.
However, by choosing the segment length appropriately, we can trade off the complexity of the score estimation at time $t$ and its corresponding recursive calls and finally obtain a better gradient complexity.
Surprisingly, when the segment length is some constant, our estimation will lead to a Kullback–Leibler (KL) convergence that only requires the second moment bound and log density smoothness rather than any isoperimetry of the target distribution $p_*$, including weak Poincar\'e inequality.
Furthermore, we prove that the gradient complexity required to achieve the KL convergence does not have exponential dependency for both the dimensionality and approximate error.
Some additional works share nearly the same assumption with ours, e.g.,~\cite{balasubramanian2022towards}, while they can only provide the relative Fisher information convergence. 
To the best of our knowledge, this is the first work for Monte Carlo algorithms to KL converge beyond isoperimetry and even with a better dimensional dependence in gradient complexity compared with ULA in general dissipative settings~\cite{hale2010asymptotic,vempala2019rapid}.
\fi

\begin{itemize}[leftmargin=*]
    \item We propose a new Diffusion Monte Carlo algorithm, called $\ourmethod$, for sampling general non-log-concave distributions. At the core is a novel and efficient recursive score estimation algorithm. In particular, based on a properly designed recursive structure, we show that the hard non-log-concave sampling problem can be divided into a series of benign sampling subproblems that can be solved very efficiently via standard ULA.
    
    \item We establish the convergence guarantee of the proposed $\ourmethod$ algorithm under very mild assumptions, which only require the target distribution to be log-smooth and to have a bounded second moment.
    In contrast, to obtain provable convergence (to the target distribution), the Langevin-based methods typically require additional isoperimetric conditions (e.g., Log-Sobolev inequality, Poincar\'e inequality, etc). This justifies that our algorithm can be applied to a broader class of distributions with rigorous theoretical convergence guarantees. 
    
    \item We prove that the gradient complexity of our algorithm is $\exp\big[\mathcal O(\log^3(d/\epsilon))\big]$ to achieve $\epsilon$ sampling error in KL divergence, which only has a quasi-polynomial dependency on the target error $\epsilon$ and dimension $d$. In contrast, under even stronger conditions in our work, the gradient complexity in prior works either need exponential dependency in $\epsilon$ (i.e., $\exp\big(\mathcal O(1/\epsilon)\big)$) \citep{huang2023monte} or exponential dependency in $d$, (i.e., $\exp\big(\mathcal O(d)\big)$) \citep{raginsky2017non,xu2018global}\footnote{We omit the $d$-dependency in \citet{huang2023monte} and $\epsilon$-dependency in \citet{raginsky2017non,xu2018global} for the ease of presentation.} (which requires the additional dissipative condition). This demonstrate the efficiency of our algorithm.

    
\end{itemize}

%% file: 0_contents/020preliminaries.tex
\section{Preliminaries}
\label{sec:pre}

In this section, we will first introduce the notations and problem settings that are commonly used in the following sections. 
We will then present some fundamental properties, such as the closed form of the transition kernel and the expectation form of score functions along the OU process.
Finally, we will specify the assumptions that the target distribution is required in our algorithms and analysis.

\paragraph{Notations.} 
We use lower case bold symbol $\mathbf x$ to denote the random vector, we use lower case italicized bold symbol $\vx$ to denote a fixed vector. We use $\|\cdot\|$ to denote the standard Euclidean distance. We say $a_n=\mathrm{poly}(n)$ if $a_n\le O(n^c)$ for some constant $c$. 


\paragraph{The segmented OU process.} We define $\mathbb{N}_{a,b}=[a,b]\cap \mathbb{N}_*$ for brevity. 
Suppose the length of each segment is $S\in \R_+$, and we divide the entire forward process with length $T$ into $K\in \mathbb{N}_+$ segments satisfying $K=T/S$.
In this condition, we can reformulate the previous SDE as
\begin{equation}
    \label{con_eq:rrds_forward}
    \begin{aligned}
        &\rvx_{k,0}\sim p_{0,0} = p_{*}\  \mathrm{when}\ k=0,\ \mathrm{else}\ \rvx_{k,0}=\rvx_{k-1,S} && k\in \mathbb{N}_{0,K-1}\\
        & \der \rvx_{k,t} =-\rvx_{k,t}\der t + \sqrt{2}\der B_{t}\quad &&k\in \mathbb{N}_{0,K-1}, t\in[0,S],\\
    \end{aligned}
\end{equation}
where $\rvx_{k,t}$ denotes the random variable of the OU process at time $(kS+t)$ with underlying density $p_{k,t}$.
Besides, we define the following conditional density, i.e., $p_{(k,t)|(k^\prime, t^\prime)}(\vx|\vx^\prime)$,
which presents the probability of obtaining $\rvx_{k,t}=\vx$ when $\rvx_{k^\prime, t^\prime}=\vx^\prime$. The diagram of SDE~\eqref{con_eq:rrds_forward} is presented in Fig~\ref{con_fig:rrds_forward}.

\paragraph{The reverse segmented OU process.}
According to \eqref{eq:reverse_process}, the reverse process of the segmented SDE~\eqref{con_eq:rrds_forward} can be presented as
\begin{equation*}
    \begin{aligned}
        &\rbkwx_{k,0} \sim p_{K-1,S}\ \mathrm{when}\ k=K-1,\ \mathrm{else}\ \rbkwx_{k,0}=\rbkwx_{k+1,S} && k\in \mathbb{N}_{0,K-1}\\
        &\der \rbkwx_{k,t} = \left[ \rbkwx_{k,t} + 2\grad\log p_{k, S-t}(\rbkwx_{k,t})\right]\der t + \sqrt{2}\der B_t\quad && k\in \mathbb{N}_{0,K-1}, t\in[0,S]\\
    \end{aligned}
\end{equation*}
where particles satisfy $\rbkwx_{k,t}=\rvx_{k,S-t}$ with underlying density $\bkwp_{k,t}=p_{k,S-t}$ for any $k\in\mathbb{N}_{0,K-1}$ and $t\in[0,S]$.
To approximately solve the SDE with numerical methods, we first split each segment into $R$ intervals $\{[(r-1)\eta, r\eta]\}_{=1,\dots,R}$, where $\eta$ is the interval length and $R = S/\eta$. Then we can 
replace the score function $\grad\log p_{k,S-t}$ as $\rbkwv_{k, t}$, and for $t\in [r\eta, (r+1)\eta]$, we freeze the value of this coefficient in the SDE at time $(k, r\eta)$ . 
Then starting from the standard Gaussian distribution, we consider the following new SDE:
 \begin{equation}
    \label{con_eq:rrds_actual_backward}
    \begin{aligned}
        &\rbkwx_{k,0} \sim p_{\infty}=\mathcal{N}(\vzero, \mI)\ \mathrm{when}\ k=K-1,\ \mathrm{else}\ \rbkwx_{k,0}=\rbkwx_{k+1,S} && k\in \mathbb{N}_{0,K-1}\\
        &\der \rbkwx_{k,t} = \left[\rbkwx_{k,t} + 2\rbkwv
        _{k, \lfloor t/\eta \rfloor\eta}\left(\rbkwx_{k, \lfloor t/\eta \rfloor \eta}\right)\right]\der t + \sqrt{2}\der B_t\quad && k\in \mathbb{N}_{0,K-1}, t\in[0,S]\\
    \end{aligned}
\end{equation}
where $p_\infty$ denotes the stationary distribution of the forward process. 
Similar to the segmented OU process, we define the following conditional density, i.e., $\bkwp_{(k,t)|(k^\prime, t^\prime)}(\vx|\vx^\prime)$,
which presents the probability of obtaining $\rbkwx_{k,t}=\vx$ when $\rbkwx_{k^\prime, t^\prime}=\vx^\prime$. The diagram of SDE~\eqref{con_eq:rrds_actual_backward} is presented in Fig~\ref{con_fig:rrds_forward}.

\begin{figure}
    \centering
    \input{0_figs/notation_explain}
    \vspace{-.1in}
    \caption{\small The illustration of  SDE~\eqref{con_eq:rrds_forward} and~\eqref{con_eq:rrds_actual_backward}, covering  the definitions in Section~\ref{sec:pre}. The top of the figure describes the underlying distribution of the segmented OU process, i.e., SDE~\eqref{con_eq:rrds_forward}, and the bottom presents the corresponding distribution in the segmented OU process, i.e., SDE~\eqref{con_eq:rrds_actual_backward}.
    For the intermediate part, the upper half describes the gradients of the log densities along the forward SDE~\eqref{con_eq:rrds_forward}, while the lower half describes 
    approximated scores used to update particles in the reverse SDE~\eqref{con_eq:rrds_actual_backward}. }
    \label{con_fig:rrds_forward}
    \vspace{-.1in}
\end{figure}
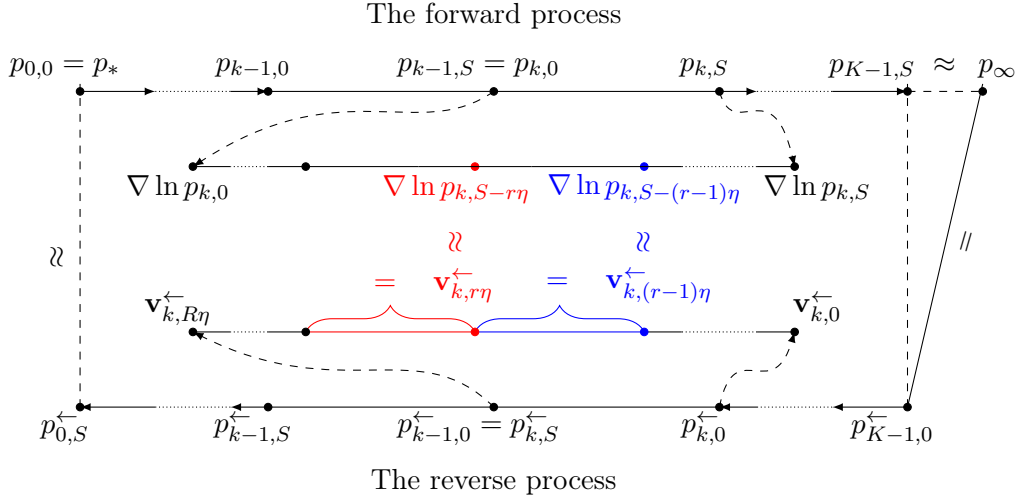

\paragraph{Basic properties of the OU process.} 
In the previous paragraph, we have demonstrated that SDE~\eqref{con_eq:rrds_forward} is an alternative presentation of the OU process.
Therefore, the properties in the OU process can be directly introduced for this segmented version.
First, the transition kernel in the $k$-th segment satisfies
\begin{equation*}
    p_{(k,t)|(k,0)}(\vx|\vx_0)= \left(2\pi \left(1-e^{-2t}\right)\right)^{-d/2}
     \cdot \exp \left[\frac{-\left\|\vx -e^{-t}\vx_0 \right\|^2}{2\left(1-e^{-2t}\right)}\right], \quad \forall\ 0< t\le S.
\end{equation*}
Plugging the transition kernel into Tweedie’s formula, the score function can be reformulated as the following lemma whose proof is deferred in Appendix~\ref{app_sec:lems_rec_ite}.
\begin{lemma}[Lemma 1 of \cite{huang2023monte}]
    \label{con_lem:score_reformul}
    For any $k\in \mathbb{N}_{0,K-1}$ and $t\in[0,S]$, the score function can be written as
    \begin{equation*}
        \grad \log p_{k,S-t}(\vx) = \mathbb{E}_{\rvx_0 \sim q_{k,S-t}(\cdot|\vx)}\left[-\frac{\vx - e^{-(S-t)}\vx_0}{\left(1-e^{-2(S-t)}\right)}\right]
    \end{equation*}
    where the conditional density function $q_{k,S-t}(\cdot | \vx)$ is defined as
    \begin{equation*}
        q_{k,S-t}(\vx_0|\vx) \propto \exp\left(\log p_{k,0}(\vx_0)-\frac{\left\|\vx - e^{-(S-t)}\vx_0\right\|^2}{2\left(1-e^{-2(S-t)}\right)}\right).
    \end{equation*}
\end{lemma}
Therefore, to approximate the score $\grad\log p_{k,S-r\eta}(\vx)$ with an estimator $\rbkwv_{k, r\eta}(\vx)$, we can draw samples from $q_{k, S- r\eta}(\cdot|\vx)$ and calculate their empirical mean.



\paragraph{Assumptions.} To guarantee the convergence in KL divergence, the Langevin-based methods require the target distribution to satisfy certain isoperimetric properties such as Log-Sobolev inequality (LSI) and  Poincar\'e inequality (PI) or even strong log-concavity~\citep{vempala2019rapid, cheng2018convergence, dwivedi2018log,ma2019sampling,zou2019stochastic,zou2021faster} (the formal definitions of these conditions are deferred to~Appendix~\ref{sec:not_ass_app}).
Some other works consider milder assumptions such as modified LSI~\citep{erdogdu2021convergence} and weak Poincar\'e inequality~\citep{mousavi2023towards}, but they are only the analytical continuation of LSI and PI, which still exhibit a huge gap with the general non-log-concave distributions. \citet{huang2023monte} requires the  target distribution $p_*$ to have a heavier tail than that of the Gaussian distribution.

Remarkably, our algorithm does not require any isoperimetric condition or condition on the tail properties of $p_*$ to establish the convergence guarantee.  We only require the following mild conditions on the target distribution.

\begin{enumerate}[label=\textbf{[{A}{\arabic*}]}]
    \item  
    \label{con_ass:lips_score} For any $k\in\mathbb{N}_{0,K-1}$ and $t\in [0,S]$, the score $\grad \log p_{k,t}$ is $L$-Lipschitz.
    \item \label{con_ass:second_moment} The target distribution has a bounded second moment, i.e., $M\coloneqq \mathbb{E}_{p_*}[\left\|\cdot\right\|^{2}]<\infty$.
\end{enumerate}
Assumption~\ref{con_ass:lips_score} corresponds to the $L$-smoothness condition of the log density $f_*$ in traditional ULA analysis, which has been widely made in prior works \citep{chen2022sampling,chen2023restoration,huang2023monte}.
It is often used to ensure that numerical discretization is feasible. We emphasize that Assumption~\ref{con_ass:lips_score} can be relaxed to only assume the target distribution is smooth rather than the entire OU process, based on the technique in \cite{chen2023improved} (see rough calculations in their Lemmas 12 and 14).
We do not include this additional relaxation in this paper to make our analysis clearer.
Assumption~\ref{con_ass:second_moment} is one of the weakest assumptions being adopted for the analyses of posterior sampling.

%% file: 0_figs/notation_explain.tex
\usetikzlibrary{decorations.pathreplacing}

\begin{tikzpicture}

\node at (5,5.2) {The forward process};

\node at (5,-1) {The reverse process};

\filldraw (-0.5,4.2) circle (0.05); \node at (-0.7, 4.5) {$p_{0,0}=p_*$}; 
\draw[-latex]  (-0.5,4.2)--(0.5,4.2);
\draw[densely dotted] (0.5,4.2)--(1.5,4.2);
\draw[-latex]  (1.5,4.2)--(2,4.2);
\filldraw (2,4.2) circle (0.05); \node at (1.8, 4.5) {$p_{k-1,0}$}; 
\draw[-] (2,4.2)--(5,4.2);
\filldraw (5,4.2) circle (0.05); \node at (4.8, 4.5) {$p_{k-1,S}=p_{k,0}$}; 
\draw[-] (5,4.2)--(8,4.2);
\filldraw (8,4.2) circle (0.05); \node at (7.8, 4.5) {$p_{k,S}$}; 
\draw[-latex]  (8,4.2)--(8.5,4.2);
\draw[densely dotted] (8.5,4.2)--(9.5,4.2);
\draw[-latex]  (9.5,4.2)--(10.5,4.2);
\filldraw (10.5,4.2) circle (0.05); \node at (10.3, 4.5) {$p_{K-1,S}\ \approx$}; 
\draw[dashed] (10.5,4.2)--(11.5,4.2);
\filldraw (11.5,4.2) circle (0.05); \node at (11.7, 4.5) {$p_\infty$};

\draw[-latex, dashed] (5,4.2)..controls (4, 3.5)and(3,4.5)..(1,3.2);
\draw[-latex, dashed] (8,4.2)..controls (8.2, 3.5)and(8.7,4.5)..(9,3.2);

\filldraw (-0.5,0) circle (0.05); \node at (-0.7, -0.3) {$\bkwp_{0,S}$}; 
\draw[-latex]  (0.5,0)--(-0.5,0);
\draw[densely dotted] (0.5,0)--(1.5,0);
\draw[-latex]  (2,0)--(1.5,0);
\filldraw (2,0) circle (0.05); \node at (1.8, -0.3) {$\bkwp_{k-1,S}$}; 
\draw[-] (2,0)--(5,0);
\filldraw (5,0) circle (0.05); \node at (4.8, -0.3) {$\bkwp_{k-1,0}=\bkwp_{k,S}$}; 
\draw[-] (5,0)--(8,0);
\filldraw (8,0) circle (0.05); \node at (7.8, -0.3) {$\bkwp_{k,0}$}; 
\draw[-latex]  (8.5,0)--(8,0);
\draw[densely dotted] (8.5,0)--(9.5,0);
\draw[-latex]  (10.5,0)--(9.5,0);
\filldraw (10.5,0) circle (0.05); \node at (10.3, -0.3) {$\bkwp_{K-1,0}$};

\draw[dashed] (-0.5,0)--(-0.5,4.3);
\draw[dashed] (10.5,0)--(10.5,4.3);
\node [rotate=90] at (-0.8,2) {$\approx$};
\draw[-] (11.5,4.3)--(10.5,0);
\node [rotate=74] at (11.3,2.2) {$=$};

\draw[-latex, dashed] (5,0)..controls (4, 1)and(3,0)..(1,1);
\draw[-latex, dashed] (8,0)..controls (8.2, 1)and(8.7,0)..(9,1);
\filldraw (1,1) circle (0.05); \node at (0.8, 1.3) {$\rbkwv_{k,R\eta}$}; 
\filldraw (9,1) circle (0.05); \node at (9.3, 1.3) {$\rbkwv_{k,0}$}; 
\draw (9,1)--(8.5,1);
\draw[densely dotted] (8.5,1)--(7.5,1);
\draw (7.5,1)--(7,1);
\filldraw[blue] (7,1) circle (0.05);\node at (6.8, 1.6) {$\textcolor{blue}{=\quad \rbkwv_{k,(r-1)\eta}}$}; 
\draw[blue] (7,1)--(4.75,1);
\draw [decorate,decoration={brace,amplitude=10pt}, blue] (4.75,1) -- (7,1);
\draw[red] (4.75,1)--(2.5,1);
\draw [decorate,decoration={brace,amplitude=10pt, mirror}, red] (4.75,1) -- (2.5,1);
\filldraw[red] (4.75,1) circle (0.05);
\node at (4.2, 1.6) {$\textcolor{red}{=\quad\rbkwv_{k,r\eta}}$}; 
\filldraw (2.5,1) circle (0.05);
\draw (2.5,1)--(2,1);
\draw[densely dotted] (2,1)--(1.5,1);
\draw (1.5,1)--(1,1);

\filldraw (1,3.2) circle (0.05); \node at (0.8, 2.9) {$\grad\ln p_{k,0}$}; 
\filldraw (9,3.2) circle (0.05); \node at (9.3, 2.9) {$\grad\ln p_{k,S}$}; 
\draw (8.5,3.2)--(9,3.2);
\draw[densely dotted] (7.5,3.2)--(8.5,3.2);
\filldraw[blue] (7,3.2) circle (0.05); \node at (7, 2.9) {$\textcolor{blue}{\grad\ln p_{k, S-(r-1)\eta}}$}; 
\draw (7.5,3.2)--(5.5,3.2);
\filldraw[red] (4.75,3.2) circle (0.05); \node at (4.5, 2.9) {$\textcolor{red}{\grad\ln p_{k, S-r\eta}}$};
\draw (5.5,3.2)--(4,3.2);
\filldraw (2.5,3.2) circle (0.05);
\draw (4,3.2)--(2,3.2);
\draw[densely dotted] (2,3.2)--(1.5,3.2);
\draw (1.5,3.2)--(1,3.2);

\node[rotate=90] at (7,2.2) {$\textcolor{blue}{\approx}$};

\node[rotate=90] at (4.5,2.2) {$\textcolor{red}{\approx}$};

\end{tikzpicture}

%% file: 0_contents/030methods.tex
\section{Proposed Methods}
\label{sec:method}

In this section, we introduce a new approach called Recursive Score Estimation (RSE) and describe the proposed Recursive Score Diffusion-based Monte Carlo (RS-DMC) method.
We start by discussing the motivations and intuitions behind the use of recursion.
Next, we provide implementation details for the RSE process and emphasize the importance of selecting an appropriate segment length.
Finally, we present the RS-DMC method based on the RSE approach.

\subsection{Difficulties of the vanilla DMC}

We consider the reverse segmented OU process, i.e., SDE~\ref{con_eq:rrds_actual_backward} and begin with the original version of DMC in \citet{huang2023monte}, which can be seen as a special case of the reverse segmented OU process with a large segment length $S=T$ and a small number of segments $K=1$.
According to the reverse SDE~\ref{con_eq:rrds_actual_backward}, for the $r$-th iteration within one single segment, we need to estimate $\grad\log p_{0,S-r\eta}$ to update the particles.
Specifically, by Lemma~\ref{con_lem:score_reformul}, we have
\begin{equation*}
    \grad \log p_{0,S-r\eta}(\vx) = \mathbb{E}_{\rvx_0 \sim q_{0, S-r\eta}(\cdot|\vx)}\left[-\frac{\vx - e^{-(S-r\eta)}\vx_0}{\left(1-e^{-2(S-r\eta)}\right)}\right]
\end{equation*}
for any $\vx\in\R^d$, where the conditional distribution $q_{0, S-r\eta}(\cdot |\vx)$ is
\begin{equation}
    \label{eq:tar_axu_dis}
    \small
    q_{0,S-r\eta}(\vx_0|\vx) \propto \exp\bigg( \log p_{0,0}(\vx_0)-\frac{\left\|\vx - e^{-(S-r\eta)}\vx_0\right\|^2}{2\left(1-e^{-2(S-r\eta)}\right)}\bigg).
\end{equation}
Since the analytic form $\grad\log p_{0, 0}=-f_*$ exists, we can use the ULA to draw samples from $q_{0, S-r\eta}(\cdot|\vx)$ and calculate the empirical mean to estimate $\grad\log p_{0,S-r\eta}(\vx)$.

However, sampling from $q_{0, S-r\eta}(\cdot|\vx)$ is not an easy task. When $r$ is very small, sampling $q_{0,S-r\eta}(\cdot|\vx)$ via ULA is almost as difficult as sampling $p_{0,0}(\vx_0)$ via ULA (see \eqref{eq:tar_axu_dis}), since the additive quadratic term, whose coefficient is $e^{-2(S-r\eta)}/2(1-e^{-2(S-r\eta)})$, will be nearly negligible in this case. This is because that $S=T$ is large and then $e^{-2(S-r\eta)}/2(1-e^{-2(S-r\eta)})\sim \exp(-2T)$ becomes extremely small when $r\eta=O(T)$.
More specifically, as shown in \citet{huang2023monte}, when $e^{-2(S-r\eta)}\le 2L/(1+2L)$, the LSI parameter of $q_{0,S-r\eta}(\cdot|\vx)$ can be as worse as $\exp\big(-\mathcal O(1/\epsilon)\big)$. Then applying ULA for sampling this distribution needs a dramatically high gradient complexity that is exponential in $1/\epsilon$.


\subsection{Intuition of the recursion}

Therefore, the key to avoiding sampling such a hard distribution is to restrict the segment length. By Lemma \ref{con_lem:score_reformul}, it can be straightforwardly verified that if the segment length satisfies $S\le \frac{1}{2}\log\left(\frac{2L+1}{2L}\right)$, 
\begin{equation}
    \label{con_ineq:cho_s_small}
    -\grad_{\vx_0}^2\log q_{k,S-r\eta}(\vx_0|\vx) \succeq -\grad^2_{\vx_0} \log p_{k,0}(\vx_0)+ \frac{e^{-2S}}{1-e^{-2S}}\cdot \mI \succeq \frac{e^{-2S}}{2(1-e^{-2S})}
\end{equation}
where the last inequality follows from Assumption~\ref{con_ass:lips_score}.
This implies that $q_{k,S-r\eta}(\vx_0|\vx)$ is strongly log-concave for all $r\le \lfloor S/\eta\rfloor$, which can be efficiently sampled via the standard ULA. However, ULA requires to calculate the score function  $\nabla_{\vx_0} \log q_{k,S-r\eta}(\vx_0|\vx)$, which further needs to calculate $\nabla\log p_{k,0}(\vx)$ according to Lemma \ref{con_lem:score_reformul}. Different from the vanilla DMC where the formula of $\nabla \log p_{0,0}(\vx)$ is known, the score $\nabla\log p_{k,0}(\vx)$ in \eqref{con_ineq:cho_s_small} is an unknown quantity, which also requires to be estimated. In fact, based on our definition, we can rewrite $p_{k,0}(\vx)$ as $p_{k-1,S}(\vx)$ (see Figure \ref{con_fig:rrds_forward}), then applying Lemma \ref{con_lem:score_reformul}, we can again decompose the problem of estimating $\nabla \log p_{k-1,S}(\vx)$ into the subproblems of sampling $q_{k-1, S}(\cdot|\vx)$ and the estimation of $\nabla \log p_{k-1,0}(\vx)$, which is naturally organized in a recursive manner. Therefore, by recursively adopting this subproblem decomposition, we summarize the recursive process for approximating $\grad\log p_{k,S-r\eta}(\vx)$ as follows and illustrate the diagram in Figure \ref{con_fig:rse}: 

\begin{itemize}[leftmargin=*,nosep]
    \item \textbf{Step 1:} We approximate the score $\grad\log p_{k,S-r\eta}(\vx)$ by a mean estimation with samples generated by running ULA over the intermediate target distribution $q_{k, S-r\eta}(\cdot|\vx)$.
    \item \textbf{Step 2:} When running ULA for $q_{k, S-t}(\cdot|\vx)$, we estimate the score $\grad\log p_{k,0} = \grad\log p_{k-1,S}$.
    \item \textbf{Step 3:} We jump to Step~1 to approximate the score $\grad\log p_{k-1,S}(\vx)$ via drawing samples from $q_{k-1, S}(\cdot|\vx)$, and continue the recursion.
\end{itemize}
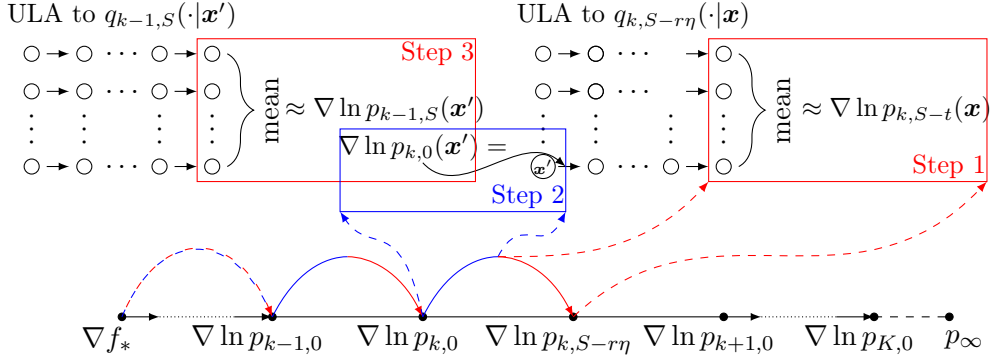
\begin{figure}
    \centering
    \input{0_figs/recursion_explain}
    \vspace{-.1in}
    \caption{\small The illustration of recursive score estimation (RSE).
    The upper half presents RSE from a local view, which shows how to utilize the former score, e.g., $\grad\log p_{k,0}(\vx^\prime)$ to update particles by ULA in the sampling subproblem formulated by the latter score, e.g., $\grad\log p_{k,S-t}(\vx)$. The lower half presents RSE from a global view, which is a series of interconnected mean estimation and sampling subproblems accordingly.}
    \label{con_fig:rse}
    \vspace{-.1in}
\end{figure}

\subsection{Recursive Score Estimation and Reverse Diffusion Sampling}

\paragraph{Recursive Score Estimation.}
In the previous section, we explained the rough intuition behind introducing recursion. By conducting the recursion, we need to solve a series of sampling and mean estimation subproblems. Then, it is demanding to control the error propagation between these subproblems in order to finally ensure small sampling errors. In particular, this amounts to the adaptive adjustment of the sample numbers for mean estimation and iteration numbers for ULA in solving sampling subproblems.
Specifically, if we require score estimation $\rbkwv_{k,r\eta}\colon \R^d\rightarrow \R^d$ to satisfy
\begin{equation}
    \label{ineq:score_req}
    \left\|\grad\log p_{k,S-r\eta}(\vx) - \rbkwv_{k,r\eta}(\vx)\right\|^2\le \epsilon,\ \forall \vx\in\R^d
\end{equation}
with a high probability, then the sample number in Step 1 and the number of calls of Step 2 (the iteration number of ULA) in Fig~\ref{con_fig:rse} will be 
two functions with respected to the target error $\epsilon$, denoted as $n_{k,r}(\epsilon)$ and $m_{k,r}(\epsilon)$ respectively.
Furthermore, when Step 2 is introduced to update ULA, we rely on an approximation of $\grad\log p_{k,0}$ instead of the exact score. 
To ensure \eqref{ineq:score_req} is met, the error resulting from estimating $\grad\log p_{k,0}$ should be typically smaller than $\epsilon$. We express this requirement as:
\begin{equation*}
    \left\|\grad\log p_{k,0}(\vx) - \rbkwv_{k,0}(\vx)\right\|^2\le l_{k,r}(\epsilon),\ \forall \vx\in\R^d.
\end{equation*}
where $l_{k,r}(\epsilon)$ is a function of $\epsilon$ that satisfies $l_{k,r}(\epsilon)\le \epsilon$.
Under this condition, we provide Alg~\ref{con_alg:rse}, i.e., RSE, to calculate the score function for the $r$-th iteration at the $k$-th segment, i.e., $\grad\log p_{k,S-r\eta}(\vx)$.

\begin{algorithm}[t]
    \caption{Recursive Score Estimation (approximate $\grad\log p_{k,S-r\eta}(\vx)$): $\mathsf{RSE}(k,r,\vx,\epsilon)$}
    \label{con_alg:rse}
    \begin{algorithmic}[1]
            \State {\bfseries Input:} The segment number $k\in \mathbb{N}_{0, K-1}$, the iteration number $r\in\mathbb{N}_{0,R-1}$, variable $\vx$ requiring the score function, error tolerance $\epsilon$.
            \If{$k\equiv -1$} \label{con_alg:rse_score_ret_condi}
                \Return $-\grad f_*(\vx)$
            \EndIf
            \State Initial the returned vector $\vv^\prime \gets \vzero$ 
            \For{$i=1$ to $n_{k,r}(\epsilon)$}\label{con_alg:rse_outer_loops}
                \State Draw $\vx^\prime_0$ from an initial distribution $q^\prime_0$ 
                \For{$j=0$ to $m_{k,r}(\epsilon,\vx)-1$} \label{con_alg:rse_inner_loops}
                    \State $\vv^\prime_j \gets \mathsf{RSE}\left(k-1, 0, \vx^\prime_{j}, l_{k,r}(\epsilon)\right)$ \Comment{Recursive score estimation $\grad\log p_{k-1,S}(\vx^\prime_{j})$}
                    \IfThenElse{$r\not\equiv 0$}{$t^\prime \gets S-r\eta$}{$t^\prime \gets S$} \Comment{The gap of time since the last call}
                    \State  \label{con_alg:rse_inner_update} Update the particle 
                        \begin{equation*}
                            \vx^\prime_{j+1}\coloneqq \vx^\prime_{j}+\tau_r\cdot \underbrace{\bigg(\vv^\prime_j + \frac{e^{-t^\prime}\vx - e^{-2t^\prime}\vx^\prime_{j}}{1-e^{-2t^\prime}} \bigg)}_{\approx \grad\log q_{k,S-r\eta}(\vx^\prime_{j}|\vx)} +\sqrt{2\tau_r}\cdot \xi
                        \end{equation*}
                \EndFor
                \State \label{con_alg:rse_score_update} Update the score estimation of $\vv^\prime\approx \grad\log p_{k,S-r\eta}(\vx)$ with empirical mean as $$\vv^\prime \coloneqq \vv^\prime + \frac{1}{n_{k,r}(\epsilon)}\bigg(-\frac{\vx - e^{-t^\prime}\vx_{m_{k,r}(\epsilon)}^\prime}{1-e^{-2t^\prime}}\bigg)$$
            \EndFor        
    \Return $\vv^\prime$. \Comment{As the approximation of $\grad\log p_{k,S-r\eta}(\vx)$}
    \end{algorithmic}
\end{algorithm}

\paragraph{Quasi-polynomial Complexity.} We consider the ideal case for interpreting the complexity of our score estimation method. In particular, since the benign error propagation, i.e., $l_{k,r}(\epsilon)=\epsilon$, is almost proven in Lemma~\ref{lem:recursive_core_lem}, we suppose the number of calls to the recursive function, $\mathsf{RSE}(k-1,0,\vx^\prime,l_{k,r}(\epsilon))$, is uniformly bounded by $m_{k,r}(\epsilon)\cdot n_{k,r}(\epsilon)$ for all feasible $(k,r)$ pairs when the RSE algorithm is executed with input $(k,r,\vx,\epsilon)$.
Then, recall that we will conduct the recursion in at most $K$ rounds, the total gradient complexity for estimating one score will be 
\begin{equation*}
    \left[m_{k,r}(\epsilon)\cdot n_{k,r}(\epsilon)\right]^{\mathcal{O}(K)} = \left[m_{k,r}(\epsilon)\cdot n_{k,r}(\epsilon)\right]^{\mathcal{O}(T/S)}. 
\end{equation*}
This formula highlights the importance of selecting a sufficiently large segment with length $S$ to reduce the number of recursive function calls and improve gradient complexity.
In our analysis, we set $S=\frac{1}{2}\log\left(\frac{2L+1}{2L}\right)$, which is ``just'' small enough to ensure that all intermediate target distributions in the sampling subproblems are strongly log-concave.
In this condition, due to the choice of $T$ is $\mathcal{O}(\log (d/\epsilon))$ in general cases and $m_{k,r}(\cdot)$ and $n_{k,r}(\cdot)$ are typically polynomial w.r.t. the target sampling error $\epsilon$ and dimension $d$ (see Theorem~\ref{thm:main_rrds_formal} in Appendix~\ref{app_sec:main_proof}), we would expect a quasi-polynomial gradient complexity. 

\paragraph{Diffusion-based Monte Carlo with Recursive Score Estimation.}
Then based on the RSE algorithm in Alg \ref{con_alg:rse}, we can directly apply the DDPM \citep{ho2020denoising} based method to perform the sampling, giving rise to the Recursive Score Diffusion-based Monte Carlo (RS-DMC) method. We summarize the proposed RS-DMC algorithm in Alg~\ref{alg:rrds} (the detailed setup of $m_{k,r}(\cdot)$, $n_{k,r}(\cdot)$, $l_{k,r}(\cdot)$ are provided in Theorem~\ref{thm:main_rrds_formal} in Appendix~\ref{app_sec:main_proof}).



\begin{algorithm}
    \caption{Recursive Score Diffusion-based Monte Carlo (RS-DMC)}
    \label{alg:rrds}
    \begin{algorithmic}[1]
            \State {\bfseries Input:} Initial particle $\rbkwx_{K,S}$ sampled from $p_\infty$, Terminal time $T$, Step size $\eta$, required convergence accuracy $\epsilon$;
            \For{$k=K-1$ down to $0$}
                \State Initialize the particle as $\bkwx_{k,0} \gets  \bkwx_{k+1, S}$
                \For{$r=0$ to $R-1$}
                    \State Approximate the score, i.e., $\grad\log p_{k,S-r\eta}(\bkwx_{k,r\eta})$ by $\vv^\prime \gets \mathsf{RSE}(k,r,\bkwx_{k,r\eta},l(\epsilon))$
                    \State $\bkwx_{k, (r+1)\eta} \gets e^\eta\bkwx_{k,r\eta}+\left(e^\eta - 1\right)\vv^\prime +\xi$ where $\xi$  is sampled from  $\mathcal{N}\left(0,  \left(e^{2\eta}-1\right)\mI_d\right)$
                \EndFor
            \EndFor
            \State {\bfseries return} $\bkwx_{0,S}$.
    \end{algorithmic}
\end{algorithm}

%% file: 0_figs/recursion_explain.tex
\begin{tikzpicture}

\node at (10.3,0.25) {\small $\approx\grad\ln p_{k,S-t}(\vx)$};
\draw [decorate,decoration={brace,amplitude=10pt}] (8.2,1) -- (8.2,-0.5);
\node[rotate=90] at (8.8, 0.25) {mean};
\node at (6.8, 1.5) {\small ULA to $q_{k,S-r\eta}(\cdot|\vx)$};

\draw[fill=none](8,1) circle (0.1);
\draw[-latex] (7.5,1)--(7.8,1);
\draw[fill=none](6.3,1) circle (0.1);
\node at (6.8,1) {$\ldots$};
\draw[fill=none](6.3,1) circle (0.1);
\draw[-latex] (5.8,1)--(6.1,1);
\draw[fill=none](5.6,1) circle (0.1);

\draw[fill=none](8,0.5) circle (0.1);
\draw[-latex] (7.5,0.5)--(7.8,0.5);
\draw[fill=none](6.3,0.5) circle (0.1);
\node at (6.8,0.5) {$\ldots$};
\draw[fill=none](6.3,0.5) circle (0.1);
\draw[-latex] (5.8,0.5)--(6.1,0.5);
\draw[fill=none](5.6,0.5) circle (0.1);

\node[rotate=90] at (8,0) {$\ldots$};
\node[rotate=90] at (7.3,0) {$\ldots$};
\node[rotate=90] at (6.3,0) {$\ldots$};
\node[rotate=90] at (5.6,0) {$\ldots$};

\draw[fill=none](8,-0.5) circle (0.1);
\draw[-latex] (7.5,-0.5)--(7.8,-0.5);
\draw[fill=none](7.3,-0.5) circle (0.1);
\node at (6.8,-0.5) {$\ldots$};
\draw[fill=none](6.3,-0.5) circle (0.1);
\draw[-latex] (5.8,-0.5)--(6.1,-0.5);
\node[draw,circle,minimum size=3pt,inner sep=0pt] at (5.6,-0.5) {\tiny $\vx^\prime$};

\node at (3.5,0.25) {\small $\approx\grad\ln p_{k-1,S}(\vx^\prime)$};
\node at (4,-0.25) {\small $\grad\ln p_{k,0}(\vx^\prime)=$};
\draw[-latex] (4,-0.45)..controls (4.5, -1)and(5.3,0)..(5.9,-0.5);
\draw [decorate,decoration={brace,amplitude=10pt}] (1.4,1) -- (1.4,-0.5);
\node[rotate=90] at (2, 0.25) {mean};
\node at (0, 1.5) {\small ULA to $q_{k-1,S}(\cdot|\vx^\prime)$};

\draw[fill=none](1.2,1) circle (0.1);
\draw[-latex] (0.7,1)--(1,1);
\draw[fill=none](0.5,1) circle (0.1);
\node at (0,1) {$\ldots$};
\draw[fill=none](-0.5,1) circle (0.1);
\draw[-latex] (-1,1)--(-0.7,1);
\draw[fill=none](-1.2,1) circle (0.1);

\draw[fill=none](1.2,0.5) circle (0.1);
\draw[-latex] (0.7,0.5)--(1,0.5);
\draw[fill=none](0.5,0.5) circle (0.1);
\node at (0,0.5) {$\ldots$};
\draw[fill=none](-0.5,0.5) circle (0.1);
\draw[-latex] (-1,0.5)--(-0.7,0.5);
\draw[fill=none](-1.2,0.5) circle (0.1);

\node[rotate=90] at (1.2,0) {$\ldots$};
\node[rotate=90] at (0.5,0) {$\ldots$};
\node[rotate=90] at (-0.5,0) {$\ldots$};
\node[rotate=90] at (-1.2,0) {$\ldots$};

\draw[fill=none](1.2,-0.5) circle (0.1);
\draw[-latex] (0.7,-0.5)--(1,-0.5);
\draw[fill=none](0.5,-0.5) circle (0.1);
\node at (0,-0.5) {$\ldots$};
\draw[fill=none](-0.5,-0.5) circle (0.1);
\draw[-latex] (-1,-0.5)--(-0.7,-0.5);
\draw[fill=none](-1.2,-0.5) circle (0.1);

\draw [red] (11.5,1.2) rectangle (7.8,-0.7);
\node [red] at (11, -0.5) {\small Step 1};

\draw [red] (4.7,1.2) rectangle (1,-0.7);
\node [red] at (4.2, 1) {\small Step 3};

\draw [blue] (5.9,0) rectangle (2.9, -1.1);
\node [blue] at (5.4, -0.9) {\small Step 2};

\filldraw (0,-2.5) circle (0.05); \node at (-0.2, -2.8) {$\grad f_*$}; 
\draw[-latex]  (0,-2.5)--(0.5,-2.5);
\draw[densely dotted] (0.5,-2.5)--(1.5,-2.5);
\draw[-latex]  (1.5,-2.5)--(2,-2.5);
\filldraw (2,-2.5) circle (0.05); \node at (1.8, -2.8) {$\grad\ln p_{k-1,0}$}; 
\draw[-] (2,-2.5)--(5,-2.5);
\filldraw (4,-2.5) circle (0.05); \node at (3.8, -2.8) {$\grad\ln p_{k,0}$}; 
\draw[-] (5,-2.5)--(8,-2.5);
\filldraw (6,-2.5) circle (0.05);\node at (5.8, -2.8) {$\grad\ln p_{k,S-r\eta}$};
\filldraw (8,-2.5) circle (0.05); \node at (7.8, -2.8) {$\grad\ln p_{k+1,0}$}; 
\draw[-latex]  (8,-2.5)--(8.5,-2.5);
\draw[densely dotted] (8.5,-2.5)--(9.5,-2.5);
\draw[-latex]  (9.5,-2.5)--(10,-2.5);
\filldraw (10,-2.5) circle (0.05); \node at (9.8, -2.8) {$\grad \ln p_{K,0}$}; 
\draw[dashed] (10,-2.5)--(11, -2.5);
\filldraw (11,-2.5) circle (0.05); \node at (11.2, -2.8) {$p_\infty$}; 

\draw[blue,dash pattern= on 3pt off 5pt] (0,-2.5)..controls (0.5, -1.3)and(1.5,-1.3)..(2,-2.5);
\draw[-latex, red,dash pattern= on 3pt off 5pt,dash phase=4pt] (0,-2.5)..controls (0.5, -1.3)and(1.5,-1.3)..(2,-2.5);

\draw[blue] (2,-2.5)..controls (2.3, -1.9)and(2.7,-1.7)..(3,-1.7);
\draw[-latex, red] (3,-1.7)..controls (3.3, -1.7)and(3.7,-1.9)..(4,-2.5);

\draw[blue] (4,-2.5)..controls (4.3, -1.9)and(4.7,-1.7)..(5,-1.7);
\draw[-latex, red] (5,-1.7)..controls (5.3, -1.7)and(5.7,-1.9)..(6,-2.5);

\draw[-latex, blue, dashed] (5,-1.7)..controls (5.3, -1.1)and(5.6,-1.7)..(5.9,-1.1);

\draw[-latex, blue, dashed] (4,-2.5)..controls (3.7, -1.1)and(3.3,-1.7)..(2.9,-1.1);

\draw[-latex, red, dashed] (5,-1.7)..controls (5.9, -1.7)and(6.9,-1.7)..(7.8,-0.7);

\draw[-latex, red, dashed] (6,-2.5)..controls (7.8, -0.8)and(9.7,-2.6)..(11.5,-0.7);

\end{tikzpicture}

%% file: 0_contents/040analysis.tex
\vspace{-.2in}
\section{Analysis of \ourmethod}
In this section, we will establish the convergence guarantee for RS-DMC and reveal how the gradient complexity depends on the problem dimension and the target sampling error. We will also compare the gradient complexity of RS-DMC with other sampling methods to justify its strength. 
Additionally, we will provide a proof roadmap that briefly summarizes the critical theoretical techniques.

\subsection{Theoretical Results}

The following theorem states that $\ourmethod$ can provably converge to the target distribution in KL-divergence with quasi-polynomial gradient complexity.
\begin{theorem}[Gradient complexity of $\ourmethod$, informal]
    \label{thm:rrds}
    Under Assumptions~\ref{con_ass:lips_score}-\ref{con_ass:second_moment}, let $\bkwp_{0,S}$ be the distribution of the samples generated by $\ourmethod$, then there exists a collection of appropriate hyperparameters $n_{k,r}, m_{k,r}, \tau_r, \eta, l_{k,r}$ and $l$ such that with probability at least $1-\epsilon$, it holds that
   $\KL{p_*}{\bkwp_{0,S}}=\tilde{O}(\epsilon)$. Besides, the gradient complexity of $\ourmethod$ is
    \begin{equation}
        \label{eq:converge_rrds_sm}
        \begin{aligned}
            \exp\big[\mathcal{O}\left( L^3\cdot \log^3 \big((Ld+M)/\epsilon\right)\cdot \max\left\{\log\log Z^2,1\right\} \big)\big],
        \end{aligned}
    \end{equation}
    where $Z$ denotes the maximum norm of particles which appears in Alg~\ref{alg:rrds}.
\end{theorem}
We defer the detailed configurations of $n_{k,r}, m_{k,r}, \tau_r, \eta, l_{k,r},$ $l$ and relative constants in the formal version of this theorem, i.e., Theorem~\ref{thm:main_rrds_formal} Appendix~\ref{app_sec:main_proof} and Table~\ref{tab:constant_list} in Appendix~\ref{sec:not_ass_app}, respectively.
Then, we show a comparison between our method and previous work.

\paragraph{Comparison with ULA.} 
The gradient complexity of ULA has been well studied for sampling the non-log-concave distribution. However, in order to prove the convergence in KL divergence or TV distance, they typically require additional isoperimetric conditions, such as Log-Soboleve and Poincar\'e inequality (see Definitions~\ref{def:lsi} and~\ref{def:pi}). In particular, when $p_*$ satisfies LSI with parameter $\alpha$, \cite{vempala2019rapid} proved the $\mathcal{O}\left( d\epsilon^{-1}\alpha^{-2}\right)$ in KL convergence. However, for general non-log-concave distributions, $\alpha$ is not dimension-free. For instance, under the Dissipative condition  \citep{hale2010asymptotic}, $\alpha$ can be as worse as $\exp(- \mathcal O(d))$ \citep{raginsky2017non}, leading to a $\exp(\mathcal O(d))$ gradient complexity results \citep{xu2018global}.

When the isoperimetric condition is absent, \cite{balasubramanian2022towards} proved the convergence of ULA  based on the Fisher information measure, i.e., $\FI{p}{p_*}\coloneqq \E_{p}[\left\|\grad\log (p/p_*)\right\|^2]$, they showed that ULA can generate the samples that satisfy $\FI{p}{p_*}\le \epsilon$ for some small error tolerance $\epsilon$.  However,  it may be unclear what can be entailed by such a guarantee $\FI{p}{p_*}\le \epsilon$.
It has demonstrated that, in some cases, even if the Fisher information $\FI{p}{p_*}$ is very small, the total variation distance/KL divergence remains bounded away from zero. This suggests that the convergence guarantee in Fisher information might be weaker than that in KL divergence (i.e., our convergence guarantee). 

\paragraph{Comparison with RDS.} Then we make a detailed comparison with RDS in \citep{huang2023monte}, which is the most similar algorithm compared to ours. Firstly, we would like to strengthen again that our convergence results are obtained on a milder assumption, while \citet{huang2023monte} additionally requires the target distribution to have a heavier tail. Besides, as discussed in the introduction section, RDS has a much worse gradient complexity since it performs all score estimation straightforwardly, while $\ourmethod$ is based on a recursive structure. Consequently, RDS involves many hard sampling subproblems that take exponential time to solve, while $\ourmethod$ only involves strongly log-concave subsampling problems that can be efficiently solved within polynomial time. As a result, the gradient complexity of RDS is proved to be $\mathrm{poly}(d)\cdot \mathrm{poly}(1/\epsilon)\cdot \exp\left(\mathcal{O}(1/\epsilon)\right)$, which is significantly worse than the quasi-polynomial gradient complexity of $\ourmethod$.

\subsection{Proof Sketch}

In this section, we aim to highlight the technical innovations by presenting the roadmap of our analysis. Due to space constraints, we have included the technical details in the Appendix.

Firstly, by requiring Novikov's conditions, we can establish an upper bound on the KL divergence gap between the target distribution $p_*$ and the underlying distribution of output particles, i.e., $\bkwp_{0,S}$, by Girsanov's Theorem which demonstrates
\begin{equation*}
    \small
    \begin{aligned}
        \KL{p_*}{\bkwp_{0,S}}\le  &\underbrace{\KL{p_{K-1,S}}{\bkwp_{K-1,0}}}_{\text{Term 1}}+ \underbrace{2\sum_{k=0}^{K-1}\sum_{r=0}^{R-1} \int_{0}^\eta \E_{\rbkwx_{k,r\eta}}\left[\left\|\grad\log p_{k,S-r\eta}(\rbkwx_{k,r\eta})- \rbkwv_{k,r\eta}(\rbkwx_{k,r\eta})\right\|^2\right] \der t}_{\text{Term 3}}\\
        & +\underbrace{2\sum_{k=0}^{K-1}\sum_{r=0}^{R-1} \int_{0}^\eta \E_{(\rbkwx_{k,t+r\eta}, \rbkwx_{k,r\eta})}  \left[\left\|\grad\log p_{k, S-(t+r\eta)}(\rbkwx_{k,t+r\eta}) - \grad\log p_{k,S-r\eta}(\rbkwx_{k,r\eta}) \right\|^2\right]\der t}_{\text{Term 2}}.
    \end{aligned}
\end{equation*}
Although Novikov's condition may not be met in general, we employ techniques in~\cite{chen2023improved} and sidestep this issue by utilizing a differential inequality argument as shown in Lemma~\ref{lem:prop8_chen2023improved}.

\paragraph{Upper bound Term 1.}
Intuitively, $\mathrm{Term\ 1}$ appears since we utilize the standard Gaussian to initialize the reverse OU process (SDE~\eqref{con_eq:rrds_actual_backward}) rather than $p_{K-1,S}$ which can hardly be sampled from directly in practice.
Therefore, the first term can be bounded using exponential mixing of the forward (Ornstein-Uhlenbeck) process towards the standard Gaussian in Lemma~\ref{lem:var_thm4_vempala2019rapid}, i.e., 
\begin{equation*}
    \KL{p_{K-1,S}}{\bkwp_{K-1,0}}\le \KL{p_*}{\bkwp_{K-1,0}}\exp(-KS) \le  (Ld+M)\exp(-KS),
\end{equation*}
where $\bkwp_{K-1,0}=\mathcal{N}(\vzero, \mI)$ as shown SDE~\eqref{con_eq:rrds_actual_backward}.

\paragraph{Upper bound Term 2.}
Term 2 corresponds to the discretization error, which has been successfully addressed in previous work~\cite{chen2022sampling,chen2023improved}. By utilizing the unique structure of the Ornstein-Uhlenbeck process, they managed to limit both the time and space discretization errors, which decrease as $\eta$ becomes smaller. To ensure the completeness of our proof, we have included it in Lemma~\ref{lem:dis_err}, utilizing the segmented notation.

\paragraph{Upper bound Term 3.}
Term 3 represents the accuracy of the score estimation.
In diffusion models, due to the parameterization of the target density, this term is trained by a neural network and assumed to be less than $\epsilon$ to ensure the convergence of the reverse process.
However, in \ourmethod, the score estimation is obtained using a non-parametric approach, i.e., Alg~\ref{con_alg:rse}.
To this end, we can provide rigorous high probability bound for this term under Alg~\ref{con_alg:rse}, which is stated in Lemma~\ref{lem:rse_err}.

Roughly speaking, for Alg~\ref{alg:rrds} with input each $(k,r,\vx,\epsilon)$, suppose the score estimation of $\grad\log p_{k,0}$ is given as $\rbkwv_{k-1,0}$ satisfying the following event
\begin{equation*}
    \bigcap_{
            \vx^\prime\in \sS_{k,r}(\vx,\epsilon)
        } \left\|\grad\log p_{k,0}(\vx^\prime) -\rbkwv_{k-1,0}(\vx^\prime)\right\|^2\le l_{k,r}(\epsilon)
\end{equation*}
where $\sS_{k,r}(\vx,\epsilon)$ denotes the set of particles appear in Alg~\ref{con_alg:rse} except for the recursion.
In this condition, Lemma~\ref{lem:recursive_core_lem} provides the upper bound of score estimation error as:
\begin{equation*}
    \small
    \begin{aligned}
        \left\|\rbkwv_{k,r\eta}(\vx) - \grad\log p_{k,S-r\eta}(\vx)\right\|^2\le  \frac{2e^{-2(S-r\eta)}}{\left(1-e^{-2(S-r\eta)}\right)^2}\cdot \underbrace{\bigg\|-\frac{1}{n_{r,k}(\epsilon)}\sum_{i=1}^{n_{r,k}(\epsilon)} \rvx_i^\prime + \E_{\rvx^\prime \sim q^\prime_{k,S-r\eta}(\cdot|\vx)}\left[\rvx^\prime\right]\bigg\|^2}_{\text{Term 3.1}}\\
        \quad + \frac{2e^{-2(S-r\eta)}}{\left(1-e^{-2(S-r\eta)}\right)^2}\cdot \underbrace{\bigg\|-\E_{\rvx^\prime \sim q^\prime_{k,S-r\eta}(\cdot|\vx)}\left[\rvx^\prime\right] + \E_{\rvx^\prime \sim q_{k,S-r\eta}(\cdot|\vx)}\left[\rvx^\prime\right]\bigg\|^2}_{\text{Term 3.2}}
    \end{aligned}
\end{equation*}
where $q^\prime_{k,S-r\eta}(\cdot|\vx)$ is the underlying distribution of output particles, i.e., $\rvx^\prime_{m_{k,r}(l_{k,r}(\epsilon))}$ in Alg~\ref{con_alg:rse}. 
Considering that the distribution $q_{k,S-r\eta}$ is strongly log-concave (given in Eq.~\ref{con_ineq:cho_s_small}) and we can get a lower bound on the strongly log-concave constant (see Lemma~\ref{lem:sm_con_of_q}). Therefore, $q^\prime_{k,S-r\eta}$ also satisfies the log-Sobolev inequality due to Lemma~\ref{lem:thm8_vempala2019rapid}, which can imply the variance upper bound (see Lemma~\ref{lem:lsi_var_bound}). Then, in our proof, we directly make use of the Sobolev inequality to derive the high-probability bound (or concentration results) for estimating the mean of $q^\prime_{k,S-r\eta}(\cdot|\vx)$ in Term 3.1 with Lemma~\ref{lem:recursive_core_lem} by selecting sufficiently large $n_{k,r}(\epsilon)$.
Besides, Term 3.2 can be upper bounded by $\mathrm{KL}\big(q^\prime_{k,S-r\eta}(\cdot|\vx)\|q_{k,S-r\eta}(\cdot|\vx)\big)$, which can be well controlled by conducting the ULA with a sufficiently large iteration number $m_{k,r}(\epsilon)$. 
Therefore, by conducting the following decomposition
\begin{equation*}
    \small
    \begin{aligned}
        &\sP\big[\left\|\grad\log p_{k,S-r\eta}(\bkwx_{k,r\eta})-\rbkwv_{k,r\eta}(\bkwx_{k,r\eta})\right\|^2\le \epsilon\big] \notag\\
        &\ge (1-\delta)  \mathbb{P}\bigg[\bigcap_{
            \vx^\prime\in \sS_{k,r}(\vx,\epsilon)
        } \left\|\grad\log p_{k,0}(\vx^\prime) -\rbkwv_{k-1,0}(\vx^\prime)\right\|^2\le l_{k,r}(\epsilon)\bigg].
    \end{aligned}
\end{equation*}
We only need to use this proof process recursively with a proper choice of $\delta$ ($\delta$ as a function of $\epsilon$) to get the bound:
\begin{equation*}
    \small
    \sP\big[\left\|\grad\log p_{k,S-r\eta}(\bkwx_{k,r\eta})-\rbkwv_{k,r\eta}(\bkwx_{k,r\eta})\right\|^2\le \epsilon\big]\ge 1-\epsilon,
\end{equation*}
which implies $\text{Term 3} \le \tilde{O}(\epsilon)$ with a probability at least $1-\epsilon$.
Due to the large amount of computation, we defer the details of the recursive proof procedure and the choice of $\delta$ to the Appendix~\ref{app_sec:core_lemmas}.

%% file: 0_contents/050exp.tex
{
    \begin{figure}[!hbpt]
    \centering
    \begin{tabular}{ccccc}
          \includegraphics[width=3cm]{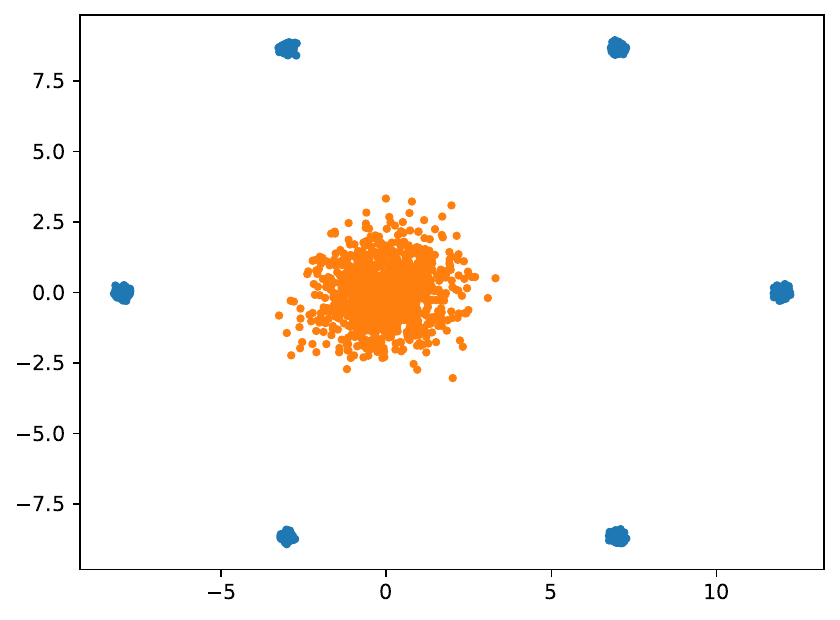}&
          \includegraphics[width=3cm]{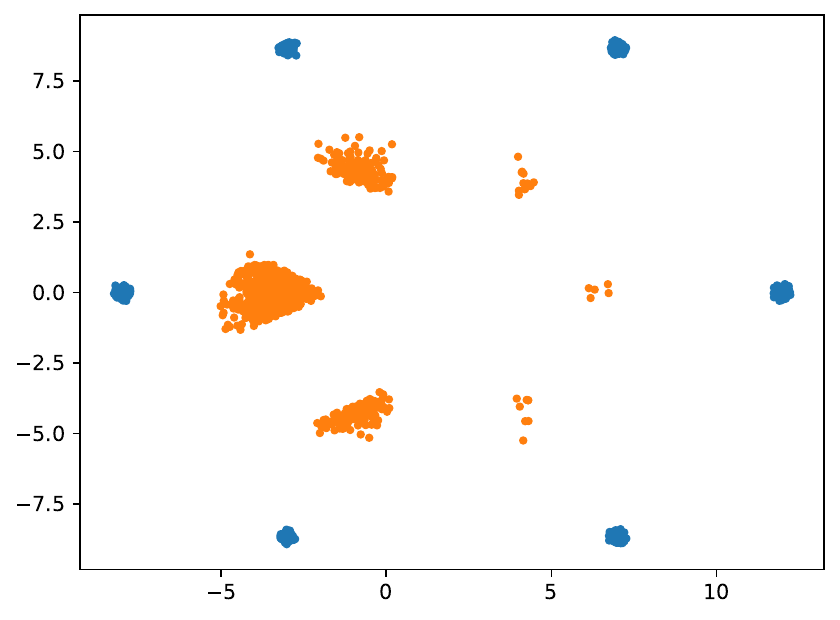}&
          \includegraphics[width=3cm]{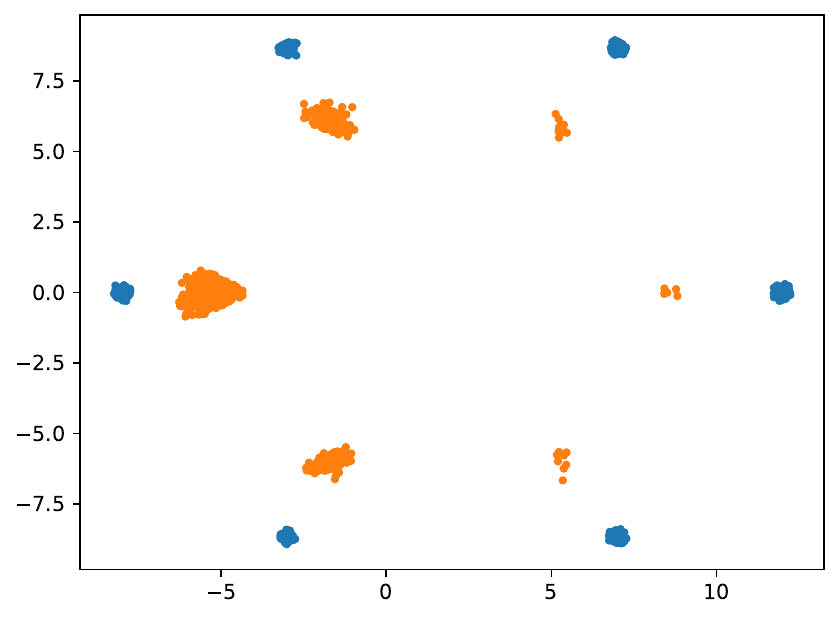}&
          \includegraphics[width=3cm]{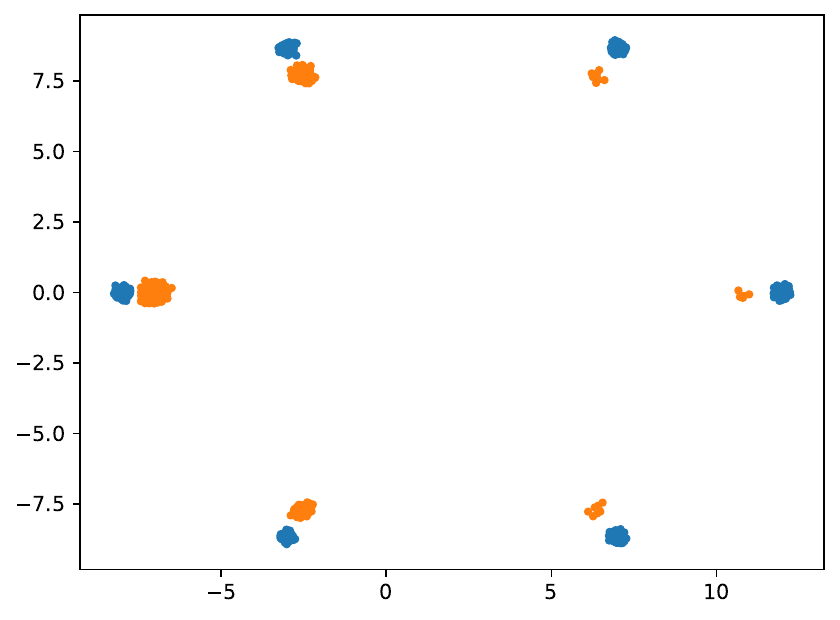}&
          \includegraphics[width=3cm]{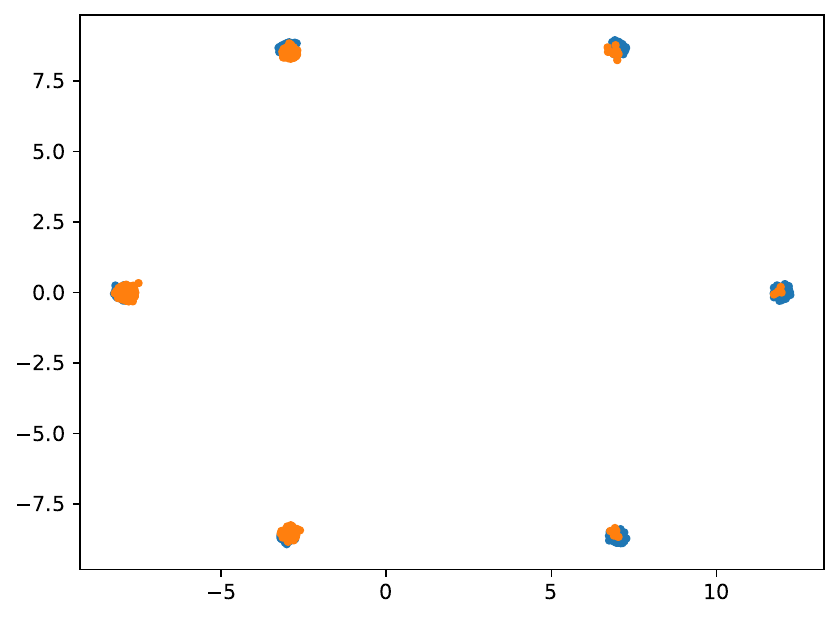}\\
          $\mathrm{grad\ num}=0$&$\mathrm{grad\ num}=25$&$\mathrm{grad\ num}=50$&$\mathrm{grad\ num}=100$&$\mathrm{grad\ num}=200$\\
          \includegraphics[width=3cm]{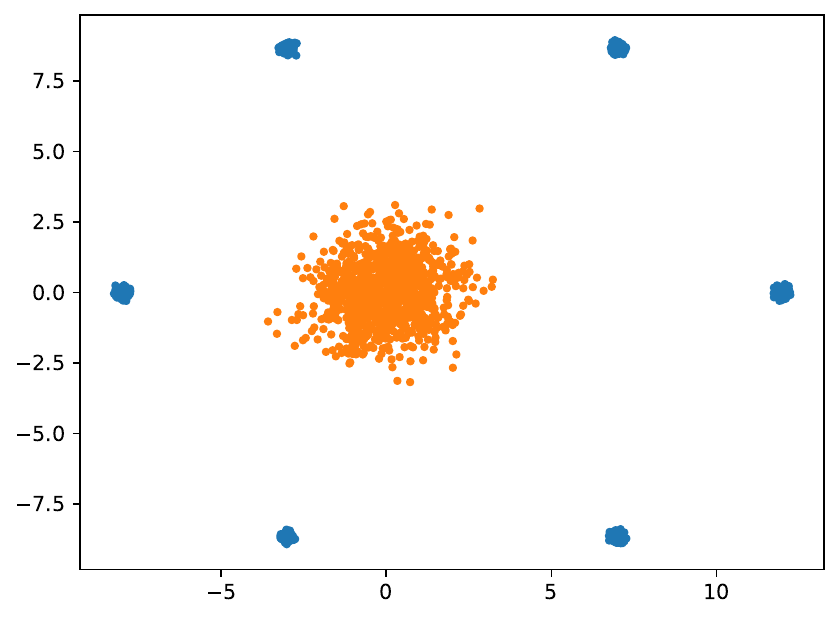}&
          \includegraphics[width=3cm]{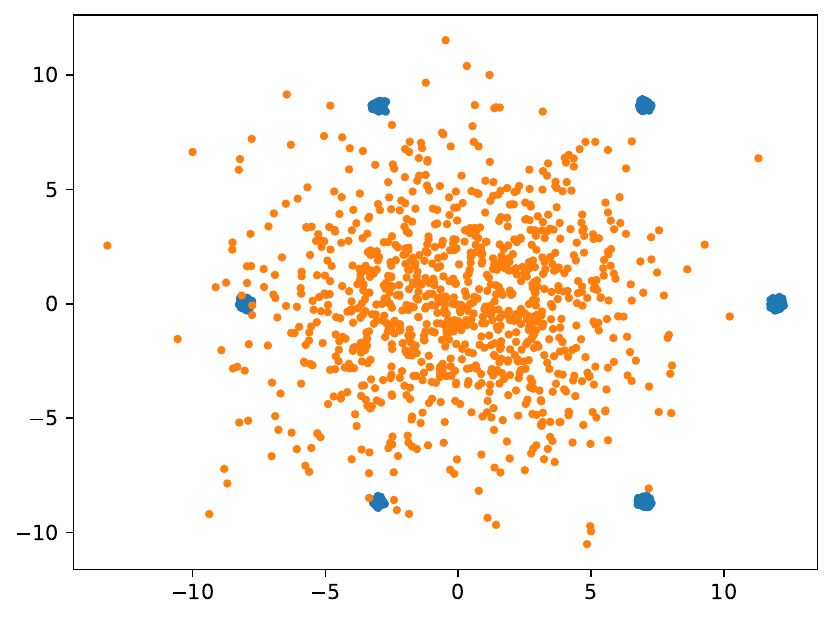}&
          \includegraphics[width=3cm]{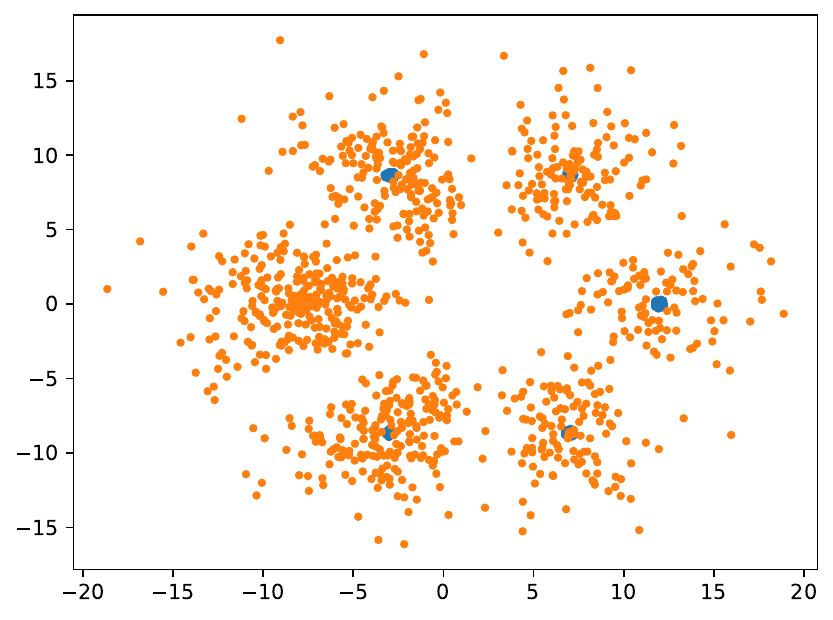}&
          \includegraphics[width=3cm]{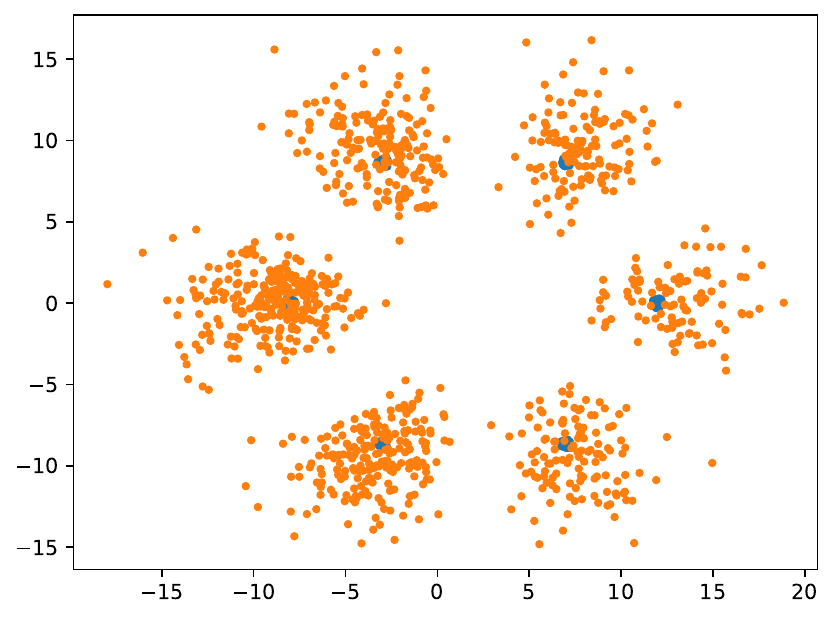}&
          \includegraphics[width=3cm]{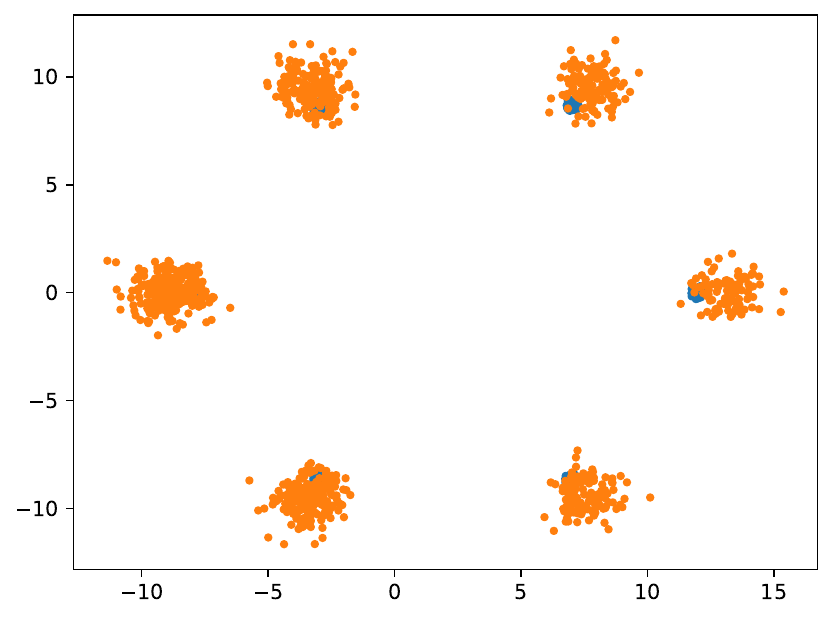}\\
          $\mathrm{grad\ num}=0$&$\mathrm{grad\ num}=100$&$\mathrm{grad\ num}=190$&$\mathrm{grad\ num}=195$&$\mathrm{grad\ num}=200$\\
          \includegraphics[width=3cm]{0_figs/RSDMC_0_dis.pdf}&
          \includegraphics[width=3cm]{0_figs/RSDMC_100_dis.pdf}&
          \includegraphics[width=3cm]{0_figs/RSDMC_190_dis.pdf}&
          \includegraphics[width=3cm]{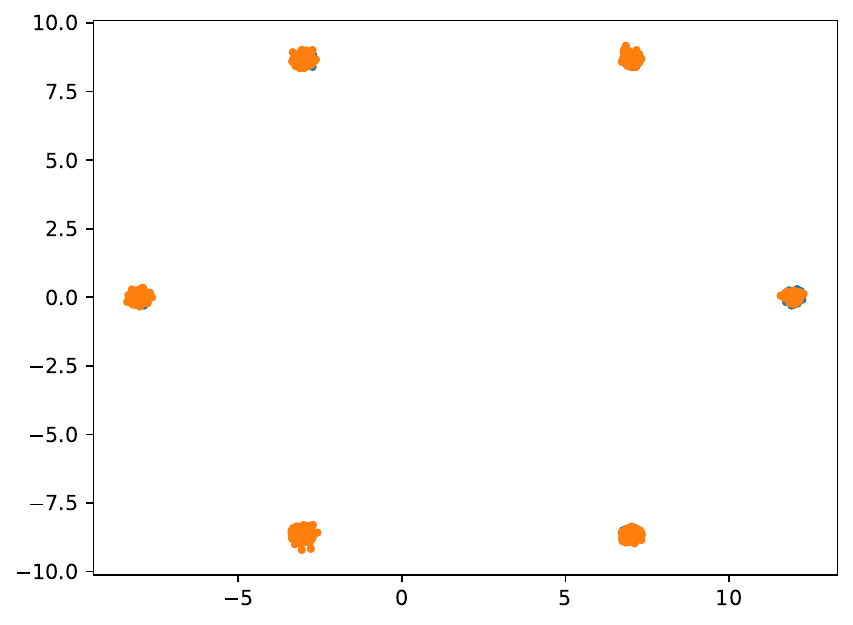}&
          \includegraphics[width=3cm]{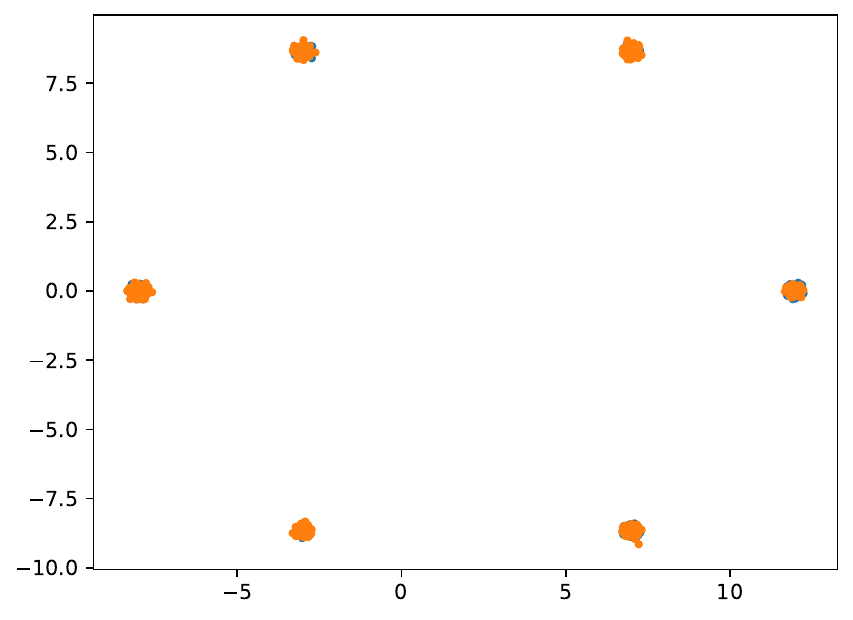}\\
          $\mathrm{grad\ num}=0$&$\mathrm{grad\ num}=100$&$\mathrm{grad\ num}=190$&$\mathrm{grad\ num}=195$&$\mathrm{grad\ num}=200$\\
    \end{tabular}
    \caption{\small Illustration of the returned particles for ULA, \ourmethod-v1 and \ourmethod-v2 shown with orange particles and the blue ones sampled from the ground truth. The first row is returned by ULA, the second is \ourmethod-v1 and the last is from \ourmethod-v2. Experimental results show that ULA converges fast in the local regions of modes, while it suffers from the problem of covering all modes. \ourmethod-v1 can cover most modes with few gradient oracles but converge slowly in local regions. \ourmethod-v2 takes advantage of both ULA and \ourmethod-v1, which can cover most modes and admit a faster local convergence. 
    }
    \label{fig:illustration_exp}
    \end{figure}
    \section{Empirical Results}
    
    \noindent\textbf{Experimental settings.} We consider the target distribution defined on $\R^2$ to be a mixture of Gaussian distributions with $6$ modes.
    Meanwhile, we draw $1,000$ particles from the target distribution, presented as blue nodes shown in Fig~\ref{fig:illustration_exp}.

    We fix the random seed and initialize particles with the standard Gaussian. Then, update particles with the following three settings:
    \begin{itemize}
        \item \textbf{ULA.} We choose  ULA~\cite{neal1992bayesian,roberts1996exponential} as the sampler, setting the step size and the iteration number as $2\cdot 10^{-4}$ and $200$ respectively.
        \item \textbf{RS-DMC-v1.} We choose RS-DMC as the sampler, setting the outer step size and the inner step size as $\eta = 5\cdot 10^{-2}$ and $\tau_r = 1\cdot 10^{-2}$ respectively. For inner loops, the number of samples and iterations, i.e., $n_{k,r}$ and $m_{k,r}$, are both chosen as $1$. For outer loops, the number of iterations is chosen as $200$, and we divide the entire process into two segments, i.e., $K=2$, and each segment contains $100$ iterations, i.e., $R=100$. 
        \item \textbf{RS-DMC-v2.} We choose the same hyper-parameter settings as that in RS-DMC-v1. Besides, we replace the last $10$ iterations as the ULA's update since when $p_{t}$ is closed to $p_*$, we have
        \begin{equation*}
            \small
            \begin{aligned}
                \lim_{t\rightarrow T} \grad \log p_0(\vx) = & \lim_{t\rightarrow T} \frac{\int \grad p_{0}(\vx_0)\cdot \left(2\pi (1-e^{-2(T-t)})\right)^{-d/2}\cdot \exp\left(-\frac{\left\|\vx-e^{-(T-t)}\vx_0\right\|^2}{2(1-e^{-2(T-t)})}\right)\der \vx_0}{\int p_{0}(\vx_0)\cdot \left(2\pi (1-e^{-2(T-t)})\right)^{-d/2}\cdot \exp\left(-\frac{\left\|\vx-e^{-(T-t)}\vx_0\right\|^2}{2(1-e^{-2(T-t)})}\right)\der \vx_0} =  -\grad f_*(\vx),
            \end{aligned}
        \end{equation*}
        which means $-\grad f_*(\vx)$ can be used to approximate the score, and we do not require the mean estimation with inner loops.
    \end{itemize}

    \paragraph{Experimental results.} To compare the behaviors of the three methods, we illustrate the particles when the algorithms return for different gradient complexity in Fig~\ref{fig:illustration_exp}.
    We note that
    \begin{itemize}
        \item ULA will quickly fall into some specific modes, and most steps are used to improve the mean estimation of each mode. However, the number of particles belonging to each mode is unbalanced and almost determined at the very beginning of the entire process. This is because the drift force of different modes at the origin varies greatly.
        \item \ourmethod-v1 quickly covers the different modes and converges to their means. Besides, the number of particles belonging to each mode is much more balanced than that in ULA. However, since we only choose $n_{k,r}=m_{k,r}=1$, and the score $\grad \log p_{k,t}(\vx)$ does not be approximated accurately, the convergence to specific modes will be relatively slow, which causes the variance of \ourmethod-v1 larger than the target distribution.
        \item \ourmethod-v2 takes the advantage of \ourmethod-v1 and estimate the the score $\grad \log p_{k,t}(\vx)$, when $p_{k,t}(\vx)$ approaches $p_*$, with $-f_*(\vx)$ directly rather than a inner-loop mean estimation. From another perspective, \ourmethod-v2 covers the different modes by \ourmethod-v1 and achieves local convergence by ULA. Hence, it has a balanced number of particles belonging to each mode and shares a variance almost the same as that in the ground truth.
    \end{itemize}
}

%% file: 0_contents/060conclusion.tex
\section{Conclusion}
In this paper, we propose a novel non-parametric score estimation algorithm, i.e., RSE, presented in Alg~\ref{con_alg:rse} and derive its corresponding reverse diffusion sampling algorithm, i.e., $\ourmethod$, and outlined in Alg~\ref{alg:rrds}.
By introducing the segment length $S$ to balance the challenges of score estimation and recursive calls, $\ourmethod$ exhibits several advantages over Langevin-based MCMC, e.g., ULA, ULD, and MALA.
It can achieve KL convergence beyond isoperimetric target distributions with a quasi-polynomial gradient complexity, i.e., 
    \begin{equation*}
        \begin{aligned}
            \exp\big[\mathcal O(L^3\cdot \log^3(d/\epsilon)\cdot \max\left\{\log\log Z^2,1\right\} )\big].
        \end{aligned}
    \end{equation*} 
Additionally, the theoretical result also demonstrates the efficiency of $\ourmethod$ in challenging sampling tasks.
To the best of our knowledge, this is the first work that eliminates the exponential dependence with only smoothness and the second moment bounded assumptions.

%% file: 0_contents/0Xappendix_OU/appendix_main.tex
\input{0_contents/0Xappendix_OU/0Anotations}

\input{0_contents/0Xappendix_OU/0Bmainproofs}

\input{0_contents/0Xappendix_OU/0CLemInit}

\input{0_contents/0Xappendix_OU/0DLemDis}

\input{0_contents/0Xappendix_OU/0ELemScr}

\input{0_contents/0Xappendix_OU/0FLemAux}

%% file: 0_contents/0Xappendix_OU/0Anotations.tex
\section{Notations}
\label{sec:not_ass_app}
\begin{table}[h]
    \centering
    \begin{tabular}{c l}
        \toprule
        Symbols & Description\\
        \midrule
        $\varphi_{\sigma^2}$ & The density function of the centered Gaussian distribution, i.e., $\mathcal{N}\left(\vzero,  \sigma^2\mI\right)$. \\
        \midrule
        $p_*, p_{0,0}$ & The target density function (initial distribution of the forward process) \\
        \midrule
        $\left\{\rvx_{k,t}\right\}_{k\in\mathbb{N}_{0,K-1}, t\in[0,S]}$ &  The forward process, i.e., SDE~\ref{con_eq:rrds_forward}\\
        $p_{k,t}$ & The density function of $\rvx_{k,t}$, i.e., $\rvx_{k,t}\sim p_{k,t}$\\
        $p_\infty$ & The density function of the stationary distribution of the forward process\\
        \midrule
        $\{\rbkwx_{k,t}\}_{k\in\mathbb{N}_{0,K-1},t\in[0,S]}$ &  The practical reverse process following from SDE~\ref{con_eq:rrds_actual_backward} with initial distribution $p_\infty$\\
        $\bkwp_{k,t}$ & The density function of $\rbkwx_{k,t}$, i.e., $\rbkwx_{k,t}\sim \bkwp_{k,t}$\\
        \bottomrule
    \end{tabular}
    \caption{The list of notations defined in Section~\ref{sec:pre}, where $\mathbb{N}_{a,b}$ is denoted as the set of natural numbers from $a\in \mathbb{N}_*$ to any $b \in \mathbb{N}_+$.}
    \label{tab:notation_list_app}
\end{table}

\begin{table}
    \centering  
    \footnotesize
    \begin{tabular}{cc|cc}
    \toprule
         Constant symbol & Value &  Constant symbol & Value\\
         \midrule
         $C_\eta$ & $2^{-14} L^{-2}$ &  $C_{m,1}$ & $\log \left(2M\cdot 3^2 \cdot 5L\right)+M\cdot 3L$ \\
         $C_n$ & $2^6\cdot 5^2\cdot C_\eta^{-1}$ &  $C_{m}$ & $2^9\cdot 3^2\cdot 5^3\cdot C_{m,1}C_{\eta}^{-1.5}$ \\
         $C_{u,1}$ & $\log\left(\frac{5C_nC_m}{10^4}\right)+\log \left(2 \max\left\{\log Z, \frac{1}{2}\right\}\right)$ & $C_{u,2}$ & $70/S^2+10/S$\\
         $C_{u,3}$ & $2C_{u,1}/S$ & $S$ & $1/2\log((2L+1)/2L)$\\
         \bottomrule
    \end{tabular}    
    \caption{Constant List independent with $\epsilon$ and $d$.}
    \label{tab:constant_list}
\end{table}

In this section, we summarize the notations defined in Section~\ref{sec:pre} in Table~\ref{tab:notation_list_app} for easy reference and cross-checking.
Additionally, another important notation is the score estimation, denoted as $\rbkwv_{k,r\eta}$, which is used to approximate $\grad\log p_{k,S-r\eta}$.
When $r=0$, $\rbkwv_{k,0}$ is expected to approximate $\grad\log p_{k,S}$ which is not explicitly defined in SDE~\ref{con_eq:rrds_forward}.
However, sine $\rvx_{k,S}=\rvx_{k+1,0}$ in Eq~\ref{con_eq:rrds_forward}, the underlying distributions, i.e., $p_{k,S}$ and $p_{k+1,0}$, are equal, and $\tilde{\rvv}_{k,0}$ can be considered as the score estimation of $\grad\log p_{k+1, 0}$.
For $\grad\log p_{0,0}$, which can be calculated exactly as $\grad f_*$, we define
\begin{equation}
    \label{def:v-10_def}
    \rbkwv_{-1,0}(\vx) = \grad \log p_{0,0}(\vx) = -\grad f_*(\vx)
\end{equation}
as a complement.

\paragraph{Isopermetric conditions and assumptions.}
According to the classical theory of Markov chains and diffusion processes, some conditions can lead to fast convergence over time without being as strict as log concavity. Isoperimetric inequalities, such as the log-Sobolev inequality (LSI) or the Poincar\'e inequality (PI), are examples of these conditions defined as follows.
\begin{definition}[Logarithmic Sobolev inequality]
\label{def:lsi}
 A distribution with density function $p$ satisfies the log-Sobolev inequality with a constant $\mu>0$ if for all smooth function $g\colon \R^d \rightarrow \R$ with $\E_p[g^2]\le \infty$,
 \begin{equation*}
    \mathbb{E}_{p}\left[g^2\log g^2\right]-\mathbb{E}_{p}\left[g^2\right]\log \mathbb{E}_{p}\left[g^2\right] \le  2\alpha^{-1}\mathbb{E}_{p}\left[\left\|\grad g\right\|^2\right].
\end{equation*}  
\end{definition}
By supposing $g=1+\epsilon \hat{g}$ with $\epsilon\rightarrow 0$, a weaker isoperimetric inequality, i.e., PI can be defined~\cite{menz2014poincare}.
\begin{definition}[Poincar\'e inequality]
\label{def:pi}
A distribution with density function $p$ satisfies the Poincar\'e inequality with a constant $\mu>0$ if for all smooth function $\hat{g}\colon \R^d \rightarrow \R$,
 \begin{equation*}
    \Var(\hat{g})\le  \alpha^{-1}\mathbb{E}_{p}\left[\left\|\grad \hat{g}\right\|^2\right].
\end{equation*}  
\end{definition}

We also provide a list of constants used in our following proof in Table~\ref{tab:constant_list} to prevent confusion.

%% file: 0_contents/0Xappendix_OU/0Bmainproofs.tex
\section{Proof of Theorem~\ref{thm:rrds}}
\label{app_sec:main_proof}
\begin{theorem}{The formal version of Theorem~\ref{thm:rrds}}
    \label{thm:main_rrds_formal}
    In Alg~\ref{alg:rrds}, suppose we set
    \begin{equation*}
        \begin{aligned}
            &{S=1/2\cdot \log (1+1/2L),\quad  K= 2\log [(Ld+M)/\epsilon] \cdot S^{-1},}\\
            &{\eta = C_\eta(M+d)^{-1}\epsilon,\quad R=S/\eta,}\\
            &l(\epsilon)=10\epsilon,\quad l_{k,r}(\epsilon)=\epsilon/960,\\
            &n_{k,r}(\epsilon) = C_n\cdot (d+M)\epsilon^{-2}\cdot \max\{d,-2\log \delta\},\\
            &m_{k,r}(\epsilon,\vx) = C_m\cdot (d+M)^3\epsilon^{-3}\cdot \max\{\log \|\vx\|^2,1\},\\
            &{\tau_r = 2^{-5}\cdot 3^{-2}\cdot e^{2(S-r\eta)}\left(1-e^{-2(S-r\eta)}\right)^2\cdot  d^{-1}\epsilon}  
        \end{aligned}
    \end{equation*}
    where $\delta$ satisfies
    \begin{equation*}
        \begin{aligned}
            \delta = &\mathrm{pow}\left(2, -\frac{2}{S}\log \frac{Ld+M}{\epsilon}\right)\cdot \mathrm{pow}\left(\frac{C_\eta S \epsilon^2}{4(d+M)}\cdot \log^{-2}\left(\frac{Ld+M}{\epsilon}\right)\cdot \mathrm{pow}\left(\left(\frac{Ld+M}{\epsilon}\right), \right.\right.\\
            &\quad \left.\left. -C_{u,2}\log \frac{Ld+M}{\epsilon}-C_{u,3}\right), \frac{2}{S}\log \frac{Ld+M}{\epsilon}+1\right),
        \end{aligned}
    \end{equation*}  
    and the initial underlying distribution $q^\prime_0$ of the Alg~\ref{con_alg:rse} with input $(k,r,\vx,\epsilon)$ satisfies
    \begin{equation*}
        q^\prime_0(\vx^\prime)\propto \exp\left(-\frac{\left\|\vx-e^{-(S-r\eta)}\vx^\prime\right\|^2}{2(1-e^{-2(S-r\eta)})}\right),
    \end{equation*}
    we have
    \begin{equation*}
        \mathbb{P}\left[\KL{\hat{p}_{0,S}}{\bkwp_{0,S}} = \tilde{O}(\epsilon)\right]\ge 1-\epsilon.
    \end{equation*}
    In this condition, the gradient complexity will be 
    \begin{equation*}
        \exp\left[\mathcal{O}\left(L^3\cdot  \log^3 \left((Ld+M)/\epsilon\right)\cdot \max\left\{\log\log Z^2,1\right\} \right)\right]
    \end{equation*}
    where $Z$ is the maximal norm of particles appeared in Alg~\ref{alg:rrds}.
\end{theorem}
\begin{proof}[Proof of Theorem~\ref{thm:main_rrds_formal}]
    According to Lemma~\ref{lem:prop8_chen2023improved}, suppose $\hat{\rvx}_{k,t}=\rvx_{k,S-t}$ whose SDE can be presented as 
    \begin{equation*}
    \begin{aligned}
        &\rbkwx_{k,0} \sim p_{K-1,S}\ \mathrm{when}\ k=K-1,\ \mathrm{else}\ \rbkwx_{k,0}=\rbkwx_{k+1,S} && k\in \mathbb{N}_{0,K-1}\\
        &\der \rbkwx_{k,t} = \left[ \rbkwx_{k,t} + 2\grad\log p_{k, S-t}(\rbkwx_{k,t})\right]\der t + \sqrt{2}\der B_t\quad && k\in \mathbb{N}_{0,K-1}, t\in[0,S]\\
    \end{aligned}
    \end{equation*}
    due to~\cite{chen2022sampling}. 
    Then, we have $\KL{p_*}{\bkwp_{0,S}} = \KL{\hat{p}_{0,S}}{\bkwp_{0,S}}$ which satisfies
    \begin{equation}
        \small
        \label{ineq:kl_chain}
        \begin{aligned}
            \KL{\hat{p}_{0,S}}{\bkwp_{0,S}}\le  &\underbrace{\KL{\hat{p}_{K-1,0}}{\bkwp_{K-1,0}}}_{\text{Term 1}}\\
            &+\sum_{k=0}^{K-1}\sum_{r=0}^{R-1} \int_{0}^\eta \E_{(\hat{\rvx}_{k,t+r\eta}, \hat{\rvx}_{k,r\eta})} \left[\left\|\grad\log p_{k, S-(t+r\eta)}(\hat{\rvx}_{k,t+r\eta}) - \rbkwv_{k,r\eta}(\hat{\rvx}_{k,r\eta})\right\|^2\right]\der t.
        \end{aligned}
    \end{equation}
    \paragraph{Upper bound Term 1.} Term 1 can be upper-bounded as
    \begin{equation*}
        \text{Term 1}=\KL{p_{K-1,S}}{\bkwp_{K-1,0}}\le (Ld+M)\cdot \exp\left(-KS/2\right)
    \end{equation*}
    with Lemma~\ref{lem:var_thm4_vempala2019rapid} when $\bkwp_{K-1,0}$ is chosen as the standard Gaussian.
    Therefore, we choose
    \begin{equation*}
        \begin{aligned}
            \quad S=\frac{1}{2}\log \frac{2L+1}{2L},\quad  K= 2\log \frac{Ld+M}{\epsilon}\cdot \left(\frac{1}{2}\log \frac{2L+1}{2L}\right)^{-1},\quad \text{and}\quad KS \ge 2\log \frac{Ld+M}{\epsilon},
        \end{aligned}
    \end{equation*}
    which make the inequality $\text{Term 1}\le \epsilon$ establish.

    For the remaining term of RHS of Eq~\ref{ineq:kl_chain}, it can be decomposed as follows:
    \begin{equation}
        \label{ineq:error_decomp_term23}
        \begin{aligned}
            &\sum_{k=0}^{K-1}\sum_{r=0}^{R-1} \int_{0}^\eta \E_{(\hat{\rvx}_{k,t+r\eta}, \hat{\rvx}_{k,r\eta})} \left[\left\|\grad\log p_{k, S-(t+r\eta)}(\hat{\rvx}_{k,t+r\eta}) - \rbkwv_{k,r\eta}(\hat{\rvx}_{k,r\eta})\right\|^2\right]\der t\\
            &\le \underbrace{2\sum_{k=0}^{K-1}\sum_{r=0}^{R-1} \int_0^\eta \E \left[\left\|\grad\log p_{k, S-(t+r\eta)}(\hat{\rvx}_{k,t+r\eta}) - \grad\log p_{k,S-r\eta}(\hat{\rvx}_{k,r\eta}) \right\|^2\right] \der t}_{\text{Term 2}} \\
            & + \underbrace{2\sum_{k=0}^{K-1}\sum_{r=0}^{R-1} \int_0^\eta \E_{(\hat{\rvx}_{k,t+r\eta}, \hat{\rvx}_{k,r\eta})} \left[\left\|\grad\log p_{k,S-r\eta}(\hat{\rvx}_{k,r\eta})- \rbkwv_{k,r\eta}(\hat{\rvx}_{k,r\eta})\right\|^2\right] \der t}_{\text{Term 3}}
        \end{aligned}
    \end{equation}

    \paragraph{Upper bound Term 2.} This term is mainly from the discretization error in the reverse process.
    Therefore, its analysis is highly related to~\cite{chen2022sampling,chen2023improved}. 
    To ensure the completeness of our proof, we have included it in our analysis, utilizing the segmented notation presented in Section~\ref{sec:not_ass_app}.
    Specifically, we have
    \begin{equation*}
        \begin{aligned}
            \mathrm{Term\ 2} \le & 4\sum_{k=0}^{K-1}\sum_{r=0}^{R-1} \int_0^\eta \E\left[\left\|\grad\log p_{k,S-(t+r\eta)}(\hat{\rvx}_{k,t+r\eta})-\grad\log p_{k,S-(t+r\eta)}(\hat{\rvx}_{k,r\eta})\right\|^2\right]\\
            & + 4\sum_{k=0}^{K-1}\sum_{r=0}^{R-1} \int_0^\eta \E\left[\left\|\grad\log \frac{p_{k, S-(t+r\eta)}(\hat{\rvx}_{k,r\eta})}{p_{k,S-r\eta}(\hat{\rvx}_{k,r\eta})}\right\|^2\right] \der t\\
            \le & 4\sum_{k=0}^{K-1}\sum_{r=0}^{R-1} \int_0^\eta \left(\E\left[L^2\left\|\hat{\rvx}_{k,t+r\eta}-\hat{\rvx}_{k,r\eta}\right\|^2\right] + \E\left[\left\|\grad\log \frac{p_{k,S-r\eta}(\hat{\rvx}_{k,r\eta})}{p_{k, S-(t+r\eta)}(\hat{\rvx}_{k,r\eta})}\right\|^2\right]\right) \der t
        \end{aligned}
    \end{equation*}
    where the last inequality follows from Assumption~\ref{con_ass:lips_score}.
    Combining this result with Lemma~\ref{lem:dis_err}, when the stepsize, i.e., $\eta$ of the reverse process is $\eta = C_\eta(M+d)^{-1}\epsilon$, then it has $\text{Term 2}\le \epsilon$.

    \paragraph{Upper bound Term 3.}
    Due to the randomness of $\rbkwv_{k,r\eta}$, we consider a high probability bound, which is formulated as
    \begin{equation}
        \label{ineq:term3_fin}
        \sP\left[\bigcap_{
                \substack{k\in \mathbb{N}_{0,K-1}\\ r\in\mathbb{N}_{0,R-1}}
           }\left\|\grad\log p_{k,S-r\eta}(\rbkwx_{k,r\eta})-\rbkwv_{k,r\eta}(\rbkwx_{k,r\eta})\right\|^2\le 10\epsilon\right]\ge 1-\epsilon,
    \end{equation}
    which means we choose $l(\epsilon)=10\epsilon$.
    Lemma~\ref{lem:rse_err} demonstrate that under the following settings, i.e.,
    \begin{equation*}
        \begin{aligned}
            &l_{k,r}(\epsilon) = \epsilon / 960, \\
            &n_{k,r}(\epsilon) = C_n\cdot (d+M)\epsilon^{-2}\cdot \max\{d,-2\log \delta\},\\
            &m_{k,r}(\epsilon,\vx) = C_m\cdot (d+M)^3\epsilon^{-3}\cdot \max\{\log \|\vx\|^2,1\},
        \end{aligned}
    \end{equation*}
    where $\delta$ satisfies
    \begin{equation*}
        \begin{aligned}
            \delta \coloneqq &\mathrm{pow}\left(2, -\frac{2}{S}\log \frac{Ld+M}{\epsilon}\right)\cdot \mathrm{pow}\left(\frac{C_\eta S \epsilon^2}{4(d+M)}\cdot \log^{-2}\left(\frac{Ld+M}{\epsilon}\right)\cdot \mathrm{pow}\left(\left(\frac{Ld+M}{\epsilon}\right), \right.\right.\\
            &\quad \left.\left. -C_{u,2}\log \frac{Ld+M}{\epsilon}-C_{u,3}\right), \frac{2}{S}\log \frac{Ld+M}{\epsilon}+1\right),
        \end{aligned}
    \end{equation*}  
    Eq~\ref{ineq:term3_fin} can be achieved with a gradient complexity:
    \begin{equation*}
        \exp\left[\mathcal{O}\left(L^3\cdot  \log^3 \left((Ld+M)/\epsilon\right)\cdot \max\left\{\log\log Z^2,1\right\} \right)\right]
    \end{equation*}
    where $Z$ is the maximal norm of particles appeared in Alg~\ref{alg:rrds}.
    All constants can be found in Table~\ref{tab:constant_list}.
    In this condition, we have
    \begin{equation*}
        \begin{aligned}
            \mathrm{Term\ 3}\le 4\cdot \frac{T}{\eta}\cdot \left(\eta \cdot 10\epsilon\right) \le 40\epsilon\log \frac{Ld+M}{\epsilon} = \tilde{O}(\epsilon).
        \end{aligned}
    \end{equation*}
    Combining the upper bound of $\mathrm{Term\ 1}$, $\mathrm{Term\ 2}$ and $\mathrm{Term\ 3}$, we have
    \begin{equation*}
        \KL{\hat{p}_{0,S}}{\bkwp_{0,S}} = \tilde{O}(\epsilon).
    \end{equation*}
    The proof is completed.
\end{proof}

{
    \begin{corollary}
    Suppose we set all parameters except for $\delta$ to be the same as that in Theorem~\ref{thm:main_rrds_formal}, and define
    \begin{equation*}
        \begin{aligned}
            \delta = &\mathrm{pow}\left(2, -\frac{2}{S}\log \frac{Ld+M}{\epsilon}\right)\cdot \mathrm{pow}\left(\frac{C_\eta S \epsilon \delta^\prime}{4(d+M)}\cdot \log^{-2}\left(\frac{Ld+M}{\epsilon}\right)\cdot \mathrm{pow}\left(\left(\frac{Ld+M}{\epsilon}\right), \right.\right.\\
            &\quad \left.\left. -C_{u,2}\log \frac{Ld+M}{\epsilon}-C_{u,3}\right), \frac{2}{S}\log \frac{Ld+M}{\epsilon}+1\right),
        \end{aligned}
    \end{equation*}
    we have
    \begin{equation*}
        \mathbb{P}\left[\KL{\hat{p}_{0,S}}{\bkwp_{0,S}} = \tilde{O}(\epsilon)\right]\ge 1-\delta^\prime.
    \end{equation*}
    In this condition, the gradient complexity will be 
    \begin{equation*}
        \exp\left[\mathcal{O}\left(L^3\cdot  \max\left\{ \left(\log \frac{Ld+M}{\epsilon}\right)^3, \log \frac{Ld+M}{\epsilon}\cdot \log \frac{1}{\delta^\prime}\right\}\cdot \max\left\{\log\log Z^2,1\right\} \right)\right]
    \end{equation*}
    where $Z$ is the maximal norm of particles appeared in Alg~\ref{alg:rrds}.
    \end{corollary}
    \begin{proof}
        In this corollary, we follow the same proof roadmap as that shown in Theorem~\ref{thm:main_rrds_formal}.
        Combining Eq~\ref{ineq:kl_chain} and Eq~\ref{ineq:error_decomp_term23}, we have
        \begin{equation}
        \small
        \label{ineq:err_decompose}
        \begin{aligned}
            \KL{\hat{p}_{0,S}}{\bkwp_{0,S}}\le  &\underbrace{\KL{\hat{p}_{K-1,0}}{\bkwp_{K-1,0}}}_{\text{Term 1}}\\
            &\le \underbrace{2\sum_{k=0}^{K-1}\sum_{r=0}^{R-1} \int_0^\eta \E \left[\left\|\grad\log p_{k, S-(t+r\eta)}(\hat{\rvx}_{k,t+r\eta}) - \grad\log p_{k,S-r\eta}(\hat{\rvx}_{k,r\eta}) \right\|^2\right] \der t}_{\text{Term 2}} \\
            & + \underbrace{2\sum_{k=0}^{K-1}\sum_{r=0}^{R-1} \int_0^\eta \E_{(\hat{\rvx}_{k,t+r\eta}, \hat{\rvx}_{k,r\eta})} \left[\left\|\grad\log p_{k,S-r\eta}(\hat{\rvx}_{k,r\eta})- \rbkwv_{k,r\eta}(\hat{\rvx}_{k,r\eta})\right\|^2\right] \der t}_{\text{Term 3}}
        \end{aligned}
        \end{equation}
        It should be noted that the techniques for upper-bounding $\mathrm{Term\ 1}$ and $\mathrm{Term\ 2}$ are the same as that in Theorem~\ref{thm:main_rrds_formal}.
        
        \paragraph{Upper bound Term 3.}
        Due to the randomness of $\rbkwv_{k,r\eta}$, we consider a high probability bound, which is formulated as
        \begin{equation}
        \label{ineq:term3_fin_2}
        \sP\left[\bigcap_{
                \substack{k\in \mathbb{N}_{0,K-1}\\ r\in\mathbb{N}_{0,R-1}}
           }\left\|\grad\log p_{k,S-r\eta}(\rbkwx_{k,r\eta})-\rbkwv_{k,r\eta}(\rbkwx_{k,r\eta})\right\|^2\le 10\epsilon\right]\ge 1-\delta^\prime,
        \end{equation}
        which means we choose $l(\epsilon)=10\epsilon$.
        Lemma~\ref{lem:rse_err_2} demonstrate that under the following settings, i.e.,
        \begin{equation*}
        \begin{aligned}
            &l_{k,r}(\epsilon) = \epsilon / 960, \\
            &n_{k,r}(\epsilon) = C_n\cdot (d+M)\epsilon^{-2}\cdot \max\{d,-2\log \delta\},\\
            &m_{k,r}(\epsilon,\vx) = C_m\cdot (d+M)^3\epsilon^{-3}\cdot \max\{\log \|\vx\|^2,1\},
        \end{aligned}
        \end{equation*}
        where $\delta$ satisfies
        \begin{equation*}
        \begin{aligned}
            \delta \coloneqq &\mathrm{pow}\left(2, -\frac{2}{S}\log \frac{Ld+M}{\epsilon}\right)\cdot \mathrm{pow}\left(\frac{C_\eta S \epsilon \delta^\prime}{4(d+M)}\cdot \log^{-2}\left(\frac{Ld+M}{\epsilon}\right)\cdot \mathrm{pow}\left(\left(\frac{Ld+M}{\epsilon}\right), \right.\right.\\
            &\quad \left.\left. -C_{u,2}\log \frac{Ld+M}{\epsilon}-C_{u,3}\right), \frac{2}{S}\log \frac{Ld+M}{\epsilon}+1\right),
        \end{aligned}
        \end{equation*}  
        Eq~\ref{ineq:term3_fin_2} can be achieved with a gradient complexity:
        \begin{equation*}
        \exp\left[\mathcal{O}\left( L^3\cdot  \max\left\{ \left(\log \frac{Ld+M}{\epsilon}\right)^3, \log \frac{Ld+M}{\epsilon}\cdot \log \frac{1}{\delta^\prime}\right\}\cdot \max\left\{\log\log Z^2,1\right\} \right)\right]
        \end{equation*}
        where $Z$ is the maximal norm of particles appeared in Alg~\ref{alg:rrds}.
        All constants can be found in Table~\ref{tab:constant_list}.
        In this condition, we have
        \begin{equation*}
        \begin{aligned}
            \mathrm{Term\ 3}\le 4\cdot \frac{T}{\eta}\cdot \left(\eta \cdot 10\epsilon\right) \le 40\epsilon\log \frac{Ld+M}{\epsilon} = \tilde{O}(\epsilon).
        \end{aligned}
        \end{equation*}
        Combining the upper bound of $\mathrm{Term\ 1}$, $\mathrm{Term\ 2}$ and $\mathrm{Term\ 3}$, we have
        \begin{equation*}
        \KL{\hat{p}_{0,S}}{\bkwp_{0,S}} = \tilde{O}(\epsilon).
        \end{equation*}
        The proof is completed.
    \end{proof}
}

%% file: 0_contents/0Xappendix_OU/0CLemInit.tex
\section{Lemmas for Bounding Initialization Error}

\begin{lemma}[Lemma 11 in~\cite{vempala2019rapid}]
    \label{lem:lem11_vempala2019rapid}
    Suppose $p\propto \exp(-f)$ and $f\colon \R^d \rightarrow \R$ is $L$-gradient Lipschitz continuous function.
    Then, we have
    \begin{equation*}
        \E_{\rvx\sim p}\left[\left\|\grad f(\rvx)\right\|^2\right]\le Ld
    \end{equation*}
\end{lemma}

\begin{lemma}
    \label{lem:init_kl_bound}
    Under the notation in Section~\ref{sec:not_ass_app}, suppose $p\propto \exp(-f)$ satisfies Assumption~\ref{con_ass:lips_score} and~\ref{con_ass:second_moment}, then we have
    \begin{equation*}
        \KL{p}{\varphi_1}\le Ld+M
    \end{equation*}
\end{lemma}
\begin{proof}
    From the analytic form of the standard Gaussian, we have $\grad^2 \log \varphi_1 = \mI$.
    Combining this fact with Lemma~\ref{lem:strongly_lsi}, we have
    \begin{equation*}
        \begin{aligned}
            \KL{p}{\varphi_1} \le &\frac{1}{2}\int p(\vx)\left\|\grad\log \frac{p(\vx)}{\varphi_1(\vx)}\right\|^2 \der\vx\\
            \le &\int p(\vx)\left\|\grad f(\vx)\right\|^2\der\vx + \int p(\vx)\left\|\vx\right\|^2\der\vx \le Ld + M.
        \end{aligned}
    \end{equation*}
    where the last inequality follows from Lemma~\ref{lem:lem11_vempala2019rapid} and Assumption~\ref{con_ass:second_moment}. 
    Hence, the proof is completed.
\end{proof}

\begin{lemma}[Variant of Theorem 4 in~\cite{vempala2019rapid}]
    \label{lem:var_thm4_vempala2019rapid}
    Under the notation in Section~\ref{sec:not_ass_app}, 
    suppose $\tilde{p}_{K-1, 0}$ is chosen as the standard Gaussian distribution. Then, we have
    \begin{equation*}
        \KL{p_{K-1,S}}{p_\infty} \le (Ld+M)\cdot \exp\left(-KS/2\right).
    \end{equation*}
\end{lemma}
\begin{proof}
    Suppose another random variable $\rvz_t \coloneqq \rvx_{\floor{t/S}, t-\floor{t/S}\cdot S }$ where $\rvx_{k,t}$ is shown in SDE~\ref{con_eq:rrds_forward}, we have
    \begin{equation*}
        \begin{aligned}
            \der \rvz_t = -\rvz_t \der t+ \sqrt{2}\der B_t,\quad \rvz_0=\rvx_{0,0},
        \end{aligned}
    \end{equation*}
    where the underlying distribution of $\rvx_{0,0}$ satisfies $p_{0,0}=p_*\propto \exp(-f_*)$.
    If we denote $\rvz_t \sim p^{(z)}_t$, then Fokker-Planck equation of the previous SDE will be
    \begin{equation*}
        \partial_t p^{(z)}_t(\vz) = \nabla\cdot\left(p^{(z)}_t(\vz) \vz \right)+ \Delta p^{(z)}_t(\vz) = \nabla \cdot \left(p^{(z)}_t(\vz)\grad \log \frac{p^{(z)}_t(\vz)}{\exp\left(-\frac{1}{2}\|\vz\|^2\right)}\right).
    \end{equation*}
    It implies that the stationary distribution is standard Gaussian, i.e., $p^{(z)}_\infty \propto \exp(-1/2\cdot \left\|\vz\right\|^2).$
    Then, we consider the KL convergence of $(\rvz_t)_{t\ge 0}$, and have
    \begin{equation}
        \label{ineq:fwd_kl_decreasing}
        \begin{aligned}
        & \frac{\der \KL{p^{(z)}_t}{p^{(z)}_\infty}}{\der t} = \frac{\der}{\der t}\int p^{(z)}_t(\vz)\log\frac{p^{(z)}_t(\vz)}{p^{(z)}_\infty(\vz)}\der\vz = \int \partial_t p^{(z)}_t(\vz) \log \frac{p^{(z)}_t(\vz)}{p^{(z)}_\infty(\vz)}\der \vz\\
        & = \int \nabla\cdot \left(p^{(z)}_t(\vz)\grad \log \frac{p^{(z)}_t(\vz)}{p^{(z)}_\infty(\vz)}\right)\cdot \log \frac{p^{(z)}_t(\vz)}{p^{(z)}_\infty(\vz)}\der\vz= - \int p^{(z)}_t(\vz) \left\|\grad \log \frac{p^{(z)}_t(\vz)}{p^{(z)}_\infty(\vz)}\right\|^2\der\vz.
        \end{aligned}
    \end{equation}
    Combining the fact $\grad^2 (-\log p^{(z)}_\infty)=\mI$ and Lemma~\ref{lem:strongly_lsi}, we have 
    \begin{equation*}
        \KL{p^{(z)}_t}{p^{(z)}_\infty}\le 2 \int p^{(z)}_t(\vz)\left\|\grad \log \frac{p^{(z)}_t(\vz)}{p^{(z)}_\infty(\vz)}\right\|^2 \der\vz.
    \end{equation*}
    Plugging this inequality into Eq~\ref{ineq:fwd_kl_decreasing}, we have
    \begin{equation*}
        \frac{\der \KL{p^{(z)}_t}{p^{(z)}_\infty}}{\der t}= - \int p^{(z)}_t(\vz) \left\|\grad \log \frac{p^{(z)}_t(\vz)}{p^{(z)}_\infty(\vz)}\right\|^2\der\vz \le -\frac{1}{2}\KL{p^{(z)}_t}{p^{(z)}_\infty}.
    \end{equation*}
    Integrating implies the desired bound,i.e.,
    \begin{equation*}
        \begin{aligned}
            \KL{p^{(z)}_t}{p^{(z)}_\infty} \le &\exp\left(-t/2\right)\cdot \KL{p^{(z)}_0}{p^{(z)}_\infty} \le (Ld+M)\cdot \exp\left(-t/2\right)
        \end{aligned}
    \end{equation*}
    where the last inequality follows from Lemma~\ref{lem:init_kl_bound}.
    It implies KL divergence between the underlying distribution of $\rvx_{K-1, S}$ and $p_{\infty}$ is 
    \begin{equation*}
        \KL{p_{K-1,S}}{p_\infty} = \KL{p^{(z)}_{KS}}{p^{(z)}_\infty} \le (Ld+M)\cdot \exp\left(-KS/2\right)
    \end{equation*}
    Hence, the proof is completed.
\end{proof}

%% file: 0_contents/0Xappendix_OU/0DLemDis.tex
\section{Lemmas for Bounding Discretization Error.}

\begin{lemma}[Lemma C.11 in~\cite{lee2022convergence}]
    \label{lem:lemc11_lee2022convergence}
    Suppose that $p(\vx)\propto e^{-f(\vx)}$ is a probability density function on $\R^d$, where $f(\vx)$ is $L$-smooth, and let $\varphi_{\sigma^2}(\vx)$ be the density function of $\mathcal{N}(\vzero, \sigma^2\mI_d)$. 
    Then for $L\le \frac{1}{2\sigma^2}$, it has
    \begin{equation*}
        \left\|\grad \log \frac{p(\vx)}{\left(p\ast \varphi_{\sigma^2}\right)(\vx)}\right\|\le 6L\sigma d^{1/2} + 2L\sigma^2 \left\|\grad f(\vx)\right\|.
    \end{equation*}
\end{lemma}

\begin{lemma}[Lemma 9 in~\cite{chen2022sampling}]
    \label{lem:lem9_chen2022sampling}
    Under the notation in Section~\ref{sec:not_ass_app}, suppose that Assumption~\ref{con_ass:lips_score} and \ref{con_ass:second_moment} hold. 
    For any $k\in \mathbb{N}_{0,K-1}$ and $t\in [0,S]$, we have
    \begin{enumerate}
        \item Moment bound, i.e., 
        \begin{equation*}
            \mathbb{E}\left[\left\|\rvx_{k,t}\right\|^2\right]\le d\vee M.
        \end{equation*}
        \item Score function bound, i.e.,
        \begin{equation*}
            \mathbb{E}\left[\left\|\grad \log p_{k, t}(\rvx_{k,t})\right\|^2\right]\le Ld.
        \end{equation*}
    \end{enumerate}
\end{lemma}

\begin{lemma}[Variant of Lemma 10 in~\cite{chen2022sampling}]
    \label{lem:var_lem10_chen2022sampling}
    Under the notation in Section~\ref{sec:not_ass_app},Suppose that Assumption~\ref{con_ass:second_moment} holds. 
    For any $k\in \left\{0, 1,\ldots, K-1\right\}$ and $0\le s\le t \le S$, we have
    \begin{equation*}
        \mathbb{E}\left[\left\|\rvx_{k,t} - \rvx_{k,s}\right\|^2\right] \le 2 \left(M +d\right)\cdot \left(t-s\right)^2 + 4 d \cdot \left(t-s\right)
    \end{equation*}
\end{lemma}
\begin{proof}
    According to the forward process, we have
    \begin{equation*}
        \begin{aligned}
            \mathbb{E}\left[\left\|\rvx_{k,t} - \rvx_{k,s}\right\|^2\right] = & \mathbb{E}\left[\left\|\int_s^t - \rvx_{k,r} \der r + \sqrt{2}\left(B_t - B_s\right)\right\|^2\right]\le  \mathbb{E}\left[2\left\|\int_s^t \rvx_{k,r} \der r\right\|^2 + 4\left\|B_t - B_s\right\|^2\right]\\
            \le & 2\mathbb{E}\left[\left(\int_s^t \left\|\rvx_{k,r}\right\| \der r\right)^2\right] + 4 d\cdot (t-s) \le 2\int_s^t \mathbb{E}\left[\left\|\rvx_{k,r}\right\|^2\right]\der r \cdot (t-s)+ 4 d\cdot (t-s) \\
            \le & 2 \left(M +d\right)\cdot \left(t-s\right)^2 + 4 d \cdot \left(t-s\right),
        \end{aligned}
    \end{equation*}
    where the third inequality follows from Holder's inequality and the last one follows from Lemma~\ref{lem:lem9_chen2022sampling}.
    Hence, the proof is completed.
\end{proof}

\begin{lemma}[Errors from the discretization]
    \label{lem:dis_err}
    Under the notation in Section~\ref{sec:not_ass_app}, if the step size of the outer loops satisfies
    \begin{equation*}
        \eta\le C_1 (d+M)^{-1}\epsilon,
    \end{equation*}
    then, for any $k\in\{0,1,\ldots, K-1\}$, $r\in\{0, 1,\ldots, R-1\}$ and $t\in [0,\eta]$, we have
    \begin{equation*}
        \E\left[L^2\left\|\hat{\rvx}_{k,t+r\eta}-\hat{\rvx}_{k,r\eta}\right\|^2\right] + \E\left[\left\|\grad\log \frac{p_{k,S-r\eta}(\hat{\rvx}_{k,r\eta})}{p_{k, S-(t+r\eta)}(\hat{\rvx}_{k,r\eta})}\right\|^2\right]\le 4\epsilon.
    \end{equation*}
\end{lemma}
\begin{proof}
    We consider the following formulation with any $t\in[0,\eta]$, 
    \begin{equation}
        \label{eq: dis_err}
         \text{Term 2}=\underbrace{\E\left[\left\|\grad\log \frac{p_{k,S-r\eta}(\hat{\rvx}_{k,r\eta})}{p_{k, S-(t+r\eta)}(\hat{\rvx}_{k,r\eta})}\right\|^2\right]}_{\text{Term 2.1}}+\E\left[L^2\left\|\hat{\rvx}_{k,t+r\eta}-\hat{\rvx}_{k,r\eta}\right\|^2\right].
    \end{equation}
    \paragraph{Upper bound Term 2.1.}
    To establish the connection between $p_{k,S-r\eta}$ and $p_{k,S-(t+r\eta)}$, due to the transition kernel of the forward process (OU process), we have
    \begin{equation}
        \label{eq:fwd_trans_ker}
        \begin{aligned}
            p_{k, S-r\eta}(\vx) = &\int p_{k,S-(r\eta+t)}(\vy)\cdot \mathbb{P}\left[\vx, (k, S-r\eta)| \vy, (S-(r\eta+t))\right]\der \vy\\
            = & \int p_{k,S-(r\eta+t)}(\vy)\cdot \left(2\pi \left(1-e^{-2t}\right)\right)^{-\frac{d}{2}}\cdot \exp\left[\frac{-\left\|\vx-e^{-t}\vy\right\|^2}{2(1-e^{-2t})}\right] \der \vy\\
            = & \int e^{td}p_{k, S-(r\eta+t)}(e^{t}\vz)\cdot \left(2\pi \left(1-e^{-2t}\right)\right)^{-\frac{d}{2}}\cdot \exp\left[\frac{-\left\|\vx - \vz\right\|^2}{2(1-e^{-2t})}\right] \der \vz
        \end{aligned}
    \end{equation}
    where the last equation follows from setting $\vz \coloneqq e^{-t} \vy$.
    We define
    \begin{equation*}
        p^\prime_{k,S-(r\eta+t)}(\vz) \coloneqq e^{td}p_{k,S-(r\eta+t)}(e^t \vz)
    \end{equation*}
    which is also a density function.
    Therefore, for each element $\hat{\rvx}_{k,r\eta}=\vx$, we have
    \begin{equation*}
        \begin{aligned}
            \left\|\grad\log \frac{p_{k,S-(r\eta+t)}(\vx)}{p_{k,S-r\eta}(\vx)}\right\|^2 \le & 2\left\|\grad \log \frac{p_{k,S-(r\eta+t)}(\vx)}{p^\prime_{k,S-(r\eta+t)}(\vx)}\right\|^2 + 2\left\|\grad\log \frac{p^\prime_{k,S-(r\eta+t)}(\vx)}{p_{k,S-r\eta}(\vx)}\right\|^2\\
            = & 2\left\|\grad \log \frac{p_{k,S-(r\eta+t)}(\vx)}{p^\prime_{k,S-(r\eta+t)}(\vx)}\right\|^2 + 2\left\|\grad\log \frac{p^\prime_{k,S-(r\eta+t)}(\vx)}{p^\prime_{k,S-(r\eta+t)}\ast \varphi_{(1-e^{-2t})}(\vx)}\right\|^2
        \end{aligned}
    \end{equation*}
    where the last inequality follows from Eq~\ref{eq:fwd_trans_ker}. 
    For the first term, we have
    \begin{equation}
        \label{ineq:lemb1_term21}
        \begin{aligned}
            & \left\|\grad \log \frac{p_{k,S-(r\eta+t)}(\vx)}{p^\prime_{k,S-(r\eta+t)}(\vx)}\right\| =  \left\|\grad \log p_{k, S-(r\eta+t)}(\vx) - e^t \cdot \grad\log p_{k, S-(r\eta+t)}(e^t \vx)\right\|\\
            & \le \left\|\grad \log p_{k, S-(r\eta+t)}(\vx) - e^t \grad\log p_{k,S-(r\eta+t)}(\vx)\right\| \\
            &\quad + e^t\cdot \left\|\grad\log p_{k,S-(r\eta+t)}(\vx) - \grad\log p_{k, S-(r\eta+t)}(e^t \vx)\right\|\\
            & = (e^t -1)\cdot\left\|\grad\log p_{k,S-(r\eta+t)}(\vx)\right\| + e^t \cdot (e^t-1)L\left\|\vx\right\|.
        \end{aligned}
    \end{equation}

    To upper bound the latter term, we expect to employ Lemma~\ref{lem:lemc11_lee2022convergence}.
    However, it requires a specific condition which denotes the smoothness of $-\grad\log p^\prime_{k,S-(r\eta+t)}$ should be upper bounded with the variance of $\varphi_{(1-e^{-2t})}$ as
    \begin{equation*}
        \left\|-\grad^2\log p^\prime_{k,S-(r\eta+t)}\right\|\le \frac{1}{2(1-e^{-2t})},
    \end{equation*}
    which can be achieved by setting
    \begin{equation*}
        \eta\le \min\left\{\frac{1}{4L}, \frac{1}{2}\right\}.
    \end{equation*}
    Since the smoothness of $-\grad\log p_{k,S-(r\eta+t)}$, i.e., Assumption~\ref{con_ass:lips_score}, implies $- \grad\log p^\prime_{k, S-(r\eta+t)}$ is $e^{2t}L$-smooth.
    Besides, there are
    \begin{equation*}
        t \le \eta \le  \min\left\{\frac{1}{4L}, \frac{1}{2}\right\}\le \log\left(1+\frac{1}{2L}\right)\ \quad \text{and}\quad e^{2t}L\le \frac{1}{2(1-e^{-2t})}.
    \end{equation*}
    Therefore, we have
    \begin{equation}
        \label{ineq:lemb1_term22}
        \begin{aligned}
            & \left\|\grad\log p^\prime_{k,S-(r\eta+t)}(\vx) - \grad\log \left(p^\prime_{k,S-(r\eta+t)}\ast \varphi_{(1-e^{-2t})} \right)(\vx)\right\|\\
            \le & 6 e^{2t}L\sqrt{1-e^{-2t}}d^{1/2}+ 2e^{3t}L(1-e^{-2t})\left\|\grad\log p_{k,S-(r\eta+t)}(e^t\vx)\right\|\\
            \le & 6 e^{2t}L\sqrt{1-e^{-2t}}d^{1/2}+ 2L \cdot e^t (e^{2t}-1) \left\|\grad\log p_{k,S-(r\eta+t)}(\vx)\right\|\\
            & + 2L \cdot e^t (e^{2t}-1) \left\|\grad\log p_{k,S-(r\eta+t)}(e^t\vx) - \grad\log p_{k,S-(r\eta+t)}(\vx)\right\|\\
            \le & 6 e^{2t}L\sqrt{1-e^{-2t}}d^{1/2}+ 2L \cdot e^t (e^{2t}-1) \left\|\grad\log p_{k,S-(r\eta+t)}(\vx)\right\|\\
            & + 2L^2\cdot e^t(e^{2t}-1)(e^t-1)\left\|\vx\right\|,
        \end{aligned}
    \end{equation}
    where the first inequality follows from Lemma~\ref{lem:lemc11_lee2022convergence}, the last inequality follows from Assumption~\ref{con_ass:lips_score}.
    Due to the range, i.e., $\eta \le 1/2$, we have the following inequalities
    \begin{equation*}
        e^{2t}\le e^{2\eta}\le 1+4\eta\le 3,\quad 1-e^{-2t}\le 2t \le 2\eta \quad \text{and}\quad e^t\le e^\eta\le 1+\frac{3}{2}\cdot \eta.
    \end{equation*}
    In this condition, Eq~\ref{ineq:lemb1_term21} can be reformulated as
    \begin{equation*}
        \begin{aligned}
            \left\|\grad \log \frac{p_{k,S-(r\eta+t)}(\vx)}{p^\prime_{k,S-(r\eta+t)}(\vx)}\right\|^2\le & 2\left[ (e^t -1)^2\cdot\left\|\grad\log p_{k,S-(r\eta+t)}(\vx)\right\|^2 + e^{2t} \cdot (e^t-1)^2L^2\left\|\vx\right\|^2 \right]\\
            \le & 5\eta^2 \left\|\grad\log p_{k,S-(r\eta+t)}(\vx)\right\|^2 + 14 L^2 \eta^2\left\|\vx\right\|^2,
        \end{aligned}
    \end{equation*}
    and Eq~\ref{ineq:lemb1_term22} implies
    \begin{equation*}
        \begin{aligned}
            & \left\|\grad\log p^\prime_{k,S-(r\eta+t)}(\vx) - \grad\log \left(p^\prime_{k,S-(r\eta+t)}\ast \phi_{(1-e^{-2t})} \right)(\vx)\right\|^2\\
            \le & 3\cdot \left[6^2 e^{4t}L^2(1-e^{-2t})d + 4L^2 e^{2t}(e^{2t}-1)^2 \left\|\grad\log p_{k,S-(r\eta+t)}(\vx)\right\|^2 + 4L^4 e^{2t}(e^{2t}-1)^2(e^t-1)^2\left\|\vx\right\|^2 \right]\\
            \le & 3\cdot \left[ 2^3\cdot 3^4L^2\eta d + 2^6\cdot 3L^2\eta^2\left\|\grad\log p_{k,S-(r\eta+t)}(\vx)\right\|^2 + 3^3\cdot 2^4L^4 \eta^4\left\|\vx\right\|^2\right]\\
            \le & 2^3 \cdot 3^5 L^2\eta d + 2^6\cdot  3^2 L^2\eta^2 \left\|\grad\log p_{k,S-(r\eta+t)}(\vx)\right\|^2+3^4\cdot L^2\eta^2\left\|\vx\right\|^2,
        \end{aligned}
    \end{equation*}
    where the last inequality follows from $\eta L \le 1/4$. 
    Hence, suppose $L\ge 1$ without loss of generality, we have
    \begin{equation*}
        \begin{aligned}
            &\text{Term 2.1}\le 2\cdot\left(\E\left[\left\|\grad \log \frac{p_{k,S-(r\eta+t)}(\hat{\rvx}_{k,r\eta})}{p^\prime_{k,S-(r\eta+t)}(\hat{\rvx}_{k,r\eta})}\right\|^2\right]+ \E \left[\left\|\grad\log \frac{p^\prime_{k,S-(r\eta+t)}(\hat{\rvx}_{k,r\eta})}{p_{k,S-r\eta}(\hat{\rvx}_{k,r\eta})}\right\|^2\right]\right)\\
            & \le 2^4\cdot 3^5 L^2 \eta d +  2^8 \cdot 3^2 L^2\eta^2 \E\left[\left\|\grad\log p_{k,S-(r\eta+t)}(\hat{\rvx}_{k,r\eta})\right\|^2\right] + 2^2\cdot 3^4 L^2\eta^2\E\left[\left\|\hat{\rvx}_{k,r\eta}\right\|^2\right]\\
            & \le 2^{14}L^2\eta d + 2^{13}L^2\eta^2 \E\left[\left\|\grad\log p_{k,S-(r\eta+t)}(\hat{\rvx}_{k,r\eta+t})\right\|^2\right] + 2^{13}L^4\eta^2 \E \left[\left\|\hat{\rvx}_{k,r\eta+t}-\hat{\rvx}_{k, r\eta}\right\|^2\right]\\
            &\quad + 2^{10} L^2\eta^2 \E\left[\left\|\hat{\rvx}_{k,r\eta}\right\|^2\right].
        \end{aligned}
    \end{equation*}
    Therefore, we have
    \begin{equation*}
        \begin{aligned}
            \text{Term 2}\le & 2^{14}L^2\eta d + 2^{10} L^2\eta^2 \E\left[\left\|\hat{\rvx}_{k,r\eta}\right\|^2\right] + 2^{13}L^2\eta^2 \E\left[\left\|\grad\log p_{k,S-(r\eta+t)}(\hat{\rvx}_{k,r\eta+t})\right\|^2\right]\\
            &+ \left(2^{13}L^2\eta^2 +1\right) L^2\E \left[\left\|\hat{\rvx}_{k,r\eta+t}-\hat{\rvx}_{k, r\eta}\right\|^2\right]\\
            \le & 2^{14}L^2\eta d + 2^{10}L^2 \eta^2(M+d) + 2^{13}L^3\eta^2 d + 2^{10}L^2\left(2(M+d)\eta^2 + 4d\eta\right)
        \end{aligned}
    \end{equation*}
    where the last inequality follows from Lemma~\ref{lem:lem9_chen2022sampling} and Lemma~\ref{lem:var_lem10_chen2022sampling}. 
    To diminish the discretization error, we require the step size of backward sampling, i.e., $\eta$ satisfies
    \begin{equation*}
        \left\{
        \begin{aligned}
            & 2^{14} L^2\eta d \le \epsilon\\
            & 2^{10}\cdot L^2\eta^2(d+M) \le \epsilon\\
            & 2^{13}\cdot L^3\eta^2 d \le \epsilon\\ 
            & 2^{10}\cdot L^2\left(2(M+d)\eta^2 + 4d\eta\right) \le \epsilon
        \end{aligned}
        \right. \quad \Leftarrow\quad \left\{
        \begin{aligned}
            & \eta\le 2^{-14} L^{-2}d^{-1}\epsilon \\
            & \eta\le 2^{-5}\cdot L^{-1}\left(d+M\right)^{-0.5}\epsilon^{0.5}\\
            & \eta\le 2^{-6.5}\cdot L^{-1.5}d^{-0.5}\epsilon^{0.5}\\
            & \eta \le 2^{-6}L^{-0.5}\left(d+M\right)^{-0.5}\epsilon^{0.5}\\
            & \eta \le 2^{-13}L^{-2}d^{-1}\epsilon.
        \end{aligned}
        \right.
    \end{equation*}
    Specifically, if we choose 
    \begin{equation*}
        \eta\le 2^{-14} L^{-2}\left(d+M\right)^{-1}\epsilon = C_\eta (d+M)^{-1}\epsilon,
    \end{equation*}
    we have
    \begin{equation*}
        \E\left[L^2\left\|\hat{\rvx}_{k,t+r\eta}-\hat{\rvx}_{k,r\eta}\right\|^2\right] + \E\left[\left\|\grad\log \frac{p_{k,S-r\eta}(\hat{\rvx}_{k,r\eta})}{p_{k, S-(t+r\eta)}(\hat{\rvx}_{k,r\eta})}\right\|^2\right]\le 4\epsilon,
    \end{equation*}
    and the proof is completed.
\end{proof}

%% file: 0_contents/0Xappendix_OU/0ELemScr.tex
\section{Lemmas for Bounding Score Estimation Error}
\label{app_sec:lems_rec_ite}

\begin{lemma}[Recursive Form of Score Functions]
    \label{lem:score_reformul}
    Under the notation in Section~\ref{sec:not_ass_app}, for any $k\in \mathbb{N}_{0,K-1}$ and $t\in[0,S]$, the score function can be written as
    \begin{equation*}
        \grad_{\vx} \log p_{k,S-t}(\vx) = \mathbb{E}_{\rvx^\prime \sim q_{k,S-t}(\cdot|\vx)}\left[-\frac{\vx - e^{-(S-t)}\rvx^\prime}{\left(1-e^{-2(S-t)}\right)}\right]
    \end{equation*}
    where the conditional density function $q_{k,S-t}(\cdot | \vx)$ is defined as
    \begin{equation*}
        q_{k,S-t}(\vx^\prime|\vx) \propto \exp\left(\grad\log p_{k,0}(\vx^\prime)-\frac{\left\|\vx - e^{-(S-t)}\vx^\prime\right\|^2}{2\left(1-e^{-2(S-t)}\right)}\right).
    \end{equation*}
\end{lemma}
\begin{proof}
    When the OU process, i.e., SDE~\ref{con_eq:rrds_forward}, is selected as the forward path, for any $k\in\mathbb{N}_{0,K}$ and $t\in[0,S]$, the transition kernel has a closed form, i.e.,
    \begin{equation*}
        p_{k, t|0}(\vx |\vx_0) = \left(2\pi \left(1-e^{-2t}\right)\right)^{-d/2}
     \cdot \exp \left[\frac{-\left\|\vx -e^{-t}\vx_0 \right\|^2}{2\left(1-e^{-2t}\right)}\right], \quad \forall\ 0\le t\le S.
    \end{equation*}
    In this condition, we have
    \begin{equation*}
        \begin{aligned}
        p_{k,S-t}(\vx) = &\int_{\R^d} p_{k,0}(\vx_0) \cdot p_{k,S-t|0}(\vx|\vx_0)\der \vx_0\\
        = & \int_{\R^d} p_{k,0}(\vx_0)\cdot \left(2\pi \left(1-e^{-2(S-t)}\right)\right)^{-d/2}
     \cdot \exp \left[\frac{-\left\|\vx -e^{-(S-t)}\vx_0 \right\|^2}{2\left(1-e^{-2(S-t)}\right)}\right]\der \vx_0
        \end{aligned}
    \end{equation*}
    Plugging this formulation into the following equation
    \begin{equation*}
        \grad_{\vx} \log p_{k, S-t}(\vx) = \frac{\grad p_{k,S-t}(\vx)}{p_{k,S-t}(\vx)},
    \end{equation*}
    we have
    \begin{equation}
        \label{equ:grad_ln_pt}
        \begin{aligned}
            \grad_{\vx} \log p_{k,S-t}(\vx) = &\frac{\grad \int_{\R^d} p_{k,0}(\vx_0)\cdot \left(2\pi \left(1-e^{-2(S-t)}\right)\right)^{-d/2} \cdot \exp \left[\frac{-\left\|\vx -e^{-(S-t)}\vx_0 \right\|^2}{2\left(1-e^{-2(S-t)}\right)}\right]\der \vx_0}{\int_{\R^d} p_{k,0}(\vx_0)\cdot \left(2\pi \left(1-e^{-2(S-t)}\right)\right)^{-d/2} \cdot \exp \left[\frac{-\left\|\vx -e^{-(S-t)}\vx_0 \right\|^2}{2\left(1-e^{-2(S-t)}\right)}\right]\der \vx_0}\\
            = &\frac{\int_{\R^d} p_{k,0}(\vx_0) \cdot \exp\left(\frac{-\left\|\vx - e^{-(S-t)}\vx_0\right\|^2}{2\left(1-e^{-2(S-t)}\right)}\right) \cdot \left(-\frac{\vx - e^{-(S-t) }\vx_0}{\left(1-e^{-2(T-t)}\right)}\right)\der \vx_0}{\int_{\R^d} p_{k,0}(\vx_0)\cdot \exp\left(\frac{-\left\|\vx - e^{-(S-t)}\vx_0\right\|^2}{2\left(1-e^{-2(S-t)}\right)}\right)\der \vx_0}\\
            = & \mathbb{E}_{\rvx_0 \sim q_{k,S-t}(\cdot|\vx)}\left[-\frac{\vx - e^{-(S-t)}\rvx_0}{\left(1-e^{-2(S-t)}\right)}\right]
        \end{aligned}
    \end{equation}
    where the density function $q_{T-t}(\cdot |\vx)$ is defined as
    \begin{equation*}
        \begin{aligned}
            q_{k,S-t}(\vx_0|\vx) = & \frac{p_{k,0}(\vx_0) \cdot \exp\left(\frac{-\left\|\vx - e^{-(S-t)}\vx_0\right\|^2}{2\left(1-e^{-2(S-t)}\right)}\right) }{\int_{\R^d} p_{k,0}(\vx_0) \cdot \exp\left(\frac{-\left\|\vx - e^{-(S-t)}\vx_0\right\|^2}{2\left(1-e^{-2(S-t)}\right)}\right) \der \vx_0}\\ 
            \propto & \exp\left(-f_{k,0}(\vx_0)-\frac{\left\|\vx - e^{-(S-t)}\vx_0\right\|^2}{2\left(1-e^{-2(S-t)}\right)}\right),
        \end{aligned}
    \end{equation*}
    where $p_{k,0}\propto \exp(-f_{k,0})$.
    Hence, the proof is completed.
\end{proof}

\begin{lemma}[Strong log-concavity and L-smoothness of the auxiliary targets]
    \label{lem:sm_con_of_q}
    Under the notation in Section~\ref{sec:not_ass_app}, for any $k\in \mathbb{N}_{0,K-1}$, $r\in\mathbb{N}_{0,R-1}$ and $\vx\in\R^d$, we define the auxiliary target distribution as
    \begin{equation*}
        q_{k,S-r\eta}(\vx^\prime|\vx) \propto \exp\left(\grad\log p_{k,0}(\vx^\prime)-\frac{\left\|\vx - e^{-(S-r\eta)}\vx^\prime\right\|^2}{2\left(1-e^{-2(S-r\eta)}\right)}\right).
    \end{equation*}
    We define 
    \begin{equation*}
        \mu_r\coloneqq \frac{1}{2}\cdot \frac{e^{-2(S-r\eta)}}{1-e^{-2(S-r\eta)}} \quad \mathrm{and}\quad L_r\coloneqq \frac{3}{2}\cdot \frac{e^{-2(S-r\eta)}}{1-e^{-2(S-r\eta)}}.
    \end{equation*}
    Then, we have
    \begin{equation*}
        \mu_r\mI \preceq -\grad^2\log q_{k,S-r\eta}(\vx^\prime|\vx) \preceq L_r\mI
    \end{equation*}
    when the segment length $S$ satisfies $S=\frac{1}{2}\log\left(\frac{2L+1}{2L}\right)$.
\end{lemma}
\begin{proof}
    We begin with the formulation of $\grad^2\log q_{k,S-t}$, i.e.,
    \begin{equation}
        \label{eq:axu_hessian}
        -\grad^2 \log q_{k,S-r\eta}(\vx^\prime|\vx) = -\grad^2 \log p_{k,0}(\vx^\prime) + \frac{e^{-2(S-r\eta)}}{1-e^{-2(S-r\eta)}}\mI.
    \end{equation}
    By supposing $S=\frac{1}{2}\log\left(\frac{2L+1}{2L}\right)$, we have
    \begin{equation*}
        \frac{e^{-2(S-r\eta)}}{1-e^{-2(S-r\eta)}} \ge \frac{e^{-2S}}{1-e^{-2S}} = 2L \ge 2\left\|\grad^2 \log p_{k,0}\right\|.
    \end{equation*}
    Plugging this inequality into Eq~\ref{eq:axu_hessian}, we have
    \begin{equation*}
        \begin{aligned}
            &-\grad^2 p_{k,0}(\vx^\prime) +\frac{e^{-2(S-r\eta)}}{1-e^{-2(S-r\eta)}}\cdot \mI \preceq \left(\left\|\grad^2\log p_{k,0}(\vx^\prime)\right\| + \frac{e^{-2(S-r\eta)}}{1-e^{-2(S-r\eta)}} \right)\cdot \mI\\
            & \preceq \frac{3}{2}\cdot \frac{e^{-2(S-r\eta)}}{1-e^{-2(S-r\eta)}}\cdot \mI = L_r \mI.
        \end{aligned}
    \end{equation*}
    Besides, it has
    \begin{equation*}
        \begin{aligned}
            &-\grad^2 p_{k,0}(\vx^\prime) +\frac{e^{-2(S-r\eta)}}{1-e^{-2(S-r\eta)}}\cdot \mI \succeq \left(-\left\|\grad^2 \log p_{k,0}(\vx^\prime)\right\|+\frac{e^{-2(S-r\eta)}}{1-e^{-2(S-r\eta)}}\right)\cdot \mI\\
            & \succeq \frac{1}{2}\cdot \frac{e^{-2(S-r\eta)}}{1-e^{-2(S-r\eta)}}\cdot \mI = \mu_r \mI.
        \end{aligned}
    \end{equation*}
    Hence, the proof is completed.
\end{proof}

\subsection{Score Estimation Error from Empirical Mean}

\begin{lemma}
    \label{lem:concen_prop_v2}
    With a little abuse of notation, for each $i\in \mathbb{N}_{1,n_{k,r}}$ in Alg~\ref{con_alg:rse}, we denote the underlying distribution of \textbf{output particles} as $\rvx^\prime_{i}\sim q^\prime_{k,S-r\eta}$ and suppose it satisfies LSI with the constant $\mu_r^\prime$.
    Then, for any $\vx\in \R^d$, we have
    \begin{equation*}
        \sP\left[\left\|-\frac{1}{n_{k,r}}\sum_{i=1}^{n_{k,r}} \rvx_i^\prime + \E_{\rvx^\prime \sim q^\prime_{k,S-r\eta}(\cdot|\vx)}\left[\rvx^\prime\right]\right\|\le 2\epsilon^\prime\right] \ge 1-\delta
    \end{equation*}
    by requiring the sample number $n_{k,r}$ to satisfy
    \begin{equation*}
        n_{k,r} \ge \frac{\max\left\{d, -2\log \delta\right\}}{\mu_r^\prime \epsilon^{\prime 2}}.
    \end{equation*}
\end{lemma}
\begin{proof}
    For any $\vx\in \R^d$, we set
    \begin{equation*}
        \vb^\prime \coloneqq \E_{q^\prime_{k,S-r\eta}(\cdot|\vx)}\left[\rvx^\prime\right]\quad \text{and}\quad  \sigma^\prime\coloneqq \E_{\left\{\rvx^\prime_i\right\}_{i=1}^{n_{k,r}}\sim q^{\prime(n_{k,r})}_{k,S-r\eta}(\cdot|\vx)  }\left[\left\|\sum_{i=1}^{n_{k,r}}\rvx^\prime_i - \E\left[\sum_{i=1}^{n_{k,r}}\rvx^\prime_i\right]\right\|\right].
    \end{equation*}
    We begin with the following probability
    \begin{equation}
        \begin{aligned}
            &\sP_{\left\{\rvx^\prime_i\right\}_{i=1}^{n_{k,r}}\sim q^{\prime(n_{k,r})}_{k,S-r\eta}(\cdot|\vx)  }\left[\left\|-\frac{1}{n_{k,r}}\sum_{i=1}^{n_{k,r}} \rvx_i^\prime + \E_{\rvx^\prime \sim q^\prime_{k,S-r\eta}(\cdot|\vx)}\left[\rvx^\prime\right]\right\|^2\ge \left(\frac{\sigma^\prime}{n_{k,r}} + \epsilon^\prime\right)^2 \right]\\
            = &  \sP_{\left\{\rvx^\prime_i\right\}_{i=1}^{n_{k,r}}\sim q^{\prime(n_{k,r})}_{k,S-r\eta}(\cdot|\vx)  }\left[\left\|\sum_{i=1}^{n_{k,r}}\rvx^\prime_i - n_{k,r}\vb^\prime\right\| \ge \sigma^\prime + n_{k,r}\epsilon^\prime\right]
        \end{aligned}
    \end{equation}
    To lower bound this probability, we expect to utilize Lemma~\ref{lem:lsi_concentration} which requires the following two conditions:
    \begin{itemize}
        \item The distribution of $\sum_{i=1}^{n_{k,r}}\rvx^\prime_i$ satisfies LSI, and its LSI constant can be obtained.
        \item The formulation $\left\|\sum_{i=1}^{n_{k,r}}\rvx^\prime_i - n_{k,r}\vb^\prime\right\| \ge \sigma^\prime + n_{k,r}\epsilon^\prime$ can be presented as $F\ge \E[F]+\mathrm{bias}$ where $F$ is a $1$-Lipschitz function.
    \end{itemize}
    For the first condition, by employing Lemma~\ref{lem:cor31_chafai2004entropies}, we have that the LSI constant of 
    \begin{equation*}
        \sum_{i=1}^{n_{k,r}}\rvx_i^\prime \sim \underbrace{q^\prime_{k,S-r\eta}(\cdot|\vx) \ast q^\prime_{k,S-r\eta}(\cdot|\vx) \cdots \ast q^\prime_{k,S-r\eta}(\cdot|\vx)}_{n_{k,r}}
    \end{equation*}
    is $\mu^\prime_r/n_{k,r}$.
    For the second condition, we set the function $F\left( \vx \right) = \left\| \vx - n_{k,r}\vb^\prime\right\| \colon \R^{d}\rightarrow \R$ is $1$-Lipschitz because
    \begin{equation*}
        \left\|F\right\|_{\mathrm{Lip}}=\sup_{\vx \not= \vy} \frac{\left|F(\vx)-F(\vy)\right|}{\left\|\vx - \vy\right\|}=\sup_{\vx\not = \vy}\frac{\left|\left\| \vx \right\| - \left\| \vy \right\|\right|}{\left\| \left(\vx - \vy \right)\right\|} = 1.
    \end{equation*}
    Besides, we have 
    \begin{equation*}
        F\left(\sum_{i=1}^{n_{k,r}}\rvx_i^\prime\right) = \left\|\sum_{i=1}^{n_{k,r}}\rvx^\prime_i - n_{k,r}\vb^\prime\right\|\quad \mathrm{and}\quad \E\left[F\left(\sum_{i=1}^{n_{k,r}}\rvx_i^\prime\right)\right] = \sigma^\prime
    \end{equation*}
    where the second equation follows from the definition of $\sigma^\prime$.
    Therefore, with Lemma~\ref{lem:lsi_concentration}, we have 
    \begin{equation}
        \label{ineq:prod_concentration}
        \sP_{\left\{\rvx_i^\prime\right\}_{i=1}^{n_{k,r}} \sim q^{\prime(n_{k,r})}_{k,S-r\eta}(\cdot|\vx) }\left[\left\|\sum_{i=1}^{n_{k,r}} \rvx_i^\prime- n_{k,r}\vb^\prime \right\|\ge \sigma^\prime + n_{k,r}\epsilon^\prime\right]\le \exp\left(- \frac{\mu^\prime_r \epsilon^{\prime 2} n_{k,r}}{2}\right).
    \end{equation}
    
    Then, we consider the range of $\sigma^\prime$ and have
    \begin{equation}
        \label{ineq:prod_var_bound}
        \begin{aligned}
            \sigma^\prime =  & n_{k,r}\cdot \mathbb{E}_{\left\{\rvx_i^\prime\right\}_{i=1}^{n_{k,r}} \sim q^{\prime(n_{k,r})}_{k,S-r\eta}(\cdot|\vx)}\left\| \frac{1}{n_{k,r}}\sum_{i=1}^{n_{k,r}} \rvx_i^\prime - \vb^\prime\right\|\\
            \le & n_{k,r}\cdot \sqrt{\mathrm{var}\left(\frac{1}{n_{k,r}}\sum_{i=1}^{n_{k,r}} \rvx_i^\prime\right)}= \sqrt{n_{k,r} \mathrm{var}\left(\rvx_i^\prime\right)} \le \sqrt{\frac{n_{k,r} d}{\mu^\prime_r}},
        \end{aligned}
    \end{equation}
    the first inequality follows from Holder's inequality and the last follows from Lemma~\ref{lem:lsi_var_bound}.
    Combining Eq~\ref{ineq:prod_concentration} and Eq~\ref{ineq:prod_var_bound}, it has
    \begin{equation*}
        \small
        \sP_{\left\{\rvx^\prime_i\right\}_{i=1}^{n_{k,r}}\sim q^{\prime(n_{k,r})}_{k,S-r\eta}(\cdot|\vx)  }\left[\left\|-\frac{1}{n_{k,r}}\sum_{i=1}^{n_{k,r}} \rvx_i^\prime + \E_{\rvx^\prime \sim q^\prime_{k,S-r\eta}(\cdot|\vx)}\left[\rvx^\prime\right]\right\|^2\ge \left(\sqrt{\frac{d}{\mu_r^\prime n_{k,r}}} + \epsilon^\prime\right)^2 \right] \le \exp\left(- \frac{\mu^\prime_r \epsilon^{\prime 2} n_{k,r}}{2}\right).
    \end{equation*}
    By requiring 
    \begin{equation}
        \label{ineq:sample_n_req}
        \frac{d}{\mu_r^\prime n_{k,r}} \le \epsilon^{\prime 2} \quad\text{and}\quad -\frac{\mu_r^\prime \epsilon^{\prime 2} n_{k,r}}{2}\le \log \delta,
    \end{equation}
    we have
    \begin{equation*}
        \begin{aligned}
            &\sP\left[\left\|-\frac{1}{n_{k,r}}\sum_{i=1}^{n_{k,r}} \rvx_i^\prime + \E_{\rvx^\prime \sim q^\prime_{k,S-r\eta}(\cdot|\vx)}\left[\rvx^\prime\right]\right\|\le 2\epsilon^\prime\right]\\
            & =  1- \sP\left[\left\|-\frac{1}{n_{k,r}}\sum_{i=1}^{n_{k,r}} \rvx_i^\prime + \E_{\rvx^\prime \sim q^\prime_{k,S-r\eta}(\cdot|\vx)}\left[\rvx^\prime\right]\right\|\ge 2\epsilon^\prime\right] \ge 1- \delta.
        \end{aligned}
    \end{equation*}
    Noted that Eq.~\ref{ineq:sample_n_req} implies the sample number $n_{k,r}$ should satisfy
    \begin{equation*}
        n_{k,r}\ge \frac{d}{\mu_r^\prime\epsilon^{\prime 2}}\quad\text{and}\quad n_{k,r}\ge \frac{2\log \delta^{-1}}{\mu_r^\prime \epsilon^{\prime 2}}.
    \end{equation*}
    Hence, the proof is completed.
\end{proof}

\subsection{Score Estimation Error from Mean Gap}

\begin{lemma}
    \label{lem:inner_err_conv_v2}
    For any given $(k,r,\vx)$ in Alg~\ref{con_alg:rse}, suppose the distribution $q_{k,S-r\eta}(\cdot|\vx)$ satisfies
    \begin{equation*}
        \mu_r\mI \preceq -\grad^2\log q_{k,S-r\eta}(\cdot|\vx) \preceq L_r\mI,
    \end{equation*}
    and $\rvx_{j}^\prime \sim q_j^\prime(\cdot |\vx)$ corresponds to Line~\ref{con_alg:rse_inner_update} of Alg~\ref{con_alg:rse}.
    If $0<\tau_r \le \mu_r/(8L^2_r)$, we have
    \begin{equation*}
        \KL{q^\prime_{j+1}(\cdot|\vx)}{q_{k,S-r\eta}(\cdot|\vx)}\le e^{-\mu_r\tau_r}\KL{q^\prime_j(\cdot|\vx)}{q_{k,S-r\eta}(\cdot|\vx)} + 28L_r^2 d\tau_r^2
    \end{equation*}
    when the score estimation satisfies $\left\|\grad\log p_{k,0} -\vv^\prime\right\|_\infty \le L_r\sqrt{2d\tau_r}$.
\end{lemma}

\begin{proof}
    Suppose the loop in Line~\ref{con_alg:rse_inner_loops} of Alg~\ref{con_alg:rse} aims to draw a sample from the target distribution $q_{k,S-r\eta}(\cdot|\vx)$ satisfying
    \begin{equation*}
        q_{k,S-r\eta}(\vx^\prime|\vx) \propto \exp(-g_{k,r}(\vx^\prime)) \coloneqq \exp\left(-f_{k,0}(\vx^\prime)-\frac{\left\|\vx-e^{-(S-r\eta)}\vx^\prime\right\|^2}{2(1-e^{-2(S-r\eta)})}\right).
    \end{equation*}
    The score function of the target, i.e., $\grad g_{k,r}(\vx^\prime)$, satisfies
    \begin{equation*}
        \grad g_{k,r}(\vx^\prime) = \grad f_{k,0}(\vx^\prime) + \frac{-e^{-(S-r\eta)}\vx + e^{-2(S-r\eta)}\vx^\prime}{1-e^{-2(S-r\eta)}}.
    \end{equation*}
    At the $j$-th iteration corresponding to Line~\ref{con_alg:rse_inner_update} in Alg~\ref{con_alg:rse}.
    The previous score is approximated by
    \begin{equation*}
        \grad g^\prime(\vx^\prime)=\vv^\prime(\vx^\prime) + \frac{-e^{-(S-r\eta)}\vx + e^{-2(S-r\eta)}\vx^\prime}{1-e^{-2(S-r\eta)}}.
    \end{equation*}
    where $\vv^\prime(\cdot)$ is used to approximate $\grad\log p_{k,0}(\cdot)$ by calling Alg~\ref{con_alg:rse} recursively.
    Suppose $\rvx^\prime_{j}=\vz_0$, the $j$-th iteration is equivalent to the following SDE
    \begin{equation*}
    \begin{aligned}
        &\der \rvz_t = -\grad g^\prime(\vz_0)\der t + \sqrt{2}\der B_t,
    \end{aligned}
    \end{equation*}
    we denote the underlying distribution of $\rvz_t$ as $q_t$.
    Similarly, we set $q_{0t}$ as the joint distribution of $(\rvz_0, \rvz_t)$, and have
    \begin{equation*}
        q_{0t}(\vz_0,\vz_t) = q_0(\vz_0)\cdot q_{t|0}(\vz_t|\vz_0).
    \end{equation*}
    According to the Fokker-Planck equation, we have
    \begin{equation*}
        \partial_t q_{t|0}(\vz_t|\vz_0)= \grad\cdot\left(q_{t|0}(\vz_t|\vz_0)\cdot \grad g^\prime(\vz_0)\right) + \Delta q_{t|0}(\vz_t|\vz_0)
    \end{equation*}
    In this condition, we have
    \begin{equation*}
        \begin{aligned}
            \partial_t q_{t}(\vz_t) = &\int \frac{\partial q_{t|0}(\vz_t|\vz_0)}{\partial t} \cdot q_0(\vz_0)\der\vz_0\\
            = & \int \left[\grad \cdot \left(q_{t|0}(\vz_t|\vz_0)\cdot \grad g^\prime(\vz_0)\right) +\Delta q_{t|0}(\vz_t|\vz_0)\right]\cdot q_0(\vz_0)\der\vz_0\\
            = & \grad\cdot \left(q_t(\vz_t)\int q_{0|t}(\vz_0|\vz_t) \grad g^\prime(\vz_0)\der\vz_0\right)+ \Delta q_t(\vz_t).
        \end{aligned}
    \end{equation*}
    For abbreviation, we suppose
    \begin{equation*}
        q_*(\cdot)\coloneqq q_{k,S-r\eta}(\cdot|\vx)\quad \text{and}\quad g_*\coloneqq g_{k,r}.
    \end{equation*}
    With these notations, the dynamic of the KL divergence between $q_t$ and $q_*$ is
    \begin{equation}
        \label{ineq:kl_descent}
        \begin{aligned}
            &\partial_t\KL{q_t}{q_*} = \int \partial_t q_t(\vz_t)\log \frac{q_t(\vz_t)}{q_*(\vz_t)}\der\vz_t\\
            = & \int \grad\cdot \left[q_t(\vz_t)\left(\int q_{0|t}(\vz_0|\vz_t)\grad g^\prime(\vz_0)\der\vz_0 + \grad\log q_t(\vz_t)\right)\right]\cdot \log \frac{q_t(\vz_t)}{q_*(\vz_t)}\der\vz_t\\
            = & -\int q_t(\vz_t)\left(\left\|\grad\log \frac{q_t(\vz_t)}{q_*(\vz_t)}\right\|^2+\left<\int q_{0|t}(\vz_0|\vz_t)\grad g^\prime(\vz_0)\der\vz_0+\grad\log q_*(\vz_t),\grad\log\frac{q_t(\vz_t)}{q_*(\vz_t)}\right>\right) \der\vz_t\\
            = & -\int q_t(\vz_t)\left\|\grad \log \frac{q_t(\vz_t)}{q_*(\vz_t)}\right\|^2 \der\vz_t + \int q_{0t}(\vz_0, \vz_t)\left<\grad g^\prime(\vz_0)-\grad g_*(\vz_t),\grad \log \frac{q_t(\vz_t)}{q_*(\vz_t)}\right>\der (\vz_0, \vz_t)\\
            \le & -\frac{3}{4}\int q_t(\vz_t)\left\|\grad\log \frac{q_t(\vz_t)}{q_*(\vz_t)}\right\|^2 \der \vz_t + \int q_{0t}(\vz_0,\vz_t)\left\|\grad g^\prime(\vz_0) - \grad g_*(\vz_t)\right\|^2 \der (\vz_0,\vz_t)\\
            \le & -\frac{3}{4}\int q_t(\vz_t)\left\|\grad\log \frac{q_t(\vz_t)}{q_*(\vz_t)}\right\|^2 + 2\int q_{0t}(\vz_0,\vz_t)\left\|\grad g^\prime(\vz_0) - \grad g_*(\vz_0)\right\|^2 \der (\vz_0,\vz_t)\\
            & + 2\int q_{0t}(\vz_0,\vz_t)\left\|\grad g_*(\vz_0) - \grad g_*(\vz_t)\right\|^2 \der (\vz_0,\vz_t).
        \end{aligned}
    \end{equation}
    \paragraph{Upper bound the first term in Eq~\ref{ineq:kl_descent}.} The target distribution $q_*$ satisfies $\mu_r$-strong convexity, i.e.,
    \begin{equation*}
        \mu_r\mI \preceq  -\grad^2 \log q_{k, S-r\eta}(\vx^\prime|\vx)=-\grad^2 \log(q_*(\vx^\prime)),
    \end{equation*}
    It means $q_*$ satisfies LSI with the constant $\mu_r$ due to Lemma~\ref{lem:strongly_lsi}.
    Hence, we have
    \begin{equation}
        \label{ineq:eq17_term1_upb}
        -\frac{3}{4}\int q_t(\vz_t)\left\|\grad\log \frac{q_t(\vz_t)}{q_*(\vz_t)}\right\|^2\le -\frac{3\mu_r}{2}\KL{q_t}{q_*}.
    \end{equation}

    \paragraph{Upper bound the second term in Eq~\ref{ineq:kl_descent}.} We assume that there is a uniform upper bound $\epsilon_g$ satisfying 
    \begin{equation}
        \label{ineq:eq17_term2_upb}
        \left\|\grad g^\prime(\vz)-\grad g_*(\vz)\right\| \le \epsilon_g\quad \Rightarrow\quad \int q_{0t}(\vz_0,\vz_t)\left\|\grad g^\prime(\vz_0) - \grad g_*(\vz_0)\right\|^2 \der (\vz_0,\vz_t)\le \epsilon_g^2.
    \end{equation}

    \paragraph{Upper bound the third term in Eq~\ref{ineq:kl_descent}.} Due to the monotonicity of $e^{-t}/(1-e^{-t})$, we have
    \begin{equation*}
      2L \le \frac{e^{-2(S-r\eta)}}{1-e^{-2(S-r\eta)}} \le \frac{e^{-2\eta}}{1-e^{-2\eta}}\le \eta^{-1}
    \end{equation*}
    where we suppose $\eta\le 1/2$ without loss of the generality to establish the last inequality.
    Hence, the target distribution $q_*$ satisfies
    \begin{equation*}
        \begin{aligned}
            &-\grad^2 \log q_* = -\grad^2 \log q_{k, S-r\eta}(\cdot |\vx)=-\grad^2\log p_{k,0} + \frac{e^{-2(S-r\eta)}}{1-e^{-2(S-r\eta)}} \\
            \preceq & \left\|\grad^2\log p_{k,0}\right\|\mI + \frac{e^{-2(S-r\eta)}}{1-e^{-2(S-r\eta)}}\mI \coloneqq L_r \mI  \preceq (L+ \eta^{-1})\mI,
        \end{aligned}
    \end{equation*}
    where the last inequality follows from Assumption~\ref{con_ass:lips_score}.
    This result implies the smoothness of $q_*$, and we have
    \begin{equation}
        \label{ineq:eq17_term3_upb}
        \begin{aligned}
            & \int q_{0t}(\vz_0,\vz_t)\left\|\grad g_*(\vz_0) - \grad g_*(\vz_t)\right\|^2 \der (\vz_0,\vz_t)\\
            \le & L_r^2 \int q_{0t}(\vz_0, \vz_t)\left\|\vz_t - \vz_0\right\|^2 \der (\vz_0,\vz_t)= L_r^2 \cdot \E_{q_{0t}}\left[\left\|-t \grad g^\prime(\vz_0)+\sqrt{2t}\xi\right\|^2\right]\\
            = & L^2_r\cdot \left(2td +t^2\E_{q_{0}}\left\|\grad g^\prime(\vz_0)-\grad g_*(\vz_0)+\grad g_*(\vz_0)\right\|^2\right) \\
            \le & 2L_r^2 \cdot \left(td + t^2\epsilon_g^2 + t^2\E_{q_0}\left\|\grad g_*(\vz_0)\right\|^2\right)\\
            \le & 2L_r^2 dt +2L_r^2\epsilon_g^2 t^2+4L_r^3 dt^2 + \frac{8L_r^4 t^2}{\mu_r}\KL{q_0}{q_*},
        \end{aligned}
    \end{equation}
    where the last inequality follows from Lemma~\ref{lem:rapidlem12}.
    
    Hence, Combining Eq~\ref{ineq:kl_descent}, Eq~\ref{ineq:eq17_term1_upb}, Eq~\ref{ineq:eq17_term2_upb}, Eq~\ref{ineq:eq17_term3_upb} with $t\le \tau_r \le 1/(2L_r)$ and $\epsilon^2_g \le 2L_r^2 d\tau_r$, we have
    \begin{equation*}
        \begin{aligned}
            \partial_t\KL{q_t}{q_*}
            \le & -\frac{3\mu_r}{2}\KL{q_t}{q_*} + 2\epsilon_g^2 + \frac{16L_r^4 t^2}{\mu_r}\KL{q_0}{q_*}+ 4L_r^2 dt + 4L_r^2 \epsilon_g^2 t^2 + 8L_r^3 dt^2\\
            \le & -\frac{3\mu_r}{2}\KL{q_t}{q_*} + 4L_r^2 d\tau_r+\frac{16L_r^4 \tau_r^2}{\mu_r}\KL{q_0}{q_*}+ 4L_r^2d\tau_r + 8L_r^4 d\tau_r^3 + 8L_r^3 d \tau_r^2\\
            \le  &-\frac{3\mu_r}{2}\KL{q_t}{q_*}+\frac{16L_r^4 t^2}{\mu_r}\KL{q_0}{q_*}+14L_r^2d\tau_r. 
        \end{aligned}
    \end{equation*}
    Multiplying both sides by $\exp(\frac{3\mu_r t}{2})$, then the previous inequality can be written as
    \begin{equation*}
        \frac{\der}{\der t}\left(e^{\frac{3\mu_r t}{2}}\KL{q_t}{q_*}\right)\le e^{\frac{3\mu_r t}{2}}\cdot \left(\frac{16L_r^4 \tau_r^2}{\mu_r}\KL{q_0}{q_*} + 14L_r^2 d\tau_r\right).
    \end{equation*}
    Integrating from $t=0$ to $t=\tau_r$, we have
    \begin{equation*}
        \begin{aligned}
            e^{\frac{3\mu_r \tau_r}{2}}\KL{q_t}{q_*} - \KL{q_0}{q_*}\le & \frac{2}{3\mu_r}\cdot \left(e^{\frac{3\mu_r \tau_r}{2}}-1\right)\cdot \left(\frac{16L_r^4 \tau_r^2}{\mu_r}\KL{q_0}{q_*} + 14L_r^2 d\tau_r\right)\\
            \le & 2\tau_r\cdot \left(\frac{16L_r^4 \tau_r^2}{\mu_r}\KL{q_0}{q_*} + 14L_r^2 d\tau_r\right)
        \end{aligned}
    \end{equation*}
    where the last inequality establishes due to the fact $e^c \le 1+2c$ when $0<c\le \frac{3}{2}\cdot \mu_r \tau_r\le 1$.
    It means we have
    \begin{equation*}
        \begin{aligned}
            \KL{q_t}{q_*}\le e^{-\frac{3\mu_r \tau_r}{2}}\cdot \left(1+\frac{32 L_r^4 \tau_r^3}{\mu_r}\right)\KL{q_0}{q_*} + e^{-\frac{3\mu_r \tau_r}{2}}\cdot 28L_r^2 d\tau_r^2.
        \end{aligned}
    \end{equation*}
    By requiring $0<\tau_r\le \mu_r/(8L_r^2)$, we have
    \begin{equation*}
        1+\frac{32 L_r^4 \tau_r^3}{\mu_r} \le 1+\frac{\mu_r \tau_r}{2}\le e^{\frac{\mu_r\tau_r}{2}}\quad \text{and}\quad e^{-\frac{3\mu_r \tau_r}{2}}\le 1.
    \end{equation*}
    Hence, there is
    \begin{equation}
        \label{ineq:suff_des_ld_v2}
        \KL{q_t}{q_*}\le e^{-\mu_r\tau_r}\KL{q_0}{q_*}  + 28L_r^2 d\tau_r^2,
    \end{equation}
    and the proof is completed.
\end{proof}

\begin{lemma}
    \label{lem:inner_initialization}
    In Alg~\ref{con_alg:rse}, suppose the input is $(k,r,\vx,\epsilon)$ and $k>0$, if we choose the initial distribution of the inner loop to be 
    \begin{equation*}
        q^\prime_0(\vx^\prime)\propto \exp\left(-\frac{\left\|\vx-e^{-(S-r\eta)}\vx^\prime\right\|^2}{2(1-e^{-2(S-r\eta)})}\right),
    \end{equation*}
    then suppose $q_{k,S-r\eta}(\cdot|\vx)$ satisfies LSI with the constant $\mu_r$ and $L_r$ smoothness.
    Their KL divergence can be upper-bounded as
    \begin{equation*}
        \log \KL{q_0^\prime(\cdot)}{q_{k,S-r\eta}(\cdot|\vx)} \le  \log \|\vx\|^2 + \log \left[\frac{L_r^2 M}{\mu_r^2}\cdot \frac{de^{S}}{1-e^{-2S}}\right] + \frac{Me^{-S}}{1-e^{-2S}}.
    \end{equation*}
\end{lemma}
\begin{proof}
    According to Lemma~\ref{lem:score_reformul}, the density $q_{k,S-r\eta}(\cdot|\vx)$ can be presented as 
    \begin{equation*}
        q_{k,S-r\eta}(\vx^\prime|\vx) \propto \exp\left(-f_{k,0}(\vx^\prime)-\frac{\left\|\vx-e^{-(S-r\eta)}\vx^\prime\right\|^2}{2(1-e^{-2(S-r\eta)})}\right)
    \end{equation*}
    where $f_{k,0}(\vx^\prime)=\grad\log p_{k,0}(\vx^\prime)$.
    Since it satisfies LSI with the constant, .i.e, $\mu_r$, due to Definition~\ref{def:lsi}, we have
    \begin{equation}
        \label{ineq:init_kl_bound}
        \begin{aligned}
            &\KL{q_0^\prime(\cdot)}{q_{k,S-r\eta}(\cdot|\vx)}\le \frac{1}{2\mu_r}\cdot \int q_0^\prime(\vx^\prime) \left\|\grad f_{k,0}(\vx^\prime)\right\|^2\der\vx^\prime\\
            & \le \mu_r^{-1}\cdot \left(\int q_0^\prime(\vx^\prime) \left\|\grad f_{k,0}(\vx^\prime) - \grad f_{k,0}(\vzero)\right\|^2\der \vx^\prime + \int q_0^\prime(\vx^\prime) \left\|\grad f_{k,0}(\vzero)\right\|^2\der \vx^\prime\right).
        \end{aligned}
    \end{equation}
    For the first term, we have
    \begin{equation*}
        \begin{aligned}
            & \int q_0^\prime(\vx^\prime) \left\|\grad f_{k,0}(\vx^\prime) - \grad f_{k,0}(\vzero)\right\|^2\der \vx^\prime\\
            & \le L_r^2\cdot \int q^\prime_0(\vx^\prime) \left\|\vx^\prime\right\|^2 \der \vx^\prime = L_r^2\cdot \E_{q_0^\prime}\left[\|\rvx^\prime\|^2\right] = L_r^2\cdot \left[\Var(\rvx^\prime)+\left\|\E \rvx^\prime\right\|^2\right]
        \end{aligned}
    \end{equation*}
    where the first inequality follows from~\ref{con_ass:lips_score}.
    The high-dimensional Gaussian distribution, i.e., $q_0^\prime$ satisfies
    \begin{equation*}
        \left\|\E_{q_0^\prime}\left[\rvx^\prime\right]\right\|=e^{S-r\eta}\left\|\vx\right\|\quad \mathrm{and}\quad \Var(\rvx^\prime)\le d\cdot \left(e^{2(S-r\eta)}-1\right),
    \end{equation*}
    where the last inequality follows from Lemma~\ref{lem:lsi_var_bound}, hence we have
    \begin{equation}
        \label{ineq:kl_term1}
        \int q_0^\prime(\vx^\prime) \left\|\grad f_{k,0}(\vx^\prime) - \grad f_{k,0}(\vzero)\right\|^2\der \vx^\prime \le L_r^2 \cdot e^{2(S-r\eta)}(d+\|\vx\|^2).
    \end{equation}

    Then we consider to bound the second term of Eq~\ref{ineq:init_kl_bound}. 
    According to the definition of $\grad f_{k,0}$, with the transition kernel of the OU process, we have
    \begin{equation*}
        \begin{aligned}
            &-\grad f_{k,0}(\vx^\prime) = \grad \log p_{k,0}(\vx^\prime) = \frac{\grad p_{k,0}(\vx^\prime)}{p_{k,0}(\vx^\prime)}\\
            & = \frac{\int_{\R^d} p_{*}(\vx_0) \cdot \exp\left(\frac{-\left\|\vx - e^{-kS}\vx_0\right\|^2}{2\left(1-e^{-2kS}\right)}\right) \cdot \left(-\frac{\vx - e^{-kS }\vx_0}{\left(1-e^{-2(T-t)}\right)}\right)\der \vx_0}{\int_{\R^d} p_{*}(\vx_0)\cdot \exp\left(\frac{-\left\|\vx - e^{-kS}\vx_0\right\|^2}{2\left(1-e^{-2kS}\right)}\right)\der \vx_0}.
        \end{aligned}
    \end{equation*}
    Therefore, we have
    \begin{equation}
        \label{ineq:kl_term2}
        \begin{aligned}
            \left\|\grad f_{k,0}(\vzero)\right\|^2 = &\left\|\frac{\int_{\R^d} p_{*}(\vx_0) \cdot \exp\left(\frac{-e^{-2kS}\left\|\vx_0\right\|^2}{2\left(1-e^{-2kS}\right)}\right) \cdot \frac{e^{-kS }\vx_0}{\left(1-e^{-2kS}\right)} \der \vx_0}{\int_{\R^d} p_{*}(\vx_0)\cdot \exp\left(\frac{-e^{-2kS}\left\|\vx_0\right\|^2}{2\left(1-e^{-2kS}\right)}\right)\der \vx_0}\right\|^2\\
            \le & \frac{e^{-kS}}{1-e^{-2kS}}\cdot \int p_*(\vx_0)\cdot  \left\|\vx_0\right\|^2\der \vx_0 \cdot \frac{\int p_*(\vx_0)\cdot \exp\left(\frac{-e^{-2kS}\left\|\vx_0\right\|^2}{\left(1-e^{-2kS}\right)}\right)\der \vx}{\left(\int_{\R^d} p_{*}(\vx_0)\cdot \exp\left(\frac{-e^{-2kS}\left\|\vx_0\right\|^2}{2\left(1-e^{-2kS}\right)}\right)\der \vx_0\right)^2}\\
            \le & \frac{e^{-kS}}{1-e^{-2kS}}\cdot M \cdot \left(\int_{\R^d} p_{*}(\vx_0)\cdot \exp\left(\frac{-e^{-2kS}\left\|\vx_0\right\|^2}{2\left(1-e^{-2kS}\right)}\right)\der \vx_0\right)^{-1}
        \end{aligned}
    \end{equation}
    where the first inequality follows from Holders' inequality, the second inequality follows from~\ref{con_ass:second_moment}.
    With, the following range:
    \begin{equation*}
        \frac{e^{-2kS}}{1-e^{-2kS}}\le \frac{e^{-kS}}{1-e^{-2kS}}\le \frac{e^{-S}}{1-e^{-2S}}
    \end{equation*}    
    we plug Eq~\ref{ineq:kl_term1} and Eq~\ref{ineq:kl_term2} into Eq~\ref{ineq:init_kl_bound} and obtain
    \begin{equation*}
        \small
        \begin{aligned}
            \log \KL{q_0^\prime(\cdot)}{q_{k,S-r\eta}(\cdot|\vx)} \le &\log \left[\mu^{-1}_r\cdot \left(L_r^2 \cdot e^{2(S-r\eta)}(d+\|\vx\|^2)  \right.\right.\\
            &\left.\left.+\frac{e^{-kS}}{1-e^{-2kS}}\cdot M \left(\int_{\R^d} p_{*}(\vx_0)\cdot \exp\left(\frac{-e^{-2kS}\left\|\vx_0\right\|^2}{2\left(1-e^{-2kS}\right)}\right)\der \vx_0\right)^{-1}\right)\right]
        \end{aligned}
    \end{equation*}
    Without loss of generality, we suppose both RHS of Eq~\ref{ineq:kl_term1} and Eq~\ref{ineq:kl_term2} are larger than $1$. 
    Then, we have
    \begin{equation*}
        \small
        \begin{aligned}
            &\log \KL{q_0^\prime(\cdot)}{q_{k,S-r\eta}(\cdot|\vx)} \\
            &\le \log \left[\frac{L_r^2}{\mu_r^2}\cdot e^{2(S-r\eta)}M \cdot \frac{e^{-kS}}{1-e^{-2kS}}\cdot (d+\|\vx\|^2)\right] -\log\left[\int_{\R^d} p_{*}(\vx_0)\cdot \exp\left(\frac{-e^{-2kS}\left\|\vx_0\right\|^2}{2\left(1-e^{-2kS}\right)}\right)\der \vx_0\right]\\
            & \le \log\left[\frac{L_r^2 M}{\mu_r^2}\cdot \frac{e^{S}}{1-e^{-2S}}\cdot (d+\|\vx\|^2)\right]+ \frac{e^{-2kS}}{2(1-e^{-2kS})}\cdot \int_{\R^d} p_*(\vx_0)\left\|\vx_0\right\|^2\der\vx_0\\
            & \le \log\left[\frac{L_r^2 M}{\mu_r^2}\cdot \frac{e^{S}}{1-e^{-2S}}\cdot (d+\|\vx\|^2)\right] + \frac{M e^{-S}}{1-e^{-2S}}\le  \log \|\vx\|^2 + \log \left[\frac{L_r^2 M}{\mu_r^2}\cdot \frac{de^{S}}{1-e^{-2S}}\right] + \frac{Me^{-S}}{1-e^{-2S}}.
        \end{aligned}
    \end{equation*}
    Hence, the proof is completed.
\end{proof}

\begin{corollary}
    \label{cor:innerloop_con}
    For any given $(k,r)$ in Alg~\ref{con_alg:rse} and $\vx\in \R^d$, suppose the distribution $q_{k,S-r\eta}(\cdot|\vx)$ satisfies
    \begin{equation*}
        \mu_r\mI \preceq -\grad^2\log q_{k,S-r\eta}(\cdot|\vx) \preceq L_r\mI,
    \end{equation*}
    and $\rvx_{j}^\prime \sim q_j^\prime(\cdot |\vx)$.
    If $0<\tau_r \le \mu_r/(8L^2_r)$, we have
    \begin{equation*}
        \KL{q^\prime_j}{q_*} \le \exp\left(-\mu_r\tau_r j\right)\cdot \KL{q^\prime_0}{q_*} + \frac{32L_r^2 d\tau_r}{\mu_r}
    \end{equation*}
    when the score estimation satisfies $\left\|\grad\log p_{k,0} -\vv^\prime\right\|_\infty \le L_r\sqrt{2d\tau_r}$.
\end{corollary}
\begin{proof}
    Due to the range $0<\tau_r\le \mu_r/8L_r^2$, we have $\mu_r \tau_r\le 1/8$.
    In this condition, we have 
    \begin{equation*}
        1-\exp\left(-\mu_r\tau_r\right)\ge \frac{7}{8}\cdot \mu_r\tau_r.
    \end{equation*}
    Plugging this into the following inequality obtained by the recursion of Eq.~\ref{ineq:suff_des_ld_v2}, we have
    \begin{equation*}
        \begin{aligned}
            \KL{q^\prime_j}{q_*}\le &\exp\left(-\mu_r\tau_r j\right)\cdot \KL{q^\prime_0}{q_*} + \frac{28L_r^2 d\tau_r^2}{\left(1-\exp\left(-\mu_r\tau_r\right)\right)}\\
            \le & \exp\left(-\mu_r\tau_r j\right)\cdot \KL{q^\prime_0}{q_*} + \frac{32L_r^2 d\tau_r}{\mu_r}.
        \end{aligned}
    \end{equation*}
    In this condition, if we require the KL divergence to satisfy $\KL{q^\prime_j}{q_*}\le \epsilon$, a sufficient condition is that 
    \begin{equation*}
        \exp\left(-\mu_r\tau_r j\right)\cdot \KL{q^\prime_0}{q_*}\le \frac{\epsilon}{2} \quad\text{and}\quad  \frac{32L_r^2 d\tau}{\mu_r}\le \frac{\epsilon}{2},
    \end{equation*}
    which is equivalent to 
    \begin{equation*}
        \tau_r\le \frac{\mu_r\epsilon}{64 L_r^2 d} \quad \text{and}\quad j\ge \frac{1}{\mu_r \tau_r}\cdot \log\frac{2\KL{q^\prime_0}{q_*}}{\epsilon}.
    \end{equation*}
    According to the upper bound of $\KL{q^\prime_0}{q_*}$ shown in Lemma~\ref{lem:inner_initialization}, we require 
    \begin{equation*}
        j\ge \frac{1}{\mu_r \tau_r}\cdot \left[ \log \frac{\|\vx\|^2}{\epsilon} + \log \left(\frac{2L_r^2 M}{\mu_r^2}\cdot \frac{de^{S}}{1-e^{-2S}}\right) + \frac{Me^{-S}}{1-e^{-2S}}\right].
    \end{equation*}
\end{proof}

\subsection{Core Lemmas}

\label{app_sec:core_lemmas}

\begin{lemma}
    \label{lem:recursive_core_lem}
    In Alg~\ref{con_alg:rse}, for any $k\in\mathbb{N}_{0,K-1}$, $r\in\mathbb{N}_{0,R-1}$ and $\vx\in \R^d$, we have
    \begin{equation*}
        \sP\left[\left\|\rbkwv_{k,r\eta}(\vx) - \grad\log p_{k,S-r\eta}(\vx)\right\|^2\le 10\epsilon\right] \ge 1-\delta
    \end{equation*}
    by requiring the segment length $S$, the sample number $n_{k,r}$ and the step size of inner loops $\tau_r$ and the iteration number of inner loops $m_{k,r}$ satisfy
    \begin{equation*}
        \begin{aligned}
            &S=\frac{1}{2}\log\frac{2L+1}{2L},\quad  n_{k,r}\ge \frac{4}{\epsilon (1-e^{-2(S-r\eta)})} \cdot \max\left\{d, -2\log \delta\right\},\\
            &\tau_r\le \frac{\mu_r}{64L_r^2 d}\cdot (1-e^{-2(S-r\eta)})\epsilon\quad \text{and}\quad m_{k,r}\ge \frac{64L_r^2 d }{\mu^2_r (1-e^{-2(S-r\eta)}) \epsilon}\cdot \left[ \log \frac{d\|\vx\|^2}{(1-e^{-2(S-r\eta)})\epsilon} +  C_{m,1}\right],
        \end{aligned}
    \end{equation*}
    where $C_{m,1}= \log \left(2M\cdot 3^2 \cdot 5L\right)+M\cdot 3L$.
    In this condition that choosing the $\tau_r$ to its upper bound, we required the score estimation in the inner loop satisfies
    \begin{equation*}
        \left\|\grad\log p_{k,0}(\vx^\prime) -\rvv^\prime_{k,0}(\vx^\prime)\right\|\le \frac{e^{-(S-r\eta)}\epsilon^{0.5}}{8}.
    \end{equation*}
\end{lemma}
\begin{proof}
    With a little abuse of notation, for each loop $i\in \mathbb{N}_{1,n_{k,r}}$ in Line~\ref{con_alg:rse_outer_loops} of  Alg~\ref{con_alg:rse}, we denote the underlying distribution of \textbf{output particles} as $\rvx^\prime_{i}\sim q^\prime_{k,S-r\eta}(\cdot |\vx)$ for any $k\in\mathbb{N}_{0,K-1}$, $r\in\mathbb{N}_{0,R-1}$ and $\vx\in \R^d$ in this lemma.
    According to Line~\ref{con_alg:rse_score_update} in Alg~\ref{con_alg:rse}, we have
    \begin{equation}
        \label{ineq:term2_decomp}
        \begin{aligned}
            & \left\|\rbkwv_{k,r\eta}(\vx) - \grad\log p_{k,S-r\eta}(\vx)\right\|^2\\
            = & \left\|\frac{1}{n_{r,k}}\sum_{i=1}^{n_{r,k}}\left(-\frac{\vx - e^{-(S-r\eta)}\rvx_i^\prime}{1-e^{-2(S-r\eta)}}\right)-\E_{\rvx^\prime\sim q_{k,S-r\eta}(\cdot|\vx)}\left[-\frac{\vx - e^{-(S-r\eta)}\rvx^\prime }{1-e^{-2(S-r\eta)}}\right] \right\|^2\\
            = & \frac{e^{-2(S-r\eta)}}{\left(1-e^{-2(S-r\eta)}\right)^2}\cdot \left\|-\frac{1}{n_{r,k}}\sum_{i=1}^{n_{r,k}} \rvx_i^\prime - \E_{\rvx^\prime \sim q_{k,S-r\eta}(\cdot|\vx)}\left[\rvx^\prime\right]\right\|^2\\
            \le & \frac{2e^{-2(S-r\eta)}}{\left(1-e^{-2(S-r\eta)}\right)^2}\cdot \left\|-\frac{1}{n_{r,k}}\sum_{i=1}^{n_{r,k}} \rvx_i^\prime + \E_{\rvx^\prime \sim q^\prime_{k,S-r\eta}(\cdot|\vx)}\left[\rvx^\prime\right]\right\|^2\\
            & + \frac{2e^{-2(S-r\eta)}}{\left(1-e^{-2(S-r\eta)}\right)^2}\cdot \left\|-\E_{\rvx^\prime \sim q^\prime_{k,S-r\eta}(\cdot|\vx)}\left[\rvx^\prime\right] + \E_{\rvx^\prime \sim q_{k,S-r\eta}(\cdot|\vx)}\left[\rvx^\prime\right]\right\|^2
        \end{aligned}
    \end{equation} 
    In the following, we respectively upper bound the concentration error and the mean gap between $q^\prime_{k,S-r\eta}(\cdot|\vx)$ and $q_{k,S-r\eta}(\cdot|\vx)$ corresponding to the former and the latter term in Eq~\ref{ineq:term2_decomp}.

    \paragraph{Upper bound the concentration error.}
    The choice of $S$, i.e., $S = \frac{1}{2}\log\left(\frac{2L+1}{2L}\right)$,
    Lemma~\ref{lem:sm_con_of_q} demonstrate that suppose 
    \begin{equation*}
        \mu_r= \frac{1}{2}\cdot \frac{e^{-2(S-r\eta)}}{1-e^{-2(S-r\eta)}} \quad \mathrm{and}\quad L_r= \frac{3}{2}\cdot \frac{e^{-2(S-r\eta)}}{1-e^{-2(S-r\eta)}}.
    \end{equation*}
    Then, we have
    \begin{equation*}
        \mu_r\mI \preceq -\grad^2\log q_{k,S-r\eta}(\vx^\prime|\vx) \preceq L_r\mI.
    \end{equation*}
    According to Alg~\ref{con_alg:rse}, we utilize ULA as the inner loop (Line~\ref{con_alg:rse_outer_loops} -- Line~\ref{con_alg:rse_inner_update}) to sample from $q_{k,S-r\eta}(\cdot|\vx)$.
    By requiring the step size, i.e., $\tau_r$ to satisfy $\tau_r\le 1/L_r$, with Lemma~\ref{lem:thm8_vempala2019rapid}, we know that the underlying distribution of output particles of the inner loops satisfies, i.e., $q^\prime_{k,S-r\eta}(\cdot|\vx)$ satisfies LSI with a constant $\mu^\prime_r$ satisfying
    \begin{equation*}
        \mu^\prime_r \ge \frac{\mu_r}{2} \ge  \frac{e^{-2(S-r\eta)}}{4(1-e^{-2(S-r\eta)})}.
    \end{equation*}
    In this condition, we employ Lemma~\ref{lem:concen_prop_v2}, by requiring
    \begin{equation*}
        \begin{aligned}
            n_{k,r}\ge &\frac{4}{\epsilon (1-e^{-2(S-r\eta)})} \cdot \max\left\{d, -2\log \delta\right\}\\
            \ge  &\frac{1}{\mu^\prime_r }\cdot \left(\frac{e^{-(S-r\eta)}}{(1-e^{-2(S-r\eta)})\epsilon^{0.5}}\right)^2\cdot \max\left\{d, -2\log \delta\right\}.
        \end{aligned}
    \end{equation*}
    and obtain
    \begin{equation*}
        \begin{aligned}
            &\sP\left[\frac{2e^{-2(S-r\eta)}}{\left(1-e^{-2(S-r\eta)}\right)^2}\cdot\left\|-\frac{1}{n_{r,k}}\sum_{i=1}^{n_{r,k}} \rvx_i^\prime + \E_{\rvx^\prime \sim q^\prime_{k,S-r\eta}(\cdot|\vx)}\left[\rvx^\prime\right]\right\|^2\le 2\epsilon\right]\\
            = & \sP\left[\left\|-\frac{1}{n_{r,k}}\sum_{i=1}^{n_{r,k}} \rvx_i^\prime + \E_{\rvx^\prime \sim q^\prime_{k,S-r\eta}(\cdot|\vx)}\left[\rvx^\prime\right]\right\|\le \frac{(1-e^{-2(S-r\eta)})\epsilon^{0.5}}{e^{-(S-r\eta)}}\right]\ge 1-\delta.
        \end{aligned}
    \end{equation*}
    \paragraph{Upper bound the mean gap.}
    According to Lemma~\ref{lem:sm_con_of_q} and Lemma~\ref{lem:strongly_lsi}, we know $q_{k,S-r\eta}(\vx^\prime|\vx)$ satisfies LSI with constant
    \begin{equation*}
        \mu_r\ge  \frac{e^{-2(S-r\eta)}}{2(1-e^{-2(S-r\eta)})}.
    \end{equation*}
    By introducing the optimal coupling between $q_{k,S-r\eta}(\cdot|\vx)$ and $q^\prime_{k,S-r\eta}(\cdot|\vx)$, we have
    \begin{equation}
        \label{ineq:term2.2_upb}
        \begin{aligned}
            &\left\|-\E_{\rvx^\prime \sim q^\prime_{k,S-r\eta}(\cdot|\vx)}\left[\rvx^\prime\right] + \E_{\rvx^\prime \sim q_{k,S-r\eta}(\cdot|\vx)}\left[\rvx^\prime\right]\right\|^2\\
            \le & W_2^2\left(q^\prime_{k,S-r\eta}(\cdot|\vx), q_{k,S-r\eta}(\cdot|\vx)\right)\le \frac{2}{\mu_r}\KL{q^\prime_{k,S-r\eta}(\cdot|\vx)}{q_{k,S-r\eta}(\cdot|\vx)},
        \end{aligned}
    \end{equation}
    where the last inequality follows from Talagrand inequality~\cite{vempala2019rapid}.
    Hence, the mean gap can be upper-bounded as
    \begin{equation*}
        \begin{aligned}
            &\frac{2e^{-2(S-r\eta)}}{\left(1-e^{-2(S-r\eta)}\right)^2}\cdot \left\|-\E_{\rvx^\prime \sim q^\prime_{k,S-r\eta}(\cdot|\vx)}\left[\rvx^\prime\right] + \E_{\rvx^\prime \sim q_{k,S-r\eta}(\cdot|\vx)}\left[\rvx^\prime\right]\right\|^2\\
            \le &\frac{2e^{-2(S-r\eta)}}{\left(1-e^{-2(S-r\eta)}\right)^2} \cdot \frac{2}{\mu_r}\KL{q^\prime_{k,S-r\eta}(\cdot|\vx)}{q_{k,S-r\eta}(\cdot|\vx)}\\
            \le  &\frac{8}{(1-e^{-2(S-r\eta)})}\KL{q^\prime_{k,S-r\eta}(\cdot|\vx)}{q_{k,S-r\eta}(\cdot|\vx)}.
        \end{aligned}
    \end{equation*}
    To provide $\epsilon$-level upper bound, we expect the required accuracy of KL convergence of inner loops to satisfy
    \begin{equation*}
        \KL{q^\prime_{k,S-r\eta}(\cdot|\vx)}{q_{k,S-r\eta}(\cdot|\vx)}\le (1-e^{-2(S-r\eta)})\epsilon.
    \end{equation*}
    According to Corollary~\ref{cor:innerloop_con}, to achieve such accuracy, we require the step size and the iteration number of inner loops to satisfy
    \begin{equation*}
        \begin{aligned}
            &\tau_r\le \frac{\mu_r}{64L_r^2 d}\cdot (1-e^{-2(S-r\eta)})\epsilon \quad \mathrm{and} \\
            &m_{k,r}\ge \frac{1}{\mu_r}\cdot \frac{64L_r^2 d }{\mu_r (1-e^{-2(S-r\eta)}) \epsilon}\cdot \left[ \log \frac{\|\vx\|^2}{(1-e^{-2(S-r\eta)})\epsilon} + \log \left(\frac{2L_r^2 M}{\mu_r^2}\cdot \frac{de^{S}}{1-e^{-2S}}\right) + \frac{Me^{-S}}{1-e^{-2S}}\right].
        \end{aligned}
    \end{equation*}
    To simplify notation, we suppose $L\ge 1$ without loss of generality, and we the following equations:
    \begin{equation*}
        \begin{aligned}
            &\frac{L_r}{\mu_r}=3,\quad  e^S=\exp\left(\frac{1}{2}\log \frac{2L+1}{2L}\right) = \sqrt{\frac{2L+1}{2L}},\\
            & \left(1-e^{-2S}\right)^{-1}= (2L+1),
        \end{aligned}
    \end{equation*}
    which implies 
    \begin{equation*}
        \begin{aligned}
            &\log \frac{d\|\vx\|^2}{(1-e^{-2(S-r\eta)})\epsilon} + \log \left(2M\cdot 3^2\cdot 5L\right)+M\cdot 3L\\
            &\ge  \log \frac{d\|\vx\|^2}{(1-e^{-2(S-r\eta)})\epsilon} + \log \left(2M\cdot \frac{L_r^2}{\mu_r^2}\cdot \sqrt{\frac{2L+1}{2L}}\cdot (2L+1)\right) + M \cdot (2L+1)\cdot \sqrt{\frac{2L}{2L+1}}\\
            & = \log \frac{d\|\vx\|^2}{(1-e^{-2(S-r\eta)})\epsilon} + \log \left(\frac{2L_r^2 M}{\mu_r^2}\cdot \frac{e^{S}}{1-e^{-2S}}\right) + \frac{Me^{-S}}{1-e^{-2S}}.
        \end{aligned}
    \end{equation*}
    Therefore, we only require $m_{k,r}$ satisfies
    \begin{equation*}
        m_{k,r}\ge \frac{1}{\mu_r}\cdot \frac{64L_r^2 d }{\mu_r (1-e^{-2(S-r\eta)}) \epsilon}\cdot \left[ \log \frac{d\|\vx\|^2}{(1-e^{-2(S-r\eta)})\epsilon} +  C_{m,1}\right]
    \end{equation*}
    where $C_{m,1}= \log \left(2M\cdot 3^2 \cdot 5L\right)+M\cdot 3L$.
    For simplicity, we choose $\tau_r$ as its upper bound and lower bound, respectively.
    In this condition, we still require 
    \begin{equation*}
        \begin{aligned}
            \left\|\grad\log p_{k,0} -\vv^\prime\right\|\le \frac{e^{-(S-r\eta)}\epsilon^{0.5}}{8} \le \frac{1}{4}\cdot \sqrt{\frac{\mu_r(1-e^{-2(S-r\eta)})}{2}\cdot \epsilon}  \le  L_r\sqrt{2d\tau_r}
        \end{aligned}
    \end{equation*}
    where the first inequality follows from the range of $\mu_r$, and the last inequality is satisfied when we choose $\tau_r$ to its upper bound.
    Hence, the proof is completed.
\end{proof}


\begin{lemma}[Errors from fine-grained score estimation]
    \label{lem:rec_err_fine_grained}
    Under the notation in Section~\ref{sec:not_ass_app}, suppose the step size satisfy $\eta=C_\eta(d+M)^{-1}\epsilon$,
    we have
    \begin{equation*}
        \begin{aligned}            & \sP\left[\left\|\grad\log p_{k,S-r\eta}(\vx)-\rbkwv_{k,r\eta}(\vx)\right\|^2\le 10\epsilon,\forall \vx\in\R^d\right]\\
        \ge & (1-\delta)\cdot \left(\min_{\vx^\prime\in\sS_{k,r}(\vx,\epsilon)} \mathbb{P}\left[\left\|\grad\log p_{k,0}(\vx^\prime) -\rbkwv_{k-1,0}(\vx^\prime)\right\|^2\le \frac{\epsilon}{96}\right]\right)^{n_{k,r}(10\epsilon)\cdot m_{k,r}(10\epsilon,\vx)},
        \end{aligned}
    \end{equation*}
    where $\sS_{k,r}(\vx,10\epsilon)$ denotes the set of particles appear in Alg~\ref{con_alg:rse} when the input is $(k,r,\vx,10\epsilon)$. 
    For any $(k,r)\in\mathbb{N}_{0,K-1}\times \mathbb{N}_{0,R-1}$ by requiring 
    \begin{equation*}
        \begin{aligned}
            &n_{k,r}(10\epsilon) = C_n\cdot \frac{(d+M)\cdot \max\{d,-2\log \delta\}}{(10\epsilon)^2}\quad \mathrm{where}\quad C_n = 2^6\cdot 5^2\cdot C_\eta^{-1},\\
            &m_{k,r}(10\epsilon,\vx) = C_m\cdot \frac{(d+M)^3\cdot \max\{\log \|\vx\|^2,1\}}{(10\epsilon)^3}\quad \mathrm{where}\quad C_m = 2^9\cdot 3^2\cdot 5^3\cdot C_{m,1}C_{\eta}^{-1.5}.
        \end{aligned}
    \end{equation*}
\end{lemma}
\begin{proof}
    According to Line~\ref{con_alg:rse_inner_update} of Alg~\ref{con_alg:rse}, for any $\vx\in\R^d$, the score estimation $\rbkwv_{k,r\eta}$ is constructed by estimating the mean in RHS of the following expectation using $n_{k,r}$ samples (i.e., calculating the empirical mean):
    \begin{align}            \grad_{\vx} \log p_{k,S-r\eta}(\vx) =&  \mathbb{E}_{\rvx^\prime \sim q_{k,S-r\eta}(\cdot|\vx)}\left[-\frac{\vx - e^{-(S-r\eta)}\rvx^\prime}{\left(1-e^{-2(S-r\eta)}\right)}\right]\label{eq:estimation_score1_proof}\\
            \mathrm{where}\quad  q_{k,S-r\eta}(\vx^\prime|\vx) \propto & \exp\left(\log\ p_{k,0}(\vx^\prime)-\frac{\left\|\vx - e^{-(S-r\eta)}\vx^\prime\right\|^2}{2\left(1-e^{-2(S-r\eta)}\right)}\right).\label{eq:estimation_score2_proof}
    \end{align}
Then in order to guarantee an accurate estimation for $\nabla_\vx\log p_{k,S-r\eta}(\vx)$, i.e., denoted by $\rbkwv_{k,r\eta}(\vx)$, with Lemma~\ref{lem:recursive_core_lem}, we require 
\begin{enumerate}
\item Get a precise estimation for $\nabla \log p_{k,0}(\vx')$, in order to 
guarantee that the estimation for $\nabla \log q_{k,S-r\eta}(\vx'|\vx)$ is accurate. 
In particular, we require
\begin{equation*}
    \left\|\grad\log p_{k,0}(\vx^\prime_{i,j}) -\rbkwv_{k-1,0}(\vx^\prime_{i,j})\right\|\le \frac{e^{-(S-r\eta)}\epsilon^{0.5}}{8}.
\end{equation*}
\item Based on the $\nabla \log q_{k,S-r\eta}(\vx'|\vx)$, we run ULA with appropriate step size $\tau_r$ and iteration number $m_{k,r}$ satisfying
\begin{equation}
    \small
    \label{ineq:fine_grained_req1}
    \tau_r\le \frac{\mu_r}{64L_r^2 d}\cdot (1-e^{-2(S-r\eta)})\epsilon\quad \text{and}\quad m_{k,r}\ge \frac{1}{\mu_r}\cdot \frac{64L_r^2 d }{\mu_r (1-e^{-2(S-r\eta)}) \epsilon}\cdot \log \frac{2C_0}{(1-e^{-2(S-r\eta)})\epsilon}
\end{equation}
to generate samples $\vx'$ whose underlying distribution $q^\prime_{k,S-r\eta}(\cdot|\vx)$ is sufficiently close to $q_{k,S-r\eta}(\vx'|\vx)$, i.e., 
\begin{equation*}
        \KL{q^\prime_{k,S-r\eta}(\cdot|\vx)}{q_{k,S-r\eta}(\cdot|\vx)}\le (1-e^{-2(S-r\eta)})\epsilon.
\end{equation*}
\item Generate a sufficient number of samples satisfying
\begin{equation}
    \label{ineq:fine_grained_req2}
    n_{k,r}\ge \frac{4}{\epsilon (1-e^{-2(S-r\eta)})} \cdot \max\left\{d, -2\log \delta\right\}.
\end{equation}
such that the empirical estimation of the expectation in \eqref{eq:estimation_score1_proof} is accurate, i.e.,
\begin{equation*}
        \begin{aligned}
            &\mathbb{P}\left[\left\|\grad\log p_{k,S-r\eta}(\vx)-\tilde{\rvv}_{k,r\eta}(\vx) \right\|^2\le 10\epsilon\right]\\
            & = \mathbb{P}\left[\left\|\grad\log p_{k,S-r\eta}(\vx) - \frac{1}{n_{k,r}}\sum_{i=1}^{n_{k,r}}\left[-\frac{\vx - e^{-(S-r\eta)}\vx_{i,m_{k,r}}^\prime}{\left(1-e^{-2(S-r\eta)}\right)}\right]\right\|\le 10\epsilon\right]\ge 1-\delta.
        \end{aligned}
\end{equation*}
\end{enumerate}

    Due to the fact $r\eta\ge 0$, the first condition can be achieved by requiring
    \begin{equation*}
        \left\|\grad\log p_{k,0}(\vx^\prime_{i,j}) -\rbkwv_{k-1,0}(\vx^\prime_{i,j})\right\|\le \sqrt{\frac{2}{3}}\cdot \frac{\epsilon^{0.5}}{8} \le \sqrt{\frac{2L}{2L+1}}\cdot \frac{\epsilon^{0.5}}{8} = \frac{e^{-S}\epsilon^{0.5}}{8}\le \frac{e^{-(S-r\eta)}\epsilon^{0.5}}{8},
    \end{equation*}
    where the second inequality is established by supposing $L\ge 1$ without loss of generality, and the last equation follows from the choice of $S$.

    To investigate the setting of hyper-parameters, i.e., the number of samples for empirical mean estimation $n_{k,r}$ and the number of iterations for ULA $m_{k,r}$.
    We first reformulate them as two functions, i.e.,
    \begin{equation*}
        \begin{aligned}
            &n_{k,r}(10\epsilon) = C_n\cdot \frac{(d+M)\cdot \max\{d,-2\log \delta\}}{(10\epsilon)^2}\quad \mathrm{where}\quad C_n = 2^6\cdot 5^2\cdot C_\eta^{-1},\\
            &m_{k,r}(10\epsilon,\vx) = C_m\cdot \frac{(d+M)^3\cdot \max\{\log \|\vx\|^2,1\}}{(10\epsilon)^3}\quad \mathrm{where}\quad C_m = 2^9\cdot 3^2\cdot 5^3\cdot C_{m,1}C_{\eta}^{-1.5}.
        \end{aligned}
    \end{equation*}
    since this presentation helps to explain the connection between them and the input of Alg~\ref{con_alg:rse}.
    Different from the results shown in Lemma~\ref{lem:recursive_core_lem}, $n_{k,r}(\cdot)$ and $m_{k,r}(\cdot,\cdot)$ is independent with $k$ and $r$.
    However, these choices will still make 
    Eq~\ref{ineq:fine_grained_req1} and Eq~\ref{ineq:fine_grained_req2} establish, because
    \begin{equation*}
        \begin{aligned}
            n_{k,r}(10\epsilon)  = & \frac{16}{\epsilon}\cdot \frac{(d+M)}{C_\eta\epsilon}\cdot \max\left\{d,-2\log\delta\right\} \ge \frac{16}{\epsilon \eta}\cdot \max\left\{d,-2\log\delta\right\}\\
            \ge &  \frac{16}{\epsilon (1-e^{-2\eta})}\cdot \max\left\{d,-2\log\delta\right\} \ge \frac{16}{\epsilon (1-e^{-2(S-r\eta)})} \cdot \max\left\{d, -2\log \delta\right\}\\
            m_{k,r}(10\epsilon,\vx) = & 576\cdot \frac{(d+M)^3}{\epsilon^3}\cdot \frac{C_{m,1}}{C_\eta^{1.5}}\cdot \max\{\log\|\vx\|^2,1\} \ge 64\cdot \frac{L_r^2}{\mu_r^2}\cdot \left(\frac{d}{\epsilon\eta}\right)^{1.5}\cdot C_{m,1}\cdot \max\{\log\|\vx\|^2,1\}\\
            \ge & 64\cdot \frac{L_r^2}{\mu_r^2}\cdot \frac{d}{\epsilon\eta}\log \frac{d}{\epsilon\eta} \cdot C_{m,1}\cdot \max\{\log\|\vx\|^2,1\}\ge 64\cdot \frac{L_r^2}{\mu_r^2}\cdot \frac{d}{\epsilon\eta }\left(\log \frac{d\|\vx\|^2}{\epsilon\eta}+C_{m,1}\right)\\
            \ge & 64\cdot \frac{L_r^2}{\mu_r^2}\cdot \frac{d}{\epsilon(1-e^{-2\eta}) }\left(\log \frac{d\|\vx\|^2}{\epsilon(1-e^{-2\eta}) }+C_{m,1}\right)\\
            \ge & \frac{64L_r^2 d }{\mu^2_r (1-e^{-2(S-r\eta)}) \epsilon}\cdot \left( \log \frac{d\|\vx\|^2}{(1-e^{-2(S-r\eta)})\epsilon} +  C_{m,1}\right)
        \end{aligned}
    \end{equation*}
    with the proper choice of step size, i.e., $\eta=C_\eta(d+M)^{-1}\epsilon$.
    With these settings, Lemma~\ref{lem:recursive_core_lem} demonstrates that
    \begin{equation*}
        \small
        \begin{aligned}
            &\sP\left[\left\|\grad\log p_{k,S-r\eta}(\vx)-\rbkwv_{k,r\eta}(\vx) \right\|^2\le 10\epsilon, \forall \vx\in\R^d \Big| \right.\\
            &\left.\quad \bigcap_{
            \vx^\prime\in \sS_{k,r}(\vx,10\epsilon)
        } \left\|\grad\log p_{k,0}(\vx^\prime) -\rbkwv_{k-1,0}(\vx^\prime)\right\|^2\le \frac{\epsilon}{96}\right] \ge 1-\delta.
        \end{aligned}
    \end{equation*}
    where $\sS_{k,r}(\vx,10\epsilon)$ denotes the set of particles appear in Alg~\ref{con_alg:rse} when the input is $(k,r,\vx,10\epsilon)$ except for the recursion.
    It satisfies $|\sS_{k,r}(\vx,10\epsilon)| = n_{k,r}(10\epsilon)\cdot m_{k,r}(10\epsilon,\vx)$. 
    Furthermore, we have
    \begin{equation}
        \label{ineq: recursive_prob_lob_fine_grained_condi1}
        \small
        \begin{aligned}
        & \sP\left[\left\|\grad\log p_{k,S-r\eta}(\vx)-\rbkwv_{k,r\eta}(\vx)\right\|^2\le 10\epsilon\right]\\
        & \ge \sP\left[\left\|\grad\log p_{k,S-r\eta}(\vx)-\rbkwv_{k,r\eta}(\vx)\right\|^2\le 10\epsilon \Big| \bigcap_{
            \vx^\prime\in \sS_{k,r}(\vx,10\epsilon)
        } \left\|\grad\log p_{k,0}(\vx^\prime) -\rbkwv_{k-1,0}(\vx^\prime)\right\|^2\le \frac{\epsilon}{96}\right]\\
        &\quad \cdot \mathbb{P}\left[\bigcap_{
            \vx^\prime\in \sS_{k,r}(\vx,10\epsilon)
        } \left\|\grad\log p_{k,0}(\vx^\prime) -\rbkwv_{k-1,0}(\vx^\prime)\right\|^2\le \frac{\epsilon}{96}\right]\\
        & \ge (1-\delta)\cdot \mathbb{P}\left[\bigcap_{
            \vx^\prime\in \sS_{k,r}(\vx,10\epsilon)
        } \left\|\grad\log p_{k,0}(\vx^\prime) -\rbkwv_{k-1,0}(\vx^\prime)\right\|^2\le \frac{\epsilon}{96}\right].
        \end{aligned}
    \end{equation}
    Considering that for each $\vx_{i,j}^\prime$, the score estimation, i.e., $\rbkwv_{k-1,0}(\vx_{i,j}^\prime)$ is independent,
    hence, we have
    \begin{equation}
        \label{ineq: recursive_prob_lob_fine_grained_condi2}
        \begin{aligned}
        &\mathbb{P}\left[\bigcap_{
            \vx^\prime\in \sS_{k,r}(\vx,10\epsilon)
        } \left\|\grad\log p_{k,0}(\vx^\prime) -\rbkwv_{k-1,0}(\vx^\prime)\right\|^2\le \frac{\epsilon}{96}\right]\\
        &= \prod_{\vx^\prime\in\sS_{k,r}(\vx,10\epsilon)}\mathbb{P}\left[\left\|\grad\log p_{k,0}(\vx^\prime) -\rbkwv_{k-1,0}(\vx^\prime)\right\|^2\le \frac{\epsilon}{96}\right]\\
        & \ge \left(\min_{\vx^\prime\in\sS_{k,r}(\vx,\epsilon)} \mathbb{P}\left[\left\|\grad\log p_{k,0}(\vx^\prime) -\rbkwv_{k-1,0}(\vx^\prime)\right\|^2\le \frac{\epsilon}{96}\right]\right)^{|\sS_{k,r}(\vx,\epsilon)|}
        \end{aligned}
    \end{equation}
    Therefore, combining Eq~\ref{ineq: recursive_prob_lob_fine_grained_condi1} and Eq~\ref{ineq: recursive_prob_lob_fine_grained_condi2}, we have
    \begin{equation*}
        \begin{aligned}            & \sP\left[\left\|\grad\log p_{k,S-r\eta}(\vx)-\rbkwv_{k,r\eta}(\vx)\right\|^2\le 10\epsilon \right]\\
        \ge & (1-\delta)\cdot \left(\min_{\vx^\prime\in\sS_{k,r}(\vx,\epsilon)} \mathbb{P}\left[\left\|\grad\log p_{k,0}(\vx^\prime) -\rbkwv_{k-1,0}(\vx^\prime)\right\|^2\le \frac{\epsilon}{96}\right]\right)^{n_{k,r}(10\epsilon)\cdot m_{k,r}(10\epsilon,\vx)},
        \end{aligned}
    \end{equation*}
    and the proof is completed.
\end{proof}

\begin{corollary}[Errors from coarse-grained score estimation]
    \label{lem:rec_err_coarse_grained}
    Under the notation in Section~\ref{sec:not_ass_app}, suppose the step size satisfy $\eta=C_1(d+M)^{-1}\epsilon$,
    we have
    \begin{equation}
        \label{ineq: recursive_prob_lob_coarse_grained}
        \begin{aligned}            & \sP\left[\left\|\grad\log p_{k+1,0}(\vx)- \rbkwv_{k,0}(\vx)\right\|^2\le 10\epsilon,\forall \vx\in\R^d\right]\\
        \ge & (1-\delta)\cdot \left(\min_{\vx^\prime\in\sS_{k,0}(\vx,\epsilon)} \mathbb{P}\left[\left\|\grad\log p_{k,0}(\vx^\prime) -\rbkwv_{k-1,0}(\vx^\prime)\right\|^2\le \frac{\epsilon}{96}\right]\right)^{n_{k,0}(10\epsilon)\cdot m_{k,0}(10\epsilon,\vx)},
        \end{aligned}
    \end{equation}
    where $\sS_{k,0}(\vx,10\epsilon)$ denotes the set of particles appear in Alg~\ref{con_alg:rse} when the input is $(k,0,\vx,10\epsilon)$.
    For any $k\in\mathbb{N}_{1,K-1}$ by requiring 
    \begin{equation*}
        \begin{aligned}
            &n_{k,0}(10\epsilon) = C_n\cdot \frac{(d+M)\cdot \max\{d,-2\log \delta\}}{(10\epsilon)^2}\quad \mathrm{where}\quad C_n = 2^6\cdot 5^2\cdot C_\eta^{-1},\\
            &m_{k,0}(10\epsilon,\vx) = C_m\cdot \frac{(d+M)^3\cdot \max\{\log \|\vx\|^2,1\}}{(10\epsilon)^3}\quad \mathrm{where}\quad C_m = 2^9\cdot 3^2\cdot 5^3\cdot C_{m,1}C_{\eta}^{-1.5}.
        \end{aligned}
    \end{equation*}
    Besides, for any $\vx\in\R^d$, we have
    \begin{equation*}
        \mathbb{P}\left[\left\|\grad\log p_{0,0}(\vx^\prime) -\rbkwv_{-1,0}(\vx^\prime)\right\|^2\le \frac{\epsilon}{96},\forall \vx^\prime\in\R^d\right]=1
    \end{equation*}
    by requiring $\tilde{\rvv}_{-1,0}(\vx^\prime)=-\grad f_*(\vx^\prime)$, which corresponds to Line~\ref{con_alg:rse_score_ret_condi} in Alg~\ref{con_alg:rse}.
\end{corollary}
\begin{proof}
    When $k>0$, plugging $r=0$ into Lemma~\ref{lem:rec_err_fine_grained}, we can obtain the result except inequality Eq~\ref{ineq: recursive_prob_lob_coarse_grained}.
    Instead, we have
    \begin{equation}
        \label{eq:margin_prob}
        \begin{aligned}            & \sP\left[\left\|\grad\log p_{k,S}(\vx)- \rbkwv_{k,0}(\vx)\right\|^2\le 10\epsilon,\forall \vx\in\R^d\right]\\
        \ge & (1-\delta)\cdot \left(\min_{\vx^\prime\in\sS_{k,0}(\vx,\epsilon)} \mathbb{P}\left[\left\|\grad\log p_{k,0}(\vx^\prime) -\rbkwv_{k-1,0}(\vx^\prime)\right\|^2\le \frac{\epsilon}{96}\right]\right)^{n_{k,0}(10\epsilon)\cdot m_{k,0}(10\epsilon,\vx)}.
        \end{aligned}
    \end{equation}
    Since the forward process, i.e., SDE~\ref{con_eq:rrds_forward}, satisfies $\rvx_{k,S}=\rvx_{k+1,0}$, we have 
    \begin{equation*}
        p_{k,S}(\vx)=p_{k+1,0}(\vx)=\int p_*(\vy)\cdot \left(2\pi \left(1-e^{-2(k+1)S}\right)\right)^{-d/2}
     \cdot \exp \left[\frac{-\left\|\vx -e^{-(k+1)S}\vy \right\|^2}{2\left(1-e^{-2(k+1)S}\right)}\right] \der \vy,
    \end{equation*}
    which means $\grad\log p_{k,S}=\grad\log p_{k+1,0}$.
    Therefore, Eq~\ref{ineq: recursive_prob_lob_coarse_grained} is established.

    When $k=0$, due to the definition of $\tilde{\rvv}_{-1,0}$ in Eq~\ref{def:v-10_def}, we know Eq~\ref{eq:margin_prob} is established.
    Hence, the proof is completed.
\end{proof}

\begin{lemma}[Errors from score estimation]
    \label{lem:rse_err}
    Under the notation in Section~\ref{sec:not_ass_app}, suppose the step size satisfy $\eta=C_\eta(d+M)^{-1}\epsilon$, we have
    \begin{equation*}
        \sP\left[\bigcap_{
                \substack{k\in \mathbb{N}_{0,K-1}\\ r\in\mathbb{N}_{0,R-1}}
           }\left\|\grad\log p_{k,S-r\eta}(\bkwx_{k,r\eta})-\rbkwv_{k,r\eta}(\bkwx_{k,r\eta})\right\|^2\le 10\epsilon\right]\ge 1-\epsilon
    \end{equation*}
    with Alg~\ref{con_alg:rse} by properly choosing the number for mean estimations and ULA iterations.
    The total gradient complexity will be at most 
    {
        \begin{equation*}
            \exp\left[\mathcal{O}\left(L^3\cdot  \left(\log \frac{Ld+M}{\epsilon}\right)^3 \cdot \max\left\{\log \log Z^2, 1\right\} \right)\right],
        \end{equation*}
    }  
    where $Z$ is the maximal norm of particles that appear in Alg~\ref{alg:rrds}.
\end{lemma}

\begin{proof}
    We begin with lower bounding the following probability with $(i,j)\in \mathbb{N}_{0,K-1}\times \mathbb{N}_{0,R-1}$ and $(i,j)\not=(0,0)$,
    \begin{equation*}
        \sP\left[\left\|\grad\log p_{k,S-r\eta}(\bkwx_{k,r\eta})- \rbkwv_{k,r\eta}(\bkwx_{k,r\eta})\right\|^2\le 10\epsilon\right].
    \end{equation*}
    In the following part of this Lemma, we set $\eta=C_\eta(d+M)^{-1}\epsilon$ and denote $\delta$ as a tiny positive constant waiting for determining.
    With Lemma~\ref{lem:rec_err_fine_grained}, we have 
    \begin{equation}
        \small
        \label{ineq:prob_lob_beg}
        \begin{aligned}            & \sP\left[\left\|\grad\log p_{k,S-r\eta}(\bkwx_{k,r\eta})- \rbkwv_{k,r\eta}(\bkwx_{k,r\eta})\right\|^2\le 10\epsilon\right]\\
        \ge & (1-\delta)\cdot \left(\min_{\vx^\prime\in\sS_{k,r}(\bkwx_{k,r\eta},10\epsilon)} \mathbb{P}\left[\left\|\grad\log p_{k,0}(\vx^\prime) -\rbkwv_{k-1,0}(\vx^\prime)\right\|^2\le \frac{10\epsilon}{960}\right]\right)^{n_{k,r}(10\epsilon)\cdot m_{k,r}(10\epsilon,\bkwx_{k,r\eta})}.
        \end{aligned}
    \end{equation}
    Then, if $k\ge 1$, for each item of the latter term, supposing $10\epsilon^\prime = \epsilon/96$, Lemma~\ref{lem:rec_err_coarse_grained} shows
    \begin{equation*}
        \small
        \begin{aligned}
            & \mathbb{P}\left[\left\|\grad\log p_{k,0}(\vx^\prime) -\rbkwv_{k-1,0}(\vx^\prime)\right\|^2\le \frac{\epsilon}{96}\right] = \mathbb{P}\left[\left\|\grad\log p_{k,0}(\vx^\prime) -\rbkwv_{k-1,0}(\vx^\prime)\right\|^2\le 10\epsilon^\prime\right] \\
            \ge & (1-\delta)\cdot \left(\min_{\vx^{\prime\prime}\in \sS_{k-1,0}(\vx^\prime, 10\epsilon^\prime)} \mathbb{P}\left[\left\|\grad\log p_{k-1,0}(\vx^{\prime\prime}) -\rbkwv_{k-2,0}(\vx^{\prime\prime})\right\|^2\le \frac{\epsilon^\prime}{96}\right]\right)^{n_{k,0}(10\epsilon^\prime)\cdot m_{k,r}(10\epsilon^\prime,\vx^\prime)}\\
            = & (1-\delta)\cdot \left(\min_{\vx^{\prime\prime}\in \sS_{k-1,0}(\vx^\prime, \epsilon/96)}\mathbb{P}\left[\left\|\grad\log p_{k-1,0}(\vx^{\prime\prime}) -\rbkwv_{k-2,0}(\vx^{\prime\prime})\right\|^2\le \frac{\epsilon}{96\cdot 960}\right]\right)^{n_{k,0}(\epsilon/96)\cdot m_{k,0}(\epsilon/96,\vx^\prime)}.
        \end{aligned}
    \end{equation*}
    Only particles that appear in the iteration will appear in powers of Eq~\ref{ineq:prob_lob_beg}.
    To simplify the notation, we set $Z$ as the upper bound of the norm of particles appear in Alg~\ref{alg:rrds}, 
    \begin{equation*}
        \begin{aligned}
            &m_{k,r}(10\epsilon,\vx)\le  m_{k,r}(10\epsilon)\coloneqq C_m\cdot \frac{(d+M)^3\cdot \max\{2\log Z,1\}}{(10\epsilon)^3}\\
            &\mathrm{and}\quad u_{k,r}(\epsilon)\coloneqq n_{k,r}(\epsilon)\cdot m_{k,r}(\epsilon).
        \end{aligned}
    \end{equation*}
    Plugging this inequality into Eq~\ref{ineq:prob_lob_beg}, we have
    \begin{equation*}
        \begin{aligned}
            &\sP\left[\left\|\grad\log p_{k,S-r\eta}(\bkwx_{k,r\eta})- \rbkwv_{k,r\eta}(\bkwx_{k,r\eta})\right\|^2\le 10\epsilon\right]\\
            \ge & (1-\delta)^{1+u_{k,r}(10\epsilon)}\cdot \left(\mathbb{P}\left[\left\|\grad\log p_{k-1,0}(\vx^{\prime\prime}) -\rbkwv_{k-2,0}(\vx^{\prime\prime})\right\|^2\le \frac{10\epsilon}{(960)^2}\right]\right)^{u_{k,r}(10\epsilon)\cdot u_{k,0}(\frac{\epsilon}{96})}.
        \end{aligned}
    \end{equation*}
    Using Lemma~\ref{lem:rec_err_coarse_grained} recursively, we will have
    \begin{equation}
        \small
        \label{ineq:prob_lob_mid}
        \begin{aligned}
            &\sP\left[\left\|\grad\log p_{k,S-r\eta}(\bkwx_{k,r\eta})- \rbkwv_{k,r\eta}(\bkwx_{k,r\eta})\right\|^2\le 10\epsilon \right]\\
            \ge & (1-\delta)^{1+u_{k,r}(10\epsilon)+ u_{k,r}(10\epsilon)\cdot u_{k,0}\left(\frac{10\epsilon}{960}\right)+\ldots + u_{k,r}(10\epsilon)\cdot \prod_{i=k}^2 u_{i,0}\left(\frac{10\epsilon}{960^{k-i+1}}\right)}\\
            &\left(\mathbb{P}\left[\left\|\grad\log p_{0,0}(\vx^\prime) -\tilde{\rvv}_{-1,0}(\vx^\prime)\right\|^2\le \frac{10\epsilon}{(960)^{k+1}},\forall \vx^\prime\in\R^d\right]\right)^{u_{k,r}(10\epsilon)\cdot \prod_{i=k}^1 u_{i,0}\left(\frac{10\epsilon}{960^{k-i+1}}\right)}\\
            = & (1-\delta)^{1+u_{k,r}(10\epsilon)+ u_{k,r}(10\epsilon)\cdot u_{k,0}\left(\frac{10\epsilon}{960}\right)+\ldots + u_{k,r}(10\epsilon)\cdot \prod_{i=k}^2 u_{i,0}\left(\frac{10\epsilon}{960^{k-i+1}}\right)}\\
            \ge & 1-\delta\cdot \left(1+u_{k,r}(10\epsilon)+ u_{k,r}(10\epsilon)\cdot u_{k,0}\left(\frac{10\epsilon}{960}\right)+\ldots + u_{k,r}(10\epsilon)\cdot \prod_{i=k}^2 u_{i,0}\left(\frac{10\epsilon}{960^{k-i+1}}\right)\right)
        \end{aligned}
    \end{equation}
    where the third inequality follows from the case $k=0$ in Lemma~\ref{lem:rec_err_coarse_grained} and the last inequality follows from union bound.

    Then, we start to upper bound the coefficient of $\delta$. 
    According to Lemma~\ref{lem:rec_err_fine_grained} and Lemma~\ref{lem:rec_err_coarse_grained}, it can be noted that the function $u_{k,r}(\cdot)$ is independent with $k$ and $r$.
    It is actually because we provide a union bound for the sample number $n_{k,r}$ and the iteration number $m_{k,r}$ when $(k,r)\in \mathbb{N}_{0,K-1}\times \mathbb{N}_{0,R-1}$.
    Therefore, the explicit form of the uniformed $u$ is defined as
    \begin{equation*}
        \begin{aligned}
            u(10\epsilon) = \underbrace{C_nC_m\cdot (d+m^2_2)^4\cdot \max\{d, \log (1/\delta^2)\}\cdot \max\{2\log Z, 1\}}_{\text{independent with $\epsilon$}}\cdot (10\epsilon)^{-5}
        \end{aligned}
    \end{equation*}
    Then, we have
    \begin{equation*}
        \begin{aligned}
            u\left(\frac{10\epsilon}{960}\right)=u(10\epsilon) \cdot 960^5\quad\mathrm{and}\quad u\left(\frac{10\epsilon}{960^i}\right)= u(10\epsilon) \cdot 960^{5i}.
        \end{aligned}
    \end{equation*}
    Combining this result with Eq~\ref{ineq:prob_lob_mid}, we obtain
    \begin{equation*}
        \begin{aligned}
            &1+u_{k,r}(10\epsilon)+ u_{k,r}(10\epsilon)\cdot u_{k,0}\left(\frac{10\epsilon}{960}\right)+\ldots + u_{k,r}(10\epsilon)\cdot \prod_{i=k}^2 u_{i,0}\left(\frac{10\epsilon}{960^{k-i+1}}\right)\\
            &\le (k+1)\cdot u(10\epsilon)\cdot \prod_{i=k}^2 u\left(\frac{10\epsilon}{960^{5(k-i+1)}}\right)= (k+1)\cdot u(10\epsilon)\cdot\prod_{i=k}^2 \left(u(10\epsilon)\cdot 960^{k-i+1} \right)\\
            &= (k+1)\cdot 960^{2.5k(k-1)}\cdot u(10\epsilon)^{k}\le K\cdot 960^{2.5(K-1)(K-2)}\cdot u(10\epsilon)^{K-1}.
        \end{aligned}
    \end{equation*}
    Considering that $K=2/S\cdot \log[(Ld+M)/\epsilon]$, to bound RHS of the previous inequality,
    we have
    \begin{equation*}
        \begin{aligned}
            & \log \left(960^{2.5(K-1)(K-2)}\cdot u(10\epsilon)^{K-1}\right)= 2.5(K-1)(K-2)\log(960)+(K-1)\log(u(10\epsilon)) \\
            &\le  2.5\cdot \log(960)\cdot \left(\frac{2}{S}\log \frac{Ld+M}{\epsilon}\right)^2+ \frac{2}{S}\log \frac{Ld+M}{\epsilon}\cdot \left(\log C_nC_m+4\log(d+M)+ \log d+\log\left(2\log \frac{1}{\delta}\right)\right.\\
            &\quad \left.+\log \left(2 \max\left\{\log Z, \frac{1}{2}\right\}\right) + \log(10^{-5})+ 5\log \frac{1}{\epsilon}\right).
        \end{aligned}
    \end{equation*}
    To make the result more clear, we set
    \begin{equation*}
        C_{u,1}\coloneqq \log(C_nC_m)+\log 2+\log \left(2 \max\left\{\log Z, \frac{1}{2}\right\}\right)-5\log 10 
    \end{equation*}
    which is independent with $d$, $\epsilon$ and $\delta$.
    Then, it has
    \begin{equation*}
        \begin{aligned}
            & \log \left(960^{2.5(K-1)(K-2)}\cdot u(10\epsilon)^{K-1}\right)\\
            \le & \frac{70}{S^2}\left(\log \frac{Ld+M}{\epsilon}\right)^2+ \frac{2}{S}\log \frac{Ld+M}{\epsilon}\cdot\left[C_{u,1}+5\log(d+M)+\log\log \frac{1}{\delta}+ 5\log \frac{1}{\epsilon}\right].
        \end{aligned}
    \end{equation*}
    which means
    \begin{equation}
        \label{ineq:grad_comp_cal}
        \begin{aligned}
            & 960^{2.5(K-1)(K-2)}\cdot u(10\epsilon)^{K-1}\\
            & \le \exp\left[\frac{70}{S^2}\left(\log \frac{Ld+M}{\epsilon}\right)^2+ \frac{2}{S}\log \frac{Ld+M}{\epsilon}\cdot\left(C_{u,1}+5\log(d+M)+\log\log \frac{1}{\delta}+ 5\log \frac{1}{\epsilon}\right)\right]\\
            & \le \mathrm{pow}\left(\frac{Ld+M}{\epsilon},\left(\left(\frac{70}{S^2}+\frac{10}{S}\right)\log\frac{Ld+M}{\epsilon}+\frac{2}{S}\log\log\frac{1}{\delta} + \frac{2C_{u,1}}{S}\right)\right)
        \end{aligned}
    \end{equation}
    where the last inequality suppose $L\ge 1$ as the previous settings.
    To simplify notation, we set
    \begin{equation*}
        \begin{aligned}
            C_{u,2}\coloneqq \frac{70}{S^2}+\frac{10}{S}\quad \mathrm{and}\quad C_{u,3}\coloneqq \frac{2C_{u,1}}{S}.
        \end{aligned}
    \end{equation*}
    Plugging this result into Eq~\ref{ineq:prob_lob_mid}, we have
    \begin{equation}
        \label{ineq:prob_lob_mid2}
        \begin{aligned}
            &\sP\left[\left\|\grad\log p_{k,S-r\eta}(\bkwx_{k,r\eta})- \rbkwv_{k,r\eta}(\bkwx_{k,r\eta})\right\|^2\le 10\epsilon \right]\\
            & \ge 1-\delta \cdot K \cdot \mathrm{pow}\left(\frac{Ld+M}{\epsilon},C_{u,2}\log\frac{Ld+M}{\epsilon}+\frac{2}{S}\log\log\frac{1}{\delta}+C_{u,3}\right) .
        \end{aligned}
    \end{equation}

    With these conditions, we can lower bound score estimation errors along Alg~\ref{alg:rrds}. 
    That is
    \begin{equation*}
        \begin{aligned}
           & \sP\left[\bigcap_{
                \substack{k\in \mathbb{N}_{0,K-1}\\ r\in\mathbb{N}_{0,R-1}}
           }\left\|\grad\log p_{k,S-r\eta}(\bkwx_{k,r\eta})-\rbkwv_{k,r\eta}(\bkwx_{k,r\eta})\right\|^2\le 10\epsilon\right]\\
           = & \prod_{
                \substack{k\in \mathbb{N}_{0,K-1}\\ r\in\mathbb{N}_{0,R-1}}
           } \sP\left[\left\|\grad\log p_{k,S-r\eta}(\bkwx_{k,r\eta})-\rbkwv_{k,r\eta}(\bkwx_{k,r\eta})\right\|^2\le 10\epsilon\right]
        \end{aligned}
    \end{equation*}
    where the first inequality establishes because the random variables, $\rbkwv_{k,r\eta}$, are independent for each $(k,r)$ pair.
    By introducing Eq~\ref{ineq:prob_lob_mid2}, we have
    \begin{equation}
        \small
        \label{ineq:prob_lob_mid3}
        \begin{aligned}
            &\prod_{
                \substack{k\in \mathbb{N}_{0,K-1}\\ r\in\mathbb{N}_{0,R-1}}
           } \sP\left[\left\|\grad\log p_{k,S-r\eta}(\bkwx_{k,r\eta})-\rbkwv_{k,r\eta}(\bkwx_{k,r\eta})\right\|^2\le 10\epsilon\right]\\
           \ge & \left(1-\delta \cdot K \cdot \mathrm{pow}\left(\frac{Ld+M}{\epsilon},C_{u,2}\log\frac{Ld+M}{\epsilon}+\frac{2}{S}\log\log\frac{1}{\delta}+C_{u,3}\right) \right)^{KR}\\
           \ge & 1- \delta\cdot K^2R\cdot \mathrm{pow}\left(\frac{Ld+M}{\epsilon},C_{u,2}\log\frac{Ld+M}{\epsilon}+\frac{2}{S}\log\log\frac{1}{\delta}+C_{u,3}\right)\\
           = & 1- \delta\cdot  \frac{4(d+M)}{SC_\eta\epsilon}\left(\log \frac{Ld+M}{\epsilon}\right)^2\cdot \mathrm{pow}\left(\frac{Ld+M}{\epsilon},C_{u,2}\log\frac{Ld+M}{\epsilon}+\frac{2}{S}\log\log\frac{1}{\delta}+C_{u,3}\right)
        \end{aligned}
    \end{equation}
    where the first inequality follows from Eq~\ref{ineq:prob_lob_mid2} and the second inequality follows from the union bound, and the last inequality follows from the combination of the choice of the step size, i.e., $\eta=C_1(d+M)^{-1}\epsilon$ and the definition of $K$ and $R$, i.e.,
    \begin{equation*}
        K=\frac{T}{S}= \frac{2}{S}\log\frac{C_0}{\epsilon},\quad R=\frac{S}{\eta}=\frac{S(d+M)}{C_\eta\epsilon}.
    \end{equation*}
    It means when $\delta$ is small enough, we can control the recursive error with a high probability, i.e.,
    \begin{equation}
        \label{ineq:prob_lob_mid4}
        \prod_{
                \substack{k\in \mathbb{N}_{0,K-1}\\ r\in\mathbb{N}_{0,R-1}}
           } \sP\left[\left\|\grad\log p_{k,S-r\eta}(\bkwx_{k,r\eta})-\rbkwv_{k,r\eta}(\bkwx_{k,r\eta})\right\|^2\le 10\epsilon\right]\ge 1-\epsilon.
    \end{equation}
    Compared with Eq~\ref{ineq:prob_lob_mid3}, Eq~\ref{ineq:prob_lob_mid4} can be achieved by requiring
    \begin{equation*}
        \small
        \begin{aligned}
            \underbrace{\frac{4(d+M)}{SC_\eta\epsilon}\left(\log \frac{Ld+M}{\epsilon}\right)^2\cdot \mathrm{pow}\left(\frac{Ld+M}{\epsilon},C_{u,2}\log\frac{Ld+M}{\epsilon}+C_{u,3}\right)}_{\text{defined as $C_B$}}\cdot \delta\mathrm{pow}\left(\frac{Ld+M}{\epsilon},\frac{2}{S}\log\log\frac{1}{\delta}\right)\le \epsilon,
        \end{aligned}
    \end{equation*}
    which can be obtained by requiring
    \begin{equation}
        \label{ineq:delta_req}
        \begin{aligned}
            &C_B\delta (-\log \delta)^{\frac{2}{S}\log \frac{Ld+M}{\epsilon}}\le \epsilon\quad \Leftrightarrow\quad (-\log \delta)^{\frac{2}{S}\log \frac{Ld+M}{\epsilon}}\le \frac{\epsilon}{C_B \delta}\\ 
            &\Leftrightarrow\quad \frac{2}{S}\log \frac{Ld+M}{\epsilon}\cdot \log\log\frac{1}{\delta}\le \log \frac{\epsilon}{ C_B \delta}
        \end{aligned}
    \end{equation}
    We suppose $\delta=\epsilon/C_B\cdot a^{-2/S\cdot \log ((Ld+M)/\epsilon)}$ and the last inequality of Eq~\ref{ineq:delta_req} becomes
    \begin{equation*}
        \begin{aligned}
            &\mathrm{LHS}= \frac{2}{S}\log \frac{Ld+M}{\epsilon}\cdot \log\left[\log\frac{C_B}{\epsilon}+\frac{2}{S}\log\frac{Ld+M}{\epsilon}\cdot \log a\right]\le  \frac{2}{S}\log \frac{Ld+M}{\epsilon}\cdot\log a = \mathrm{RHS},
        \end{aligned}
    \end{equation*}
    which is hold if we require 
    \begin{equation*}
        a\ge \max\left\{\frac{2C_B}{\epsilon}, \left(\frac{Ld+M}{\epsilon}\right)^{2/S},1\right\}.
    \end{equation*}
    Because in this condition, we have
    \begin{equation*}
        \begin{aligned}
            \log \frac{C_B}{\epsilon}+\frac{2}{S}\log \frac{Ld+M}{\epsilon}\cdot \log a\le  \log \frac{a}{2}+ (\log a)^2 \le  \frac{2a}{5}+\frac{3a}{5}=a\quad \mathrm{when}\quad a\ge 1,
        \end{aligned}
    \end{equation*}
    where the first inequality follows from the monotonicity of function $\log(\cdot)$. 
    Therefore, we have
    \begin{equation*}
        \log\left[\log\frac{C_B}{\epsilon}+\frac{2}{S}\log\frac{Ld+M}{\epsilon}\cdot \log a\right]\le \log a
    \end{equation*}
    and Eq~\ref{ineq:delta_req} establishes. 
    Without loss of generality, we suppose $3C_B/\epsilon$ dominates the lower bound of $a$. 
    Hence, the choice of $\delta$ can be determined.

    After determining the choice of $\delta$, the only problem left is the gradient complexity of Alg~\ref{alg:rrds}.
    The number of gradients calculated in Alg~\ref{alg:rrds} is equal to the number of calls
    for $\tilde{\rvv}_{-1,0}$.
    According to Eq~\ref{ineq:prob_lob_mid}, we can easily note that the number of calls of $\tilde{\rvv}_{-1,0}$ is 
    \begin{equation*}
        u_{k,r}(10\epsilon)\cdot \prod_{i=k}^1 u_{i,0}\left(\frac{10\epsilon}{960^{k-i+1}}\right) = u(10\epsilon)\prod_{i=k}^1 u\left(\frac{10\epsilon}{960^{k-i+1}}\right)
    \end{equation*}
    for each $(k,r)$ pair. 
    We can upper bound RHS of the previous equation as
    \begin{equation*}
        \begin{aligned}
            & u(10\epsilon)\prod_{i=k}^1 u\left(\frac{10\epsilon}{960^{k-i+1}}\right)= u(10\epsilon)\cdot\prod_{i=k}^2 \left(u(10\epsilon)\cdot 960^{k-i+1} \right)\\
            = & 960^{2.5k(k-1)}\cdot u(10\epsilon)^{k} \le 960^{2.5(K-1)(K-2)}\cdot u(10\epsilon)^{K-1}.
        \end{aligned}
    \end{equation*}
    Combining this result with the total number of $(k,r)$ pair, i.e., $T/\eta$, the total gradient complexity can be relaxed as
    \begin{equation}
        \label{ineq:grad_comp_fin1}
        \begin{aligned}
            & \frac{T}{\eta} \cdot 960^{2.5k(k-1)}\cdot u(10\epsilon)^{k} \le K^2R\cdot 960^{2.5(K-1)(K-2)}\cdot u(10\epsilon)^{K-1}\\
            & \le \frac{4(d+M)}{SC_\eta\epsilon}\left(\log \frac{Ld+M}{\epsilon}\right)^2 \cdot \mathrm{pow}\left(\frac{Ld+M}{\epsilon},C_{u,2}\log\frac{Ld+M}{\epsilon}+\frac{2}{S}\log\log\frac{1}{\delta}+C_{u,3}\right)\\
            & = C_B\cdot (-\log \delta)^{\frac{2}{S}\log \frac{Ld+M}{\epsilon}} \le \frac{\epsilon}{\delta} = C_B\cdot a^{\frac{2}{S}\log \frac{Ld+M}{\epsilon}}
        \end{aligned}
    \end{equation}
    where the first inequality follows from the fact $T/\eta = KR$, the second inequality follows from the combination of the choice of the step size, i.e., $\eta=C_1(d+M)^{-1}\epsilon$ and the definition of $K$ and $R$, i.e.,
    \begin{equation*}
        K=\frac{T}{S}= \frac{2}{S}\log\frac{C_0}{\epsilon},\quad R=\frac{S}{\eta}=\frac{S(d+M)}{C_1\epsilon}
    \end{equation*}
    and the last inequality follows from~\ref{ineq:delta_req}.
    Choosing $a$ as its lower bound, i.e., $2C_B/\epsilon$, RHS of Eq~\ref{ineq:grad_comp_fin1} satisfies
    \begin{equation}
        \label{ineq:grad_comp_fin2}
        \begin{aligned}
            & C_B\cdot a^{\frac{2}{S}\log \frac{Ld+M}{\epsilon}} = C_B \cdot \left(\frac{2C_B}{\epsilon}\right)^{\frac{2}{S}\log \frac{Ld+M}{\epsilon}}\le \left(\frac{2C_B}{\epsilon}\right)^{\frac{4}{S}\log \frac{Ld+M}{\epsilon}}\\
            & \le  \mathrm{pow}\left(\frac{8(d+M)}{S C_\eta\epsilon^2}\cdot \left(\log \frac{Ld+M}{\epsilon}\right)^2, \frac{4}{S}\log\frac{Ld+M}{\epsilon}\right)\\
            &\quad \cdot \mathrm{pow}\left(\frac{Ld+M}{\epsilon}, \frac{4C_{u,2}}{S}\left(\log\frac{Ld+M}{\epsilon}\right)^2+\frac{4C_{u,3}}{S}\left(\log\frac{Ld+M}{\epsilon}\right)\right)\\
            & = \exp\left[\mathcal{O}\left( \left(\log \frac{Ld+M}{\epsilon}\right)^3 \right)\right].
        \end{aligned}
    \end{equation}
    If we consider the effect of the norm of particles and the dependency of smoothness $L$ since we have
    \begin{equation*}
        \begin{aligned}
            &S = \frac{1}{2}\log\left(1+\frac{1}{2L}\right)=\Theta(L^{-1}),\quad \mathrm{when}\quad L\ge 1,\\
            &\frac{4C_{u,2}}{S} = \frac{70}{S^3}+\frac{10}{S^2}=\Theta(L^3),\quad
            \frac{4C_{u,3}}{S}=\frac{8C_{u,1}}{S^2} = \Theta\left(L^2\cdot \left(\max\left\{\log \log Z^2, 1\right\}\right)\right),
        \end{aligned}
    \end{equation*}
    Combining this result with Eq~\ref{ineq:grad_comp_fin2}, the proof is completed.
\end{proof}

{
    \begin{lemma}
        \label{lem:rse_err_2}
        Under the notation in Section~\ref{sec:not_ass_app}, suppose the step size satisfy $\eta=C_\eta(d+M)^{-1}\epsilon$, we have
        \begin{equation*}
            \sP\left[\bigcap_{
                \substack{k\in \mathbb{N}_{0,K-1}\\ r\in\mathbb{N}_{0,R-1}}
           }\left\|\grad\log p_{k,S-r\eta}(\bkwx_{k,r\eta})-\rbkwv_{k,r\eta}(\bkwx_{k,r\eta})\right\|^2\le 10\epsilon\right]\ge 1-\delta^\prime
        \end{equation*}
        with Alg~\ref{con_alg:rse} by properly choosing the number for mean estimations and ULA iterations.
        The total gradient complexity will be at most 
        \begin{equation*}
            \exp\left(\mathcal{O}\left(\max\left\{ \left(\log \frac{Ld+M}{\epsilon}\right)^3, \log \frac{Ld+M}{\epsilon}\cdot \log \frac{1}{\delta^\prime}\right\} \cdot \max\left\{\log \log Z^2, 1\right\}\right)\right),
        \end{equation*}
        where $Z$ is the maximal norm of particles appeared in Alg~\ref{alg:rrds}.
    \end{lemma}
    \begin{proof}
        In this lemma, we follow the same proof roadmap as that shown in Lemma~\ref{lem:rse_err}. 
        According to Eq~\ref{ineq:prob_lob_mid3}, we have
        \begin{equation*}
            \small
            \begin{aligned}
                &\prod_{
                \substack{k\in \mathbb{N}_{0,K-1}\\ r\in\mathbb{N}_{0,R-1}}
           } \sP\left[\left\|\grad\log p_{k,S-r\eta}(\bkwx_{k,r\eta})-\rbkwv_{k,r\eta}(\bkwx_{k,r\eta})\right\|^2\le 10\epsilon\right]\\
           \ge & 1- \delta\cdot  \frac{4(d+M)}{SC_\eta\epsilon}\left(\log \frac{Ld+M}{\epsilon}\right)^2\cdot \mathrm{pow}\left(\frac{Ld+M}{\epsilon},C_{u,2}\log\frac{Ld+M}{\epsilon}+\frac{2}{S}\log\log\frac{1}{\delta}+C_{u,3}\right)
            \end{aligned}
        \end{equation*}
        where the parameter $\delta$ satisfies Lemma~\ref{lem:rec_err_fine_grained} under certain conditions.
        It means we can control the recursive error with a high probability, i.e.,
        \begin{equation}
            \label{ineq:prob_lob_mid5}
            \prod_{
                \substack{k\in \mathbb{N}_{0,K-1}\\ r\in\mathbb{N}_{0,R-1}}
           } \sP\left[\left\|\grad\log p_{k,S-r\eta}(\bkwx_{k,r\eta})-\rbkwv_{k,r\eta}(\bkwx_{k,r\eta})\right\|^2\le 10\epsilon\right]\ge 1-\delta^\prime.
        \end{equation}
        when $\delta$ satisfies
        \begin{equation*}
        \small
        \begin{aligned}
            \underbrace{\frac{4(d+M)}{SC_\eta\epsilon}\left(\log \frac{Ld+M}{\epsilon}\right)^2\cdot \mathrm{pow}\left(\frac{Ld+M}{\epsilon},C_{u,2}\log\frac{Ld+M}{\epsilon}+C_{u,3}\right)}_{\text{defined as $C_B$}}\cdot \delta\mathrm{pow}\left(\frac{Ld+M}{\epsilon},\frac{2}{S}\log\log\frac{1}{\delta}\right)\le \delta^\prime.
        \end{aligned}
        \end{equation*}
        We can reformulate the above inequality as follows.
        \begin{equation}
        \label{ineq:delta_req_detp}
        \begin{aligned}
            &C_B\delta (-\log \delta)^{\frac{2}{S}\log \frac{Ld+M}{\epsilon}}\le \delta^\prime \quad \Leftrightarrow\quad (-\log \delta)^{\frac{2}{S}\log \frac{Ld+M}{\epsilon}}\le \frac{\delta^\prime}{C_B \delta}\\ 
            &\Leftrightarrow\quad \frac{2}{S}\log \frac{Ld+M}{\epsilon}\cdot \log\log\frac{1}{\delta}\le \log \frac{\delta^\prime}{ C_B \delta}.
        \end{aligned}
        \end{equation}
        By requiring $\delta=\delta^\prime/C_B\cdot a^{-2/S\cdot \log ((Ld+M)/\epsilon)}$, the last inequality of the above can be written as 
        \begin{equation*}
        \begin{aligned}
            &\mathrm{LHS}= \frac{2}{S}\log \frac{Ld+M}{\epsilon}\cdot \log\left[\log\frac{C_B}{\delta^\prime}+\frac{2}{S}\log\frac{Ld+M}{\epsilon}\cdot \log a\right]\le  \frac{2}{S}\log \frac{Ld+M}{\epsilon}\cdot\log a = \mathrm{RHS},
        \end{aligned}
        \end{equation*}
        when the choice of $a$ satisfies
        \begin{equation}
            \label{ineq:choice_of_a_cor}
            a\ge \max\left\{\frac{2C_B}{\delta^\prime}, \left(\frac{Ld+M}{\epsilon}\right)^{2/S},1\right\}.
        \end{equation}
        Since we have
        \begin{equation*}
        \begin{aligned}
            \log \frac{C_B}{\delta^\prime}+\frac{2}{S}\log \frac{Ld+M}{\epsilon}\cdot \log a\le  \log \frac{a}{2}+ (\log a)^2 \le  \frac{2a}{5}+\frac{3a}{5}=a\quad \mathrm{when}\quad a\ge 1,
        \end{aligned}
        \end{equation*}
        where the first inequality follows from the monotonicity of function $\log(\cdot)$. 
        Then, it has
        \begin{equation*}
        \log\left[\log\frac{C_B}{\delta^\prime}+\frac{2}{S}\log\frac{Ld+M}{\epsilon}\cdot \log a\right]\le \log a
        \end{equation*}
        and Eq~\ref{ineq:delta_req_detp} establishes.

        To achieve the accurate score estimation with a high probability shown in Eq~\ref{ineq:prob_lob_mid5}, the total gradient complexity will be 
        \begin{equation*}
            \frac{T}{\eta} \cdot 960^{2.5k(k-1)}\cdot u(10\epsilon)^{k} \le  C_B\cdot a^{\frac{2}{S}\log \frac{Ld+M}{\epsilon}}
        \end{equation*}
        shown in Eq~\ref{ineq:grad_comp_fin1}.
        Plugging the choice of $a$ (Eq~\ref{ineq:choice_of_a_cor}) into the above inequality, we have
        \begin{equation*}
            \small
            \begin{aligned}
                C_B\cdot a^{\frac{2}{S}\log \frac{Ld+M}{\epsilon}} & \le C_B\cdot \max\left\{ \mathrm{pow}\left(\frac{2C_B}{\delta^\prime}, \frac{2}{S}\log \frac{Ld+M}{\epsilon}\right), \mathrm{pow}\left(\frac{Ld+M}{\epsilon}, \frac{4}{S^2}\log \frac{Ld+M}{\epsilon}\right)\right\}\\
                &\le \max\left\{ \underbrace{\mathrm{pow}\left(\frac{2C_B}{\delta^\prime}, \frac{4}{S}\log \frac{Ld+M}{\epsilon}\right)}_{\mathrm{Term\ Comp.1}}, \underbrace{C_B\cdot \mathrm{pow}\left(\frac{Ld+M}{\epsilon}, \frac{4}{S^2}\log \frac{Ld+M}{\epsilon}\right)}_{\mathrm{Term\ Comp.2}}\right\}
            \end{aligned}
        \end{equation*}
        It can be easily noted that Term $\mathrm{Comp.2}$ will be dominated by Term $\mathrm{Comp.1}$.
        Then, we provide the upper bound of $\mathrm{Comp.1}$ as
        \begin{equation*}
            \small
            \begin{aligned}
                \log \left(\mathrm{Comp.1}\right) = & \frac{4}{S}\log\frac{Ld+M}{\epsilon}\cdot \left(\log 2C_B + \log (1/\delta^\prime)\right)\\
                = & \frac{4}{S}\log\frac{Ld+M}{\epsilon}\cdot \left( \log \frac{8}{SC_\eta} + \log \frac{d+M}{\epsilon} + 2\log\log \frac{Ld+M}{\epsilon} \right.\\
                &\quad \left.+ \log \frac{Ld+M}{\epsilon}\cdot \left(C_{u,2}\log\frac{Ld+M}{\epsilon}+C_{u,3}\right)+ \log (1/\delta^\prime)\right) \\
                = & \mathcal{O}\left(L^3\cdot \max\left\{ \left(\log \frac{Ld+M}{\epsilon}\right)^3, \log \frac{Ld+M}{\epsilon}\cdot \log \frac{1}{\delta^\prime}\right\}\right),
            \end{aligned}
        \end{equation*}
        which utilizes similar techniques shown in Lemma~\ref{lem:rse_err} and means
        \begin{equation*}
            C_B\cdot a^{\frac{2}{S}\log \frac{Ld+M}{\epsilon}}  \le  \exp\left(\mathcal{O}\left(L^3\cdot \max\left\{ \left(\log \frac{Ld+M}{\epsilon}\right)^3, \log \frac{Ld+M}{\epsilon}\cdot \log \frac{1}{\delta^\prime}\right\}\right)\right).
        \end{equation*}
        Hence, the proof is completed.
    \end{proof}   
}

%% file: 0_contents/0Xappendix_OU/0FLemAux.tex
\section{Auxiliary Lemmas}

\subsection{The chain rule of KL divergence}

\begin{lemma}[Lemma 6 in~\cite{chen2023improved}]
    \label{lem:lem6_chen2023improved}
    Consider the following two It\^o processes,
    \begin{equation*}
        \begin{aligned}
            \der \rvx_t = &\vf_1(\rvx_t, t)\der t+g(t)\der B_t,\quad \rvx_0 = \va,\\
            \der \rvy_t = &\vf_2(\rvy_t, t)\der t + g(t)\der B_t,\quad  \rvy_0=\va,
        \end{aligned}
    \end{equation*}
    where $\vf_1, \vf_2 \colon \R^d\rightarrow \R$ and $g\colon \R\rightarrow \R$ are continuous functions and may depend on $\va$.
    We assume the uniqueness and regularity conditions:
    \begin{itemize}
        \item The two SDEs have unique solutions.
        \item $\rvx_t, \rvy_t$ admit densities $p_t, q_t \in C^2(\R^d)$ for $t>0$.
    \end{itemize}
    Define the relative Fisher information between $p_t$ and $q_t$ by
    \begin{equation*}
        \FI{p_t}{q_t}\coloneqq \int p_t(\vx)\left\|\grad \log \frac{p_t(\vx)}{q_t(\vx)}\right\|^2 \der \vx.
    \end{equation*}
    Then for any $t>0$, the evolution of $\KL{p_t}{q_t}$ is given by
    \begin{equation*}
        \frac{\partial}{\partial t} \KL{p_t}{q_t} = -\frac{g^2(t)}{2}\FI{p_t}{q_t}+ \E \left[\left<\vf_1(\rvx_t, t)-\vf_2(\rvx_t,t),\grad\log\frac{p(\rvx_t)}{q(\rvx_t)}\right>\right].
    \end{equation*}
\end{lemma}

Lemma~\ref{lem:lem6_chen2023improved} is applied to show the KL convergence between the underlying distribution of the SDEs that have the same diffusion term and a bounded difference between their drift terms.

\begin{lemma}[Lemma 7 in~\cite{chen2023improved}]
    \label{lem:lem7_chen2023improved}
    Under the notation in Section~\ref{sec:not_ass_app}, for $k\in \mathbb{N}_{0,K-1}$ and $r\in  \mathbb{N}_{0,R-1}$, consider the reverse SDE starting from $\rbkwx_{k,r\eta} = \va$
    \begin{equation}
        \label{app_sde:axu_1}
        \der \hat{\rvx}_{k,t} = \left[\hat{\rvx}_{k,t} + 2\grad\log p_{k, S-t}(\hat{\rvx}_{k,t})\right]\der t + \sqrt{2}\der B_t,\quad \rbkwx_{k,r\eta}=\va
    \end{equation}
    and its discrete approximation
    \begin{equation}
        \label{app_sde:act_iter}
        \der \rbkwx_{k,t} = \left[\rbkwx_{k,t} + 2\rbkwv_{k, r\eta}\left(\rbkwx_{k, r\eta}\right)\right]\der t + \sqrt{2}\der B_t, \quad \rbkwx_{k,r\eta}=\va
    \end{equation}
    for time $t\in [k\eta, (k+1)\eta]$. 
    Let $\hat{p}_{k, t|r\eta}$ be the density of $\hat{\rvx}_{k,t}$ given $\hat{\rvx}_{k,r\eta}$ and $\bkwp_{k, t|r\eta}$ be the density of $\rbkwx_{k,t}$ given $\rbkwx_{k,r\eta}$ . Then, we have
    \begin{itemize}
        \item For any $\va \in \R^d$, the two processes satisfy the uniqueness and regularity condition stated in Lemma~\ref{lem:lem6_chen2023improved}, which means SDE~\ref{app_sde:axu_1} and SDE~\ref{app_sde:act_iter} have unique solutions and $\hat{p}_{k, t|r\eta}(\cdot|\va), \bkwp_{k, t|r\eta}(\cdot |\va)\in C^2(\R^d)$ for $t\in (r\eta, (r+1)\eta]$. 
        \item For a.e., $\va\in\R^d$, we have
        \begin{equation*}
            \lim_{t\rightarrow r\eta_+} \KL{\hat{p}_{k, t|r\eta}(\cdot|\va)}{\tilde{p}_{k, t|r\eta}(\cdot |\va)} = 0.
        \end{equation*}
    \end{itemize}
\end{lemma}

\begin{lemma}[Variant of Proposition 8 in~\cite{chen2023improved}]
    \label{lem:prop8_chen2023improved}
    Under the notation in Section~\ref{sec:not_ass_app} and Algorithm~\ref{alg:rrds}, we have
    \begin{equation*}
        \begin{aligned}
            \KL{\hat{p}_{0,S}}{\bkwp_{0,S}}\le  &\KL{\hat{p}_{K-1,0}}{\bkwp_{K-1,0}}\\
            &+\sum_{k=0}^{K-1}\sum_{r=0}^{R-1} \int_{0}^\eta \E_{(\hat{\rvx}_{k,t+r\eta}, \hat{\rvx}_{k,r\eta})} \left[\left\|\grad\log p_{k, S-(t+r\eta)}(\hat{\rvx}_{k,t+r\eta}) - \rbkwv_{k,r\eta}(\hat{\rvx}_{k,r\eta})\right\|^2\right]\der t.
        \end{aligned}
    \end{equation*}
\end{lemma}
\begin{proof}
    Under the notation in Section~\ref{sec:not_ass_app}, for $k\in \mathbb{N}_{0,K-1}$ and $r\in  \mathbb{N}_{0,R-1}$, let $\hat{p}_{k,t|r\eta}$ be the density of $\hat{\rvx}_{k,t}$ given $\hat{\rvx}_{k,r\eta}$ and $\bkwp_{k,t|r\eta}$ be the density of $\rbkwx_{k,t}$ given $\rbkwx_{k,r\eta}$.
    According to Lemma~\ref{lem:lem7_chen2023improved} and Lemma~\ref{lem:lem6_chen2023improved}, for any $\rbkwx_{k,r\eta}=\va$, we have
    \begin{equation*}
        \begin{aligned}
            &\frac{\der}{\der t}\KL{\hat{p}_{k, t|r\eta}(\cdot|\va)}{\bkwp_{k,t|r\eta}(\cdot|\va)} \\
            & =  -\FI{\hat{p}_{k,t|r\eta}(\cdot|\va)}{\bkwp_{k,t|r\eta}(\cdot|\va)} + 2\E_{\rvx \sim \hat{p}_{k,t|r\eta}(\cdot|\va)}\left[\left<\grad\log p_{k,S-t}(\rvx) - \rbkwv_{k,r\eta}(\va),\grad \log \frac{\hat{p}_{k,t|r\eta}(\rvx|\va)}{\bkwp_{k,t|r\eta}(\rvx|\va)}\right>\right]\\
            & \le \E_{\rvx\sim \hat{p}_{k,t|r\eta}(\cdot|\va)}\left[\left\|\grad \log p_{k,S-t}(\rvx)- \rbkwv_{k,r\eta}(\va)\right\|^2\right].
        \end{aligned}
    \end{equation*}
    Due to Lemma~\ref{lem:lem7_chen2023improved}, for any $\va\in \R^d$, we have
    \begin{equation*}
        \lim_{t\rightarrow r\eta_+}\KL{\hat{p}_{k,t|r\eta}(\cdot|\va)}{\bkwp_{k,t|r\eta}(\cdot|\va)} = 0,
    \end{equation*}
    which implies 
    \begin{equation*}
        \begin{aligned}
            \KL{\hat{p}_{k,t|r\eta}(\cdot|\va)}{\bkwp_{k,t|r\eta}(\cdot|\va)} = \int_{r\eta}^t \E_{\rvx \sim \hat{p}_{\tau|r\eta}(\cdot|\va)}\left[\left\|\grad\log p_{k, S-\tau}(\rvx) - \rbkwv_{k,r\eta}(\va)\right\|^2\right]\der \tau.
        \end{aligned}
    \end{equation*}
    Integrating both sides of the equation, we have
    \begin{equation*}
        \begin{aligned}
            &\E_{\hat{\rvx}_{k,r\eta}\sim \hat{p}_{k,r\eta}}\left[\KL{\hat{p}_{k,t|r\eta}(\cdot | \hat{\rvx}_{k,r\eta})}{\bkwp_{k,t|r\eta}(\cdot|\hat{\rvx}_{k,r\eta})}\right]\le  \int_{r\eta}^{t} \E \left[\left\|\grad\log p_{k,S-\tau}(\hat{\rvx}_{k,\tau}) - \rbkwv_{k,r\eta}(\hat{\rvx}_{k,r\eta})\right\|^2\right]\der \tau.
        \end{aligned}
    \end{equation*}
    According to the chain rule of KL divergence~\cite{chen2023improved}, we have
    \begin{equation*}
        \begin{aligned}
            & \KL{\hat{p}_{k,(r+1)\eta}}{\bkwp_{k,(r+1)\eta}}\\
            &\le \KL{\hat{p}_{k,r\eta}}{\bkwp_{k,r\eta}}+ \E_{\hat{\rvx}_{k,r\eta}\sim \hat{p}_{k,r\eta}}\left[\KL{\hat{p}_{k,(r+1)\eta|r\eta}(\cdot | \hat{\rvx}_{k,r\eta})}{\bkwp_{k,(r+1)\eta|r\eta}(\cdot|\hat{\rvx}_{k,r\eta})}\right]\\
            &\le \KL{\hat{p}_{k,r\eta}}{\bkwp_{k,r\eta}}+ \int_{0}^\eta \E_{(\hat{\rvx}_{k,t+r\eta}, \hat{\rvx}_{k,r\eta})} \left[\left\|\grad\log p_{k, S-(t+r\eta)}(\hat{\rvx}_{k,t+r\eta}) - \rbkwv_{k,r\eta}(\hat{\rvx}_{k,r\eta})\right\|^2\right]\der t.
        \end{aligned}
    \end{equation*}
    Summing over $r\in  \{0, 1,\ldots, R-1\}$, it has
    \begin{equation*}
        \KL{\hat{p}_{k,R\eta}}{\bkwp_{k,R\eta}}\le \KL{\hat{p}_{k,0}}{\bkwp_{k,0}}+\sum_{r=0}^{R-1} \int_{0}^\eta \E_{(\hat{\rvx}_{k,t+r\eta}, \hat{\rvx}_{k,r\eta})} \left[\left\|\grad\log p_{k, S-(t+r\eta)}(\hat{\rvx}_{k,t+r\eta}) - \rbkwv_{k,r\eta}(\hat{\rvx}_{k,r\eta})\right\|^2\right]\der t.
    \end{equation*}
    Similarly, by considering all segments, we have
    \begin{equation*}
        \begin{aligned}
            \KL{\hat{p}_{0,S}}{\bkwp_{0,S}}\le  &\KL{\hat{p}_{K-1,0}}{\bkwp_{K-1,0}}\\
            &+\sum_{k=0}^{K-1}\sum_{r=0}^{R-1} \int_{0}^\eta \E_{(\hat{\rvx}_{k,t+r\eta}, \hat{\rvx}_{k,r\eta})} \left[\left\|\grad\log p_{k, S-(t+r\eta)}(\hat{\rvx}_{k,t+r\eta}) - \rbkwv_{k,r\eta}(\hat{\rvx}_{k,r\eta})\right\|^2\right]\der t.
        \end{aligned}
    \end{equation*}
\end{proof}

\begin{lemma}[Variant of Lemma 10 in~\cite{cheng2018convergence}]
    \label{lem:strongly_lsi}
    Suppose $-\log p_*$ is $m$-strongly convex function, for any distribution with density function $p$, we have
    \begin{equation*}
        \KL{p}{p_*}\le \frac{1}{2m}\int p(\vx)\left\|\grad\log \frac{p(\vx)}{p_*(\vx)}\right\|^2\der\vx.
    \end{equation*}
    By choosing $p(\vx)=g^2(\vx)p_*(\vx)/\mathbb{E}_{p_*}\left[g^2(\rvx)\right]$ for the test function $g\colon \R^d\rightarrow \R$ and  $\mathbb{E}_{p_*}\left[g^2(\rvx)\right]<\infty$, we have
    \begin{equation*}
        \mathbb{E}_{p_*}\left[g^2\log g^2\right] - \mathbb{E}_{p_*}\left[g^2\right]\log \mathbb{E}_{p_*}\left[g^2\right]\le \frac{2}{m} \mathbb{E}_{p_*}\left[\left\|\grad g\right\|^2\right],
    \end{equation*}
    which implies $p_*$ satisfies $m$-log-Sobolev inequality.
\end{lemma}

\begin{lemma}(Corollary 3.1 in~\cite{chafai2004entropies})
\label{lem:cor31_chafai2004entropies}
If $\nu, \tilde{\nu}$ satisfy LSI with constants $\alpha, \tilde{\alpha} > 0$, respectively, then $\nu * \tilde{\nu}$ satisfies LSI with constant $(\frac{1}{\alpha} + \frac{1}{\tilde{\alpha}})^{-1}$. 
\end{lemma}

\begin{lemma}[Lemma 16 in~\cite{vempala2019rapid}]
    \label{lem:lem16_vempala2019rapid}
    Suppose a probability distribution $p$  satisfies LSI with constant $\mu>0$.
    Let a map $T\colon \R^d\rightarrow \R^d$, be a differentiable L-Lipschitz map. 
    Then, $\tilde{p}=T_{\#}p$ satisfies LSI with constant $\mu/L^2$
\end{lemma}

\begin{lemma}[Lemma 17 in~\cite{vempala2019rapid}]
    \label{lem:lem17_vempala2019rapid}
    Suppose a probability distribution $p$ satisfies LSI with a constant $\mu$.
    For any $t>0$, the probability distribution $\tilde{p}_t = p\ast \mathcal{N}(\vzero, t\mI)$ satisfies LSI with the constant $(\mu^{-1}+t)^{-1}$.
\end{lemma}

\begin{lemma}[Theorem 8 in~\cite{vempala2019rapid}]
    \label{lem:thm8_vempala2019rapid}
    Suppose $p\propto \exp(-f)$ is $\mu$ strongly log concave and $L$-smooth.
    If we conduct ULA with the step size satisfying $\eta\le 1/L$, then, for any iteration number, the underlying distribution of the output particle satisfies LSI with a constant larger than $\mu/2$.
\end{lemma}
\begin{proof}
    Suppose we run ULA from $\rvx_0\sim p_0$ to $\rvx_k \sim p_k$ where the LSI constant of $p_k$ is denoted as $\mu_k$. 
    When the step size of ULA satisfies $0<\eta\le 1/L $, due to the strong convexity of $p$, the map $\vx \mapsto \vx-\eta \grad f(\vx)$ is $(1-\eta\mu)$-Lipschitz.
    Combining the LSI property of $p_k$ and Lemma~\ref{lem:lem16_vempala2019rapid}, the distribution of $\rvx_k - \eta \grad f(\rvx_k)$ satisfies LSI with a constant $\mu_k/(1-\eta\mu)^2$.
    Then, by Lemma~\ref{lem:lem17_vempala2019rapid}, $\rvx_{k+1}=\rvx_k - \eta\grad f(\rvx_{k})+\sqrt{2\eta}\mathcal{N}(0,\mI)  \sim p_{k+1}$ satisfies $\mu_{k+1}$-LSI with
    \begin{equation*}
        \begin{aligned}
            \frac{1}{\mu_{k+1}}\le \frac{(1-\eta\mu)^2}{\mu_k}+2\eta.
        \end{aligned}
    \end{equation*} 
    For any $k$, if there is $\mu_k\ge \mu/2$, with the setting of $\eta$, i.e., $\eta\le 1/L\le 1/\mu$, then
    \begin{equation*}
        \frac{1}{\mu_{k+1}}\le \frac{(1-\eta \mu)^2}{\mu/2}+2\eta =\frac{2}{\mu}-2\eta(1-\eta\mu)\le \frac{2}{\mu}.
    \end{equation*}
    It means for any $k^\prime>k$, we have $\mu_{k^\prime}\ge \mu/2$.
    By requiring the LSI constant of initial distribution, i.e., $p_0$ to satisfy $\mu_0\ge \mu/2$, we have the underlying distribution of the output particle satisfies LSI with a constant larger than $\mu/2$. 
    Hence, the proof is completed.
\end{proof}

\begin{lemma}
    \label{lem:lsi_concentration}
    If $\nu$ satisfies a log-Sobolev inequality with log-Sobolev constant $\mu$ then every $1$-Lipschitz function $f$ is integrable with respect to $\nu$ and satisfies the concentration inequality
    \begin{equation*}
        \nu\left\{f\ge \mathbb{E}_\nu [f] +t\right\}\le \exp\left(-\frac{\mu t^2}{2}\right).
    \end{equation*}
\end{lemma}
\begin{proof}
    According to Lemma~\ref{lem:markov_inequ}, it suffices to prove that for any $1$-Lipschitz function $f$ with expectation $\mathbb{E}_\nu [f] = 0$,
    \begin{equation*}
        \mathbb{E}\left[e^{\lambda f}\right]\le e^{\lambda^2/(2\mu)}.
    \end{equation*}
    To prove this, it suffices, by a routine truncation and smoothing argument, to prove it for bounded, smooth, compactly supported functions $f$ such that $\|\grad f\|\le 1$. 
    Assume that $f$ is such a function.
    Then for every $\lambda\ge 0$ the log-Sobolev inequality implies
    \begin{equation*}
        \mathrm{Ent}_\nu \left(e^{\lambda f}\right) \le \frac{2}{\mu}\mathbb{E}_\nu \left[\left\|\grad e^{\lambda f/2}\right\|^2\right],
    \end{equation*}
    which is written as
    \begin{equation*}
        \mathbb{E}_\nu\left[\lambda f e^{\lambda f}\right] - \mathbb{E}_\nu\left[e^{\lambda f}\right]\log \mathbb{E}\left[e^{\lambda f}\right]\le \frac{\lambda^2}{2\mu}\mathbb{E}_\nu\left[\left\|\grad f\right\|^2 e^{\lambda f}\right].
    \end{equation*}
    With the notation $\varphi(\lambda)=\mathbb{E}\left[e^{\lambda f}\right]$ and $\psi(\lambda)=\log \varphi (\lambda)$, the above inequality can be reformulated as
    \begin{equation*}
        \begin{aligned}
            \lambda \varphi^\prime(\lambda)\le & \varphi(\lambda) \log \varphi(\lambda)+ \frac{\lambda^2}{2\mu}\mathbb{E}_\nu\left[\left\|\grad f\right\|^2 e^{\lambda f}\right]\\
            \le & \varphi(\lambda) \log \varphi(\lambda)+ \frac{\lambda^2}{2\mu}\varphi(\lambda),
        \end{aligned}
    \end{equation*}
    where the last step follows from the fact $\left\|\grad f\right\|\le 1$.
    Dividing both sides by $\lambda^2 \varphi(\lambda)$ gives
    \begin{equation*}
         \big(\frac{\log(\varphi(\lambda))}{\lambda}\big)^{\prime} \le \frac{1}{2 \mu}.
    \end{equation*}
    Denoting that the limiting value $\frac{\log(\varphi(\lambda))}{\lambda} \mid_{\lambda = 0} = \lim_{\lambda \to 0^{+}} \frac{\log(\varphi(\lambda))}{\lambda} = \mathbb{E}_{\nu} [f] = 0$, we have
    \begin{equation*}
        \frac{\log(\varphi(\lambda))}{\lambda} = \int_{0}^{\lambda} \big(\frac{\log(\varphi(t))}{t}\big)^{\prime} dt \le \frac{\lambda}{2 \mu},
    \end{equation*}
    which implies that 
    \begin{equation*}
        \psi(\lambda)\le \frac{\lambda^2}{2\mu} \Longrightarrow \varphi(\lambda)\le \exp\left(\frac{\lambda^2}{2\mu}\right)
    \end{equation*}
    Then the proof can be completed by a trivial argument of Lemma~\ref{lem:markov_inequ}.
\end{proof}

\begin{lemma}
    \label{lem:markov_inequ}
    Let $\rvx$ be a real random variable. 
    If there exist constants $C,A<\infty$ such that $\mathbb{E}\left[e^{\lambda \rvx}\right] \le Ce^{A\lambda^2}$ for all $\lambda>0$ then
    \begin{equation*}
        \mathbb{P}\left\{\rvx\ge t\right\}\le C\exp\left(-\frac{t^2}{4A}\right)
    \end{equation*}
\end{lemma}
\begin{proof}
    According to the non-decreasing property of exponential function $e^{\lambda \vx}$, we have
    \begin{equation*}
        \mathbb{P}\left\{\rvx\ge t\right\} = \mathbb{P}\left\{e^{\lambda \rvx}\ge e^{\lambda t}\right\} \le \frac{\mathbb{E}\left[e^{\lambda \rvx}\right]}{e^{\lambda t}}\le C\exp\left(A\lambda^2 -\lambda t\right),
    \end{equation*}
    The first inequality follows from Markov inequality, and the second follows from the given conditions.
    By minimizing the RHS, i.e., choosing $\lambda = t/(2A)$, the proof is completed.
\end{proof}

\begin{lemma}
    \label{lem:lsi_var_bound}
    Suppose $q$ is a distribution which satisfies LSI with constant $\mu$, then its variance satisfies
    \begin{equation*}
        \int q(\vx) \left\|\vx - \mathbb{E}_{\tilde{q}}\left[\rvx\right]\right\|^2 \der \vx  \le \frac{d}{\mu}.
    \end{equation*}
\end{lemma}
\begin{proof}
    It is known that LSI implies Poincar\'e inequality with the same constant, i.e., $\mu$, which means if for all smooth
    function $g\colon \R^d \rightarrow \R$,
    \begin{equation*}
        \mathrm{var}_{q}\left(g(\rvx)\right)\le \frac{1}{\mu}\mathbb{E}_{q}\left[\left\|\grad g(\rvx)\right\|^2\right].
    \end{equation*}
    In this condition, we suppose $\vb=\mathbb{E}_{q}[\rvx]$, and have the following equation
    \begin{equation*}
        \begin{aligned}
            &\int q(\vx) \left\|\vx - \mathbb{E}_{q}\left[\rvx\right]\right\|^2 \der \vx = \int q(\vx) \left\|\vx - \vb\right\|^2 \der \vx\\
            =& \int \sum_{i=1}^d q(\vx) \left(\vx_{i}-\vb_i\right)^2 \der \vx =\sum_{i=1}^d \int q(\vx) \left( \left<\vx, \ve_i\right> - \left<\vb, \ve_i\right> \right)^2 \der\vx\\
            =& \sum_{i=1}^d \int q(\vx)\left(\left<
            \vx,\ve_i\right> - \mathbb{E}_{q}\left[\left<\rvx, \ve_i\right>\right]\right)^2 \der \vx  =\sum_{i=1}^d \mathrm{var}_{q}\left(g_i(\rvx)\right)
        \end{aligned}
    \end{equation*}
    where $g_i(\vx)$ is defined as $g_i(\vx) \coloneqq \left<\vx, \ve_i\right>$
    and $\ve_i$ is a one-hot vector ( the $i$-th element of $\ve_i$ is $1$ others are $0$).
    Combining this equation and Poincar\'e inequality, for each $i$, we have
    \begin{equation*}
        \mathrm{var}_{q}\left(g_i(\rvx)\right) \le \frac{1}{\mu} \mathbb{E}_{q}\left[\left\|\ve_i\right\|^2\right]=\frac{1}{\mu}.
    \end{equation*}
    Hence, the proof is completed.
\end{proof}

\begin{lemma}(Lemma 12 in~\cite{vempala2019rapid})
    \label{lem:rapidlem12}
    Suppose $p\propto \exp(-f)$ satisfies Talagrand’s inequality with constant $\mu$ and is $L$-smooth. 
    For any $p^\prime$, 
    \begin{equation*}
        \E_{p^\prime}\left[\left\|\grad f(\rvx)\right\|^2\right]\le \frac{4L^2}{\mu}\KL{p^\prime}{p}+2Ld.
    \end{equation*}
\end{lemma}